\documentclass[11pt]{article}
\usepackage[utf8]{inputenc}

\usepackage[top=1in, bottom=1in, left=1in, right=1in]{geometry}

\usepackage[square,sort,comma,numbers]{natbib}

\usepackage{amsmath,amsfonts,amssymb,mathrsfs,mathtools}

\usepackage{stmaryrd,tocloft}
\usepackage{amsthm, textcomp}
\usepackage{color}
\usepackage{changes}%
\usepackage[ruled,vlined,linesnumbered]{algorithm2e}

\usepackage{dsfont}

\usepackage[font=small,labelfont=bf,margin=1cm]{caption} 

\usepackage{enumitem}

\usepackage{listings,xcolor}
\usepackage{tocloft}

\usepackage{hyperref}   %
\hypersetup{
  colorlinks,
  linkcolor={red!50!black},
  citecolor={blue!50!black},
  urlcolor={blue!80!black}
}
\usepackage[noabbrev, capitalise]{cleveref} %

\newcommand{\de}{\partial}
\newcommand{\wt}{\widetilde}

\newcommand{\radius}{R}

\DeclareMathOperator{\Osc}{Osc}
\DeclareMathOperator{\Ric}{Ric}

\DeclareMathOperator{\grad}{grad\,}
\let\div\relax %
\DeclareMathOperator{\div}{div}
\DeclareMathOperator{\Vol}{Vol}
\DeclareMathOperator{\Tr}{Tr}
\DeclareMathOperator{\Var}{Var}

\DeclareMathOperator{\LSI}{LSI}
\DeclareMathOperator{\PI}{PI}
\DeclareMathOperator{\Ker}{Ker}
\DeclareMathOperator{\BM}{BM}
\DeclareMathOperator{\SDP}{SDP}
\DeclareMathOperator{\MaxCut}{MaxCut}
\DeclareMathOperator{\poly}{poly}
\DeclareMathOperator{\ddiag}{ddiag}

\DeclareMathOperator{\WF}{WF}
\newcommand{\sff}{\mathrm{I\!I}}

\numberwithin{equation}{section}

\theoremstyle{plain}
\newtheorem{theorem}{Theorem}[section]

\newtheorem{corollary}[theorem]{Corollary}
\newtheorem{lemma}[theorem]{Lemma}
\newtheorem{proposition}[theorem]{Proposition}

\theoremstyle{definition}
\newtheorem{definition}[theorem]{Definition}
\newtheorem*{definition*}{Definition}
\newtheorem*{remark*}{Remark}

\newtheorem{assumption}[theorem]{Assumption}
\crefname{assumption}{Assumption}{Assumption}
\crefname{assumption}{Assumption}{Assumption}
\crefformat{equation}{(#2#1#3)}

\newcommand{\Murat}[1]{}

\definecolor{OliveGreen}{rgb}{0,0.6,0}
\newcommand{\bill}[1]{}

\title{Riemannian Langevin Algorithm for Solving Semidefinite Programs}
 \author{
  Mufan (Bill) Li\thanks{
    Department of Statistical Sciences at
    University of Toronto, and Vector Institute, \texttt{mufan.li@mail.utoronto.ca}
  }
  \and 
  Murat A. Erdogdu\thanks{
    Department of Computer Science and Department of Statistical Sciences at
    University of Toronto, and Vector Institute, \texttt{erdogdu@cs.toronto.edu}
  }
}

\date{\today}

\begin{document} 

\maketitle

\begin{abstract}

We propose a Langevin diffusion-based algorithm 
for non-convex optimization and sampling 
on a product manifold of spheres. 
Under a logarithmic Sobolev inequality, 
we establish a guarantee for finite iteration 
convergence to the Gibbs distribution 
in terms of Kullback--Leibler divergence. 
We show that with an appropriate temperature choice, 
the suboptimality gap to the global minimum is guaranteed to be 
arbitrarily small with high probability.  

As an application, we consider the Burer--Monteiro approach 
for solving a semidefinite program (SDP) with diagonal constraints, 
and analyze the proposed Langevin algorithm 
for optimizing the non-convex objective. 
In particular, we establish a logarithmic Sobolev inequality 
for the Burer--Monteiro problem when there are no spurious local minima,
but under the presence saddle points. 
Combining the results, we then provide a global optimality guarantee 
for the SDP and the Max-Cut problem. 
More precisely, we show that the Langevin algorithm achieves 
$\epsilon$ accuracy with high probability in 
$\widetilde{\Omega}( \epsilon^{-5} )$ iterations. 

\end{abstract}

\tableofcontents

\section{Introduction}
We consider the following optimization problem on a manifold
\begin{equation}
\label{eq:opt_problem}
	\min_{x \in M} F(x) \,, 
	\quad 
	\text{ where }\quad
	M = \underbrace{S^d \times \cdots \times S^d}_{n \text{ times}} \,, 
\end{equation}
where $S^d$ is the $d$-dimensional unit sphere, 
and $F:M \to \mathbb{R}$ is a non-convex objective function. 
Manifold structures often arise naturally from 
adding constraints to optimization problems on a Euclidean space. 
In matrix optimization for example, 
we often have constraints on rank, positive definiteness, symmetry etc., 
which lead to a Riemannian manifold \citep{absil2009optimization}. 
Most notably, the Burer--Monteiro method applied to various semidefinite programs with diagonal constraints 
can be written in the above form~\citep{burer2003nonlinear} with many applications including Max-Cut, community detection, and group synchronization.
See also \cite{hu2019brief} for a recent survey on 
manifold optimization. 

For non-convex optimization and sampling on Euclidean spaces, 
the unadjusted Langevin algorithm and its variants 
have been widely studied. 
See 
\cite{gelfand1991recursive,raginsky2017nonconvex,cheng2018sharp,dalalyan2019user,durmus2017nonasymptotic,erdogdu2018global,li2019stochastic,vempala2019rapid}
and the references therein. 
The main goal of these algorithms is to approximate 
the Langevin diffusion on $\mathbb{R}^N$ 
\begin{equation}
\label{eq:intro_langevin_diffusion}
	dZ_t = - \grad F(Z_t) \, dt + \sqrt{ \frac{2}{\beta} } \, dW_t \,, 
\end{equation}
where we use $\grad F$ to denote the Riemannian gradient 
(reserving $\nabla$ for the Levi--Civita connection, 
and in this special case the manifold is $\mathbb{R}^N$), 
$\{W_t\}_{t\geq 0}$ to denote the standard Brownian motion 
on $\mathbb{R}^N$, 
and $\beta > 0$ to denote the inverse temperature.
Under appropriate assumptions, it is well known that 
$\{Z_t\}_{t \geq 0}$ has a stationary Gibbs distribution 
$\nu(x) = \frac{1}{\mathcal{Z}} e^{-\beta F(x)}$, 
where $\mathcal{Z}$ is the constant normalizing $\nu(x)$ to be a probability density 
\citep{bakry2013analysis}. 
For global optimization, 
the density $\nu(x)$ will concentrate 
around the global minima in the limit $\beta \to \infty$ 
\citep{gelfand1991recursive}. 
For a finite choice of $\beta$, 
it is also possible to characterize the suboptimality
of the Gibbs distribution $\nu(x)$
\citep{raginsky2017nonconvex,erdogdu2018global}. 

Despite the success of algorithms based on Langevin diffusion in Euclidean spaces, 
the manifold setting introduces significant difficulties.
In fact, continuous time diffusion processes 
on manifolds are generally well understood \citep{hsu2002stochastic}; 
however, the numerical discretizations remain scarcely studied. 
Langevin algorithms on manifolds were first proposed 
using local coordinates to construct their updates 
\cite{girolami2011riemann,patterson2013stochastic}. 
For a special class of Hessian-type manifolds, 
Langevin updates can be done in a dual Euclidean space
via a mirror descent-type algorithm \cite{zhang2020wasserstein}. 
However, this assumption does not generalize to many manifolds, 
including compact manifolds considered in our setting.
Sampling close to a level set manifold in a Euclidean space 
is also studied in the context of unconstrained matrix optimization 
with a Langevin algorithm \cite{moitra2020fast}.

In this paper, our main focus is global non-convex optimization with Langevin algorithm on manifolds.  
Our main contributions can be summarized as follows. 
\begin{enumerate}[noitemsep]
	\item \textbf{Algorithm:} We propose a practical 
		Riemannian Langevin algorithm on the manifold of 
		products of spheres. The algorithm is first order and 
		relies on sampling an exact Brownian motion increment on a sphere, 
		which is shown to be achievable only 
		recently~\cite{jenkins2017exact,mijatovic2018note}. 
	\item \textbf{Sampling rate:} Under a logarithmic Sobolev inequality (LSI) 
		with constant $\alpha > 0$ and 
		for all $\epsilon > 0$, 
		we show that in $\Omega\left( \frac{\beta nd}{\alpha^2 \, \epsilon} 
			\log \frac{1}{\epsilon} \right)$ iterations 
			(recall $\dim M = nd$), 
		the proposed Langevin algorithm is within 
		$\epsilon$-Kullback--Leibler divergence 
		to the Gibbs distribution $\nu$. 
		We note this matches the best known complexity 
		in the Euclidean case \citep{vempala2019rapid}. 
	\item \textbf{Optimization error:} We show that the Gibbs distribution $\nu$ finds 
		an $\epsilon$-global minimum with probability $1-\delta$ 
		whenever $\beta \geq \Omega\left(\frac{nd}{\epsilon} 
			\log \frac{n}{\epsilon \delta} \right)$. 
		We note this matches the best bound in 
		the Euclidean case up to log factors~\cite{raginsky2017nonconvex,erdogdu2018global}. 
	\item \textbf{LSI under unique minimum:} 
		We develop a novel escape time based technique 
		to show that for a non-convex objective function $F$
		with a unique global minimum ($F$ may still have saddles),
		the Gibbs distribution $\frac{1}{\mathcal{Z}} e^{-\beta F(x)}$ 
		with a sufficiently large $\beta > 0$ 
		satisfies a Poincar\'{e} inequality 
		with dimension and temperature independent constant, 
		and a LSI with constant 
		$\alpha^{-1} = O(n \beta)$. 
		This significantly improves many existing bounds 
		with exponential dependence on $nd$ and $\beta$.
		Using this result, for all $\epsilon>0$, 
		we show that the Langevin algorithm
		finds an $\epsilon$-optimal 
		solution of the problem \eqref{eq:opt_problem} in 
		$\wt{\Omega}\left( \epsilon^{-2.5} \right)$ 
		iterations, 
		where $\wt{\Omega}(\cdot)$ ignores log factors. 
	\item \textbf{LSI for the Burer--Monteiro problem:} 
		As an application, 
		we study the Burer--Monteiro relaxation of 
		the Max-Cut SDP~\citep{burer2003nonlinear} 
		using the Langevin algorithm. 
		We show that the LSI constant for the Burer--Monteiro problem
		is of order $O(n \beta^2)$ whenever all first order critical points 
		are either saddle or global minima. 
		This result implies that
		the Langevin algorithm finds an $\epsilon$-optimal global minimum 
		of the Burer--Monteiro problem in 
		$\wt{\Omega}\left( \epsilon^{-4.5} \right)$ 
		iterations. Compared to the previous case, 
		the rate is worse due to the nonuniqueness of global minima. 
\end{enumerate}

The rest of this article is organized as follows. 
In \cref{sec:langevin_under_lsi}, 
we introduce the Riemannian Langevin algorithm, 
and state the sampling convergence guarantee 
under LSI.
In \cref{sec:lsi_overview}, 
we state the main result on estimating the LSI constant 
for non-convex minimization problems, 
as well as the corresponding runtime complexity of the Langevin algorithm. 
In \cref{sec:sdp}, 
we discuss the connection to SDPs and Max-Cut problems. 
In \cref{sec:discussion}, 
we provide a discussion of the results in this work, 
including the potential to extend to more general manifolds. 
In \cref{sec:sketch}, 
we provide an overview of the proof approach 
for the main results. 
In \cref{sec:conv_proof,sec:subopt_proof,sec:escape_time_proof,sec:cor_proof}, 
we provide the detailed proofs of the main results.

\section{Riemannian Langevin Algorithm under LSI}
\label{sec:langevin_under_lsi}

We start by introducing some notations, 
where most of Riemannian geometry notations 
are based on \cite{lee2019riemann}, 
and the diffusion theory notations are based on 
\cite{bakry2013analysis}. 
Let us equip the $nd$-dimensional manifold $M$ 
with a Riemannian metric $g$ 
via the natural embedding into $\mathbb{R}^{n(d+1)}$. 
Let $\nabla$ denote the Levi--Civita connection, 
let $dV_g$ denote the Riemannian volume form, 
and let $dx$ denote the Euclidean volume form. 
For $x \in M$, we denote the tangent space at $x$ by $T_x M$, 
the tangent bundle as $TM$, 
and the space of smooth vector fields on $M$ 
as $\mathfrak{X}(M)$. 
For $u,v \in T_x M$, we let $\langle u, v \rangle_g$ 
denote the Riemannian inner product, and $|u|_g$ denote the resulting norm.
For $u,v \in \mathbb{R}^p$, we also denote by
$\langle u, v \rangle$ and $|u|$,
the corresponding Euclidean counterparts. 
We denote by $\langle A, B\rangle = \Tr(A^\top B)$,
the trace inner product of the matrices 
$A, B \in \mathbb{R}^{p\times q}$. 

Let $C^k(M)$ denote the same space 
$k$-time differentiable functions on $M$. 
For $\phi \in C^2(M)$, 
we denote the Riemannian gradient as 
$\grad \phi \in \mathfrak{X}(M)$ 
(differentiating from the connection $\nabla$), 
the Hessian as 
$\nabla^2 \phi:\mathfrak{X}(M) \times \mathfrak{X}(M) 
	\to C(M)$, 
the divergence of a vector field 
$A \in \mathfrak{X}(M)$ as $\div A$, 
and the Laplace--Beltrami operator (or Laplacian for short) 
as $\Delta \phi = \div \grad \phi$. 
We will also use the musical isomorphisms $\sharp, \flat$ 
to raise and drop an index, in particular we will write 
$\nabla^2 \phi^\sharp: \mathfrak{X}(M) \to \mathfrak{X}(M)$ 
such that $\nabla^2 \phi(x)[u,v] = 
\langle \nabla^2 \phi(x)^\sharp[u], v \rangle_g$ 
for all $u,v \in T_x M$. 
We say that a function 
$\phi:M \to \mathbb{R}$ is $K_1$-Lipschitz 
if its (weak) derivative exists and is bounded, i.e. 
$\sup_{|v|_g = 1} \langle \grad \phi(x) , v \rangle_g \leq K_1$ 
for all $x \in M$. 
Similarly, we say that 
$\grad \phi$ is $K_2$-Lipschitz if 
$\sup_{|v|_g = 1} \nabla^2 \phi(x)[v,v] \leq K_2$ 
for all $x \in M$, 
and $\nabla^2 \phi$ is $K_3$-Lipschitz if 
$\sup_{|v_1|_g,|v_2|_g,|v_3|_g \leq 1} 
\nabla^3 \phi(x)[v_1,v_2,v_3] \leq K_3$ for all $x \in M$. 
This allows us to state our first assumption on 
the objective function $F$. 

\begin{assumption}
\label{asm:c2}
	$F \in C^2(M)$.
	Without loss of generality, 
	we let $\min_{x \in M} F(x) = 0$. 
\end{assumption}
Here, we remark that $F \in C^2(M)$ implies that $F$ and $\grad F$
are Lipschitz due to $M$ being compact; henceforth, we let
$K_1, K_2 \geq 1$ denote the Lipschitz constants of
$F$ and $\grad F$, respectively.
We will next introduce several definitions for 
stochastic analysis on manifolds 
\citep[Section 1.2-1.3]{hsu2002stochastic}. 
Let $(\Omega, \mathcal{F}_*, \mathbb{P})$ 
be a filtered probability space, 
where $\mathcal{F}_* = \{ \mathcal{F}_t \}_{t \geq 0}$ 
is a complete right continuous filtration,
and
let $L$ be a second order elliptic differential operator on 
a compact Riemannian manifold $(M,g)$. 
We say $\{Z_t\}_{t\geq 0}$ is a $M$-valued 
diffusion process generated by $L$ 
if $\{Z_t\}_{t \geq 0}$ is 
an $\mathcal{F}_*$-adapted semimartingale and 
\begin{equation}
\label{eq:martingale_representation_defn}
	M^f(Z)_t := f(Z_t) - f(Z_0) - \int_0^t \, Lf(Z_s) \, ds \,, 
\end{equation}
is an $\mathcal{F}_*$-adapted local martingale 
for all $f \in C^\infty(M)$. 
In particular, if $\{Z_t\}_{t\geq 0}$ is generated by 
$L \, \phi := \langle - \grad F, \grad \phi \rangle_g 
+ \frac{1}{\beta} \, \Delta \phi$, 
where we emphasize $\grad$ is the Riemannian gradient 
and $\Delta$ is the Laplace--Beltrami operator, 
we say $\{Z_t\}_{t\geq 0}$ is a Langevin diffusion. 

While this definition is not constructive, 
the existence and uniqueness of $L$-diffusions 
are well established \citep[Theorem 1.3.4 and 1.3.6]{hsu2002stochastic}. 
Furthermore, there is an explicit construction via 
horizontal lifts to the frame bundle 
\citep[Section 2.3]{hsu2002stochastic}, 
which can be interpretted as rolling the manifold 
against the corresponding Euclidean diffusion path 
without ``slipping''. 
See \cref{fig:sphere} 
for an example of a Brownian motion 
on a sphere $S^2$ generated by rolling along 
a Brownian motion path on $\mathbb{R}^2$. 

\begin{figure}[h]
\centering
\includegraphics[trim={0 0.2 0in 1.2in}, clip,width = 0.6\textwidth]{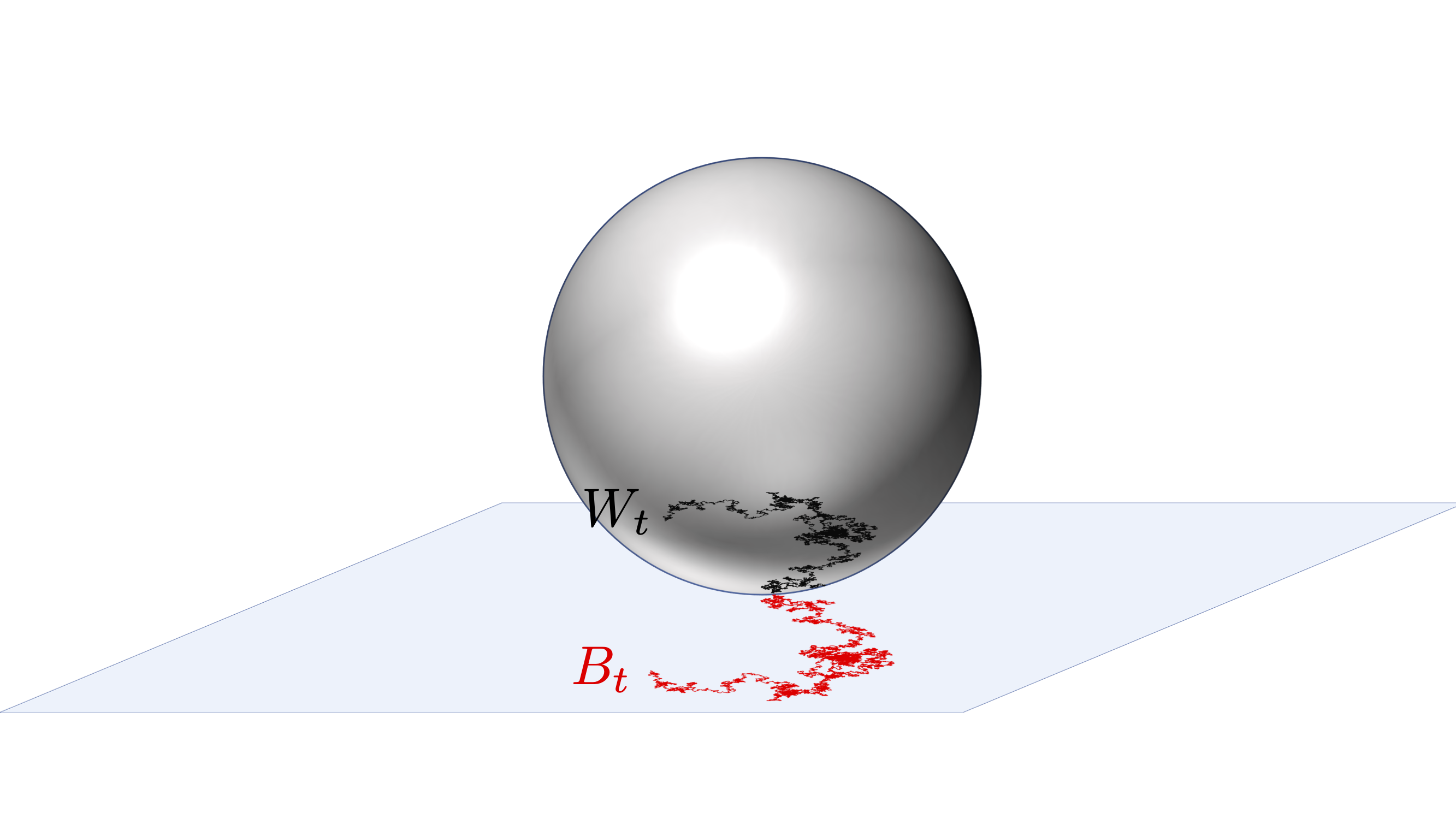} 
\vspace{-.2in}
\caption{Generating a Brownian motion $\{W_t\}_{t\geq 0}$ 
	on a sphere $S^2$ by ``rolling without slipping'' 
	on an Euclidean Brownian motion path $\{B_t\}_{t \geq 0}$ 
	in $\mathbb{R}^2$. }
\label{fig:sphere}
\end{figure}
\noindent This allows us to write the Langevin diffusion 
using the same (formal) notation 
\begin{equation}
\label{eq:langevin_diffusion}
	dZ_t = - \grad F(Z_t) \, dt + \sqrt{ \frac{2}{\beta} } \, dW_t \,, 
\end{equation}
where $\{W_t\}_{t\geq 0}$ is a Brownian motion on $M$, 
and $Z_0$ is initialized with a distribution $\mu_0$ 
supported on $M$. 
We emphasize the above stochastic differential equation (SDE) 
is interpreted in the sense that for all $f \in C^2(M)$, 
the process $M^f(Z)_t$ defined in \cref{eq:martingale_representation_defn} 
is a martingale. 

Recall the generator 
$L \, \phi = \langle -\grad F, \grad \phi \rangle_g 
+ \frac{1}{\beta} \Delta \, \phi$ 
for all $\phi \in C^2(M)$, and
the stationary Gibbs density 
$\nu(x) := \frac{1}{\mathcal{Z}} e^{-\beta F(x)}$ on $M$.
We define the carr\'{e} du champ operator 
$\Gamma (\phi) := L( \phi^2 ) - 2 \phi \, L \, \phi 
	= \frac{1}{\beta} \, | \grad \phi |_g^2$, 
and
call $(M,\nu,\Gamma)$ the Markov triple 
as per conventions of \cite{bakry2013analysis}.

\begin{definition*}
We say the Markov triple $(M, \nu, \Gamma)$ satisfies 
a \textbf{logarithmic Sobolev inequality} 
with constant $\alpha > 0$, 
denoted $\LSI(\alpha)$, 
if for all probability measures $\mu$ such that $\mu \ll \nu$ 
with $h := \frac{d\mu}{d\nu} \in C^1(M)$, 
we have the inequality 
\begin{equation}
	H_\nu(\mu) := \int_M h \log h \, d\nu 
	\leq 
		\frac{1}{2 \alpha} \int_M \frac{ \Gamma(h) }{ h } \, d\nu 
	=: 
		\frac{1}{2 \alpha} I_\nu(\mu) \,, 
\end{equation}
where $H_\nu(\mu)$ is the Kullback--Leibler (KL) divergence, 
and $I_\nu(\mu)$ is the relative Fisher information.
\end{definition*}

Intuitively, Langevin diffusion $Z_t$ in \cref{eq:langevin_diffusion} 
can be interpreted as a gradient flow in 
the space of probability distributions 
\citep{jordan1998variational}, 
and $\LSI(\alpha)$ acts like 
a gradient domination condition on the KL divergence. 
Just as Polyak-\L ojasiewicz inequality implies 
an exponential rate of convergence 
for a gradient flow 
\citep{polyak1963gradient,lojasiewicz1963topological}, 
$\LSI(\alpha)$ implies exponential convergence 
for Langevin diffusion in this divergence. 
More precisely, 
if we let $Z_t \sim \rho_t$ for all $t \geq 0$, 
then we have a well known exponential decay result 
\citep[Theorem 5.2.1]{bakry2013analysis} 
\begin{equation}
\label{eq:kl_exp_conv}
	H_\nu(\rho_t) \leq H_\nu(\rho_0) \, e^{-2\alpha t} \,. 
\end{equation}

We remark that 
we chose the Markov triple convention 
to adjust for a factor of $\beta$. 
This convention gives us 
a factor of $e^{-2\alpha t}$ in \cref{eq:kl_exp_conv}, 
whereas if we choose the LSI convention without reference to 
the carr\'{e} du champ operator $\Gamma$, 
we would end up with a factor of 
$\exp( -\frac{2\alpha t}{\beta} )$ instead. 
We note there are also alternative approaches to analyze 
convergence in KL divergence based on a modified LSI 
\citep{toscani2000trend,erdogdu2020convergence}, 
which is a strictly weaker condition. 
While we assume $\LSI(\alpha)$ in this section 
to establish convergence, 
we will derive a non-trivial lower bound for $\alpha$ 
in \cref{sec:lsi_overview}. 

\begin{assumption}
\label{asm:lsi}
$(M,\nu,\Gamma)$ satisfies $\LSI(\alpha)$ for some 
constant $\alpha > 0$. 
Note our Markov triple convention is temperature dependent, 
that is for all distribution $\mu$ and $h = \frac{d\mu}{d\nu}$ we have 
\begin{equation}
	H_\nu(\mu) 
	\leq 
	\frac{1}{2\alpha} \frac{1}{\beta} 
	\int \frac{ |\grad h|_g^2 }{h} \, d\nu \,. 
\end{equation}

\end{assumption}

We will next turn to discretizing 
the Langevin diffusion in \cref{eq:langevin_diffusion}. 
Before we consider the manifold setting, 
we recall the Langevin algorithm 
in the Euclidean space $\mathbb{R}^N$ 
with a step size $\eta > 0$ 
\begin{equation}
	X_{k+1} 
	= X_k - \eta \, \grad F(X_k) 
        + \sqrt{ \frac{2\eta}{\beta} } \, \xi_k \, \text{ where }\   \xi_k \sim N(0, I_{N}).
\end{equation}
First, we observe that we cannot simply add a gradient update 
on the manifold, 
as generally straight lines are not contained on manifolds. 
Therefore, we propose to instead take a geodesic path 
starting in the direction of the gradient. 
This operation is known as the exponential map 
$\exp: TM \to M$, 
which is explicitly computable on spheres 
via the embedding $S^d \subset \mathbb{R}^{d+1}$ 
\citep{absil2009optimization} 
\begin{equation}
	\exp(x, tv) = x \cos( t|v| ) + \frac{v}{|v|} \sin( t|v| ) \,, 
\end{equation}
where $x, v \in \mathbb{R}^{d+1}$ such that $|x| = 1$, 
$\langle x, v \rangle = 0$, 
and $t \in \mathbb{R}$. 
The exponential map can also be naturally 
extended to the product manifold $M$ 
by computing it for each sphere separately. 

Secondly, we also cannot add an increment of Brownian motion 
on the manifold, as it is no longer Gaussian. 
In fact, it is generally difficult to sample Brownian motion 
increments on manifolds. 
While the direction is simple (since it is uniform), 
the magnitude, also known as the radial process, 
is difficult to sample exactly 
\begin{equation}
	dr_t = \frac{d-1}{2} \cot(r_t) \, dt + dB_t \,. 
\end{equation}
We remark that while an approximation of the diffusion is easy, 
the analysis of non-exact Brownian increments is difficult to analyze. 

This line of work of exact sampling of one dimensional diffusions 
was first started by \cite{beskos2006retrospective} 
based on an idea of rejecting biased samples to recover exact samples. 
The main breakthrough came with \cite{jenkins2017exact}, 
where the authors were able to sample a transformation of the radial process, 
known as the Wright--Fisher diffusion (see \cref{subsec:app_wright_fisher}). 
This was further transformed back to an algorithm to 
sample spherical Brownian motions again by \cite{mijatovic2018note}. 
Since $M$ is a product of spheres, 
it is equivalent to sampling $n$ independent Brownian increments on $S^d$. 
We provide more details on the sampler in \cref{sec:app_brownian_inc}. 

Putting the two operations together, 
we can now define the interpolated Langevin update 
with step size $\eta>0$ as 
\begin{equation}
\label{eq:langevin_algorithm}
	\widehat{X}_{k\eta + t} = \exp(X_{k\eta}, - t \grad F(X_{k\eta})) \,,
	\quad 
	X_{k\eta + t} = W\left( \widehat{X}_{k\eta + t}, \frac{2 t}{\beta} 
		\right) \,,
\end{equation}
where $W(x,t): M \times \mathbb{R}^+ \to M$ is 
a Brownian motion increment starting at $x \in M$ for time $t \geq 0$, 
and we abuse notation slightly to define the discrete Langevin update as 
$X_k := X_{k\eta}$ for a step size $\eta>0$. 
We also initialize $X_0 \sim \rho_0$ 
where $\rho_0$ is supported on $M$.

By viewing one step of the Langevin update as a continuous time process, 
we can derive a de Bruijn's identity for the process $\{X_t\}_{0\leq t\leq\eta}$. 

\begin{lemma}
[De Bruijn's Identity for the Discretization]
\label{lm:de_bruijn_disc_main}
Let $\{X_t\}_{0\leq t\leq\eta}$ be the continuous time representation 
of the Langevin algorithm defined above, 
and let $\rho_{t}(x)$ be the density of $X_t$. 
Then we have the following de Bruijn type identity 
\begin{equation}
\begin{aligned}
	&\overbrace{ 
	\de_t H_\nu( \rho_t ) 
		= - I_\nu(\rho_t) }^{
				\text{de Bruijn's identity}}
	+ \overbrace{
			\mathbb{E} \left\langle
			\grad F(X_t) - 
			b(t, X_0, X_t) \,, \,
			\grad \log \frac{\rho_t(X_t)}{\nu(X_t)}
		\right\rangle_g }^{
				\text{Euclidean discretization error}}
				\\ 
	&\qquad\qquad\quad\,\,
	\underbrace{
		- \mathbb{E} \div_{X_t} b(t,X_0,X_t) 
			\log \frac{\rho_t(X_t)}{\nu(X_t)} 
			}_{
			\text{Riemannian discretization error}}
			\,, 
\end{aligned}
\end{equation}
where we define $\gamma_t(x_0):= \exp( x_0, - t \, \grad F(x_0) )$, 
$b(t, x_0, x) := P_{\gamma_t(x_0), x} 
		P_{x_0, \gamma_t(x_0)} \grad F(x_0)$, 
and we use $P_{x,y}:T_x M \to T_y M$ to denote 
the parallel transport map along the unique shortest geodesic 
connecting $x,y$ when it exists, 
and zero otherwise. 

\end{lemma}

The full proof can be found in \cref{subsec:de_bruijn_proof}. 
Here we observe first component is exactly 
the classical de Bruijn's identity 
$\de_t H_\nu( \rho_t ) = - I_\nu(\rho_t)$ 
for the Langevin diffusion $\{Z_t\}_{t \geq 0}$
\citep[Proposition 5.2.2]{bakry2013analysis}. 
Since $\LSI(\alpha)$ implies 
$\de_t H_\nu( \rho_t ) = - I_\nu(\rho_t) 
\leq - 2 \alpha H_\nu(\rho_t)$, 
we can almost use a Gr\"{o}nwall type inequality 
to establish an exponential decay result similar to \cref{eq:kl_exp_conv}. 
To complete the argument, we seek a bound on the extra terms, 
which can be interpretted as discretization errors. 
We note the second term involving the inner product 
correspond to the error arising from discretizing Langevin diffusion 
in Euclidean space \citep[Equation (31)]{vempala2019rapid}, 
hence we call it the \emph{Euclidean discretization error}. 
Additionally, the last term is a new source of error arising 
from the manifold structure. 

We emphasize the divergence term 
in the Riemannian discretization error has the form 
$\div_x P_{y,x} v$ for some $v \in T_y M$, 
and this term diverges as $x$ moves towards the cut locus of $y$, 
i.e. where the geodesic is no longer unique. 
Therefore a careful analysis of the distribution 
of the Brownian motion is required to 
provide a tight bound on this error term. 

We are now ready to state our convergence result 
for sampling with the Langevin algorithm. 

\begin{theorem}
[Sampling Bound in KL Divergence]
\label{thm:finite_iteration_KL_bound}
Let $F$ satisfy \cref{asm:c2}, 
$(M,\nu,\Gamma)$ satisfy \cref{asm:lsi}, 
and suppose $d\geq 3$. 
Let $\{X_k\}_{k\geq 1}$ be the Langevin algorithm 
defined in \cref{eq:langevin_algorithm}, 
with initialization $\rho_0 \in C^1(M)$. 
If we choose $\beta \geq 1$ and 
$0 \leq \eta \leq \min\left( \frac{2}{3\alpha}, 
\frac{\alpha}{ 24 K_2 \sqrt{(\beta+d)d} } \right)$, 
then we have the following bound on the KL divergence 
of $\rho_k := \mathcal{L}(X_k)$ 
\begin{equation}
	H_\nu(\rho_k) 
	\leq 
		H_\nu(\rho_0) \,  
		e^{-\alpha k \eta} 
		+ 45 nd K_2^2 
		\frac{\eta}{\alpha} 
		\,. 
\end{equation}
\end{theorem}
The full proof can be found in \cref{subsec:conv_kl_bounds}. 
We reiterate that our convention of $\LSI(\alpha)$ 
includes an adjustment of a factor of $\beta$ in the Fisher information, 
more precisely 
\begin{equation}
	H_\nu(\mu) \leq \frac{1}{2\alpha} \frac{1}{\beta} 
		\int \frac{|\grad h|_g^2}{h} \, d\nu \,, 
\end{equation}
where $h = \frac{d\mu}{d\nu}$, 
and in the case of $\beta=1$ our convention aligns with the usual LSI. 
We also remark the order matches the best known sampling bound 
for the Langevin algorithm in an Euclidean space 
\citep{vempala2019rapid}. 
Given this error bound, we will choose the step size 
$\eta \leq O\left( \frac{\epsilon \alpha}{ nd K_2^2 } \right)$ 
(note the Theorem also requires 
$\eta \leq O\left( \frac{1}{\sqrt{\beta}} \right)$), 
to get the run complexity of 
$k \geq \Omega\left( \frac{nd K_2^2 }{ \alpha^2 \epsilon } 
\log \frac{H_\nu(\rho_0)}{ \epsilon } \right)$. 
The dependence on parameters are as follows. 
We have a linear dependence on dimension $nd$, 
which is likely tight for the unadjusted Langevin algorithm 
without additional assumptions, 
as it matches many existing work with different metrics
\cite{vempala2019rapid,cheng2018convergence,dalalyan2017further,dalalyan2019user,durmus2019analysis}. 
In terms of error tolerance, we have $\frac{1}{\epsilon}$, 
which can only be improved if we use a higher order discretization method 
or if the target distribution is unbiased (e.g. via Metropolis adjusting). 
With respect to condition number, 
we have $\frac{K_2^2}{\alpha^2} = \kappa^2$, 
which matches other existing analysis for KL divergence 
\cite{vempala2019rapid,cheng2018convergence}, 
but it is likely not tight as other analyses of first order methods 
have achieved better dependence \cite{wibisono2019proximal}. 
We will discuss a comparison with other sampling algorithms 
in \cref{sec:discussion}.

\subsection{Global Optimization Error}

To convert our sampling bound into bounds on optimization error, we now turn our attention
to bounding the suboptimality of the Gibbs distribution $\nu$ compared against the global minimum of $F$.
We state the following high probability bound.

\begin{theorem}
[Gibbs Suboptimality Bound]
\label{thm:gibbs_high_prob_bound}
Let $F$ satisfy \cref{asm:c2} 
and suppose $d \geq 3$. 
For all 
$\epsilon \in (0,1]$ and $\delta \in (0,1)$, 
if we choose 
\begin{equation}
	\beta 
	\geq 
		\frac{3 nd }{\epsilon} 
		\log \frac{ n K_2 }{ \epsilon \, \delta }
		\,, 
\end{equation}
then, we have that the Gibbs distribution 
$\nu(x) := \frac{1}{Z} e^{ -\beta F(x) }$ 
satisfies the following bound 
\begin{equation}
	\nu\left( F - \min_{y \in M} F(y) \geq \epsilon \right) 
	\leq \delta \,. 
\end{equation}
In other words, 
samples from $\nu$ find an $\epsilon$-approximate 
global minimum of $F$ with probability $1-\delta$. 
\end{theorem}
The full proof can be found in 
\cref{sec:subopt_proof}. 
We observe that dropping logarithmic factors 
using the notation $\wt{\Omega}(\cdot)$, 
we can write the condition on $\beta$ as 
$\beta \geq \wt{\Omega}( \frac{nd}{\epsilon} )$. 
This dependence is likely tight as it matches 
the case of quadratic objective functions on $\mathbb{R}^{nd}$, 
where Gibbs distribution is a Gaussian, 
and the expected suboptimality is exactly $\frac{nd}{\beta}$. 
Finally, above results allow us to compute the runtime complexity 
of the Langevin algorithm to reach an $\epsilon$-global minimum. 

\begin{corollary}
[Runtime Complexity]
\label{cor:runtime_complexity_lsi}
Let $F$ satisfy \cref{asm:c2}, 
and let $(M,\nu,\Gamma)$ satisfy \cref{asm:lsi}. 
Further let the initialization $\rho_0 \in C^1(M)$, 
and $d \geq 3$. 
Then for all choices of 
$\epsilon \in (0,1]$ 
and $\delta \in (0,1)$, 
we can choose 
\begin{equation}
\begin{aligned}
	\beta &\geq 
		\frac{3nd}{\epsilon} 
		\log \frac{ 2 n K_2 }{ \epsilon \, \delta }
			\,, \quad 
	\eta \leq 
		\min\left\{ 
				\frac{2}{3 \alpha} \, , \, 
				\frac{ \alpha \delta^2 \sqrt{\epsilon} }{ 
				180 nd K_2^2 
				\sqrt{ \log \frac{2nK_2}{\epsilon \delta} } } 
			\right\} 
			\,, \\ 
	k &\geq 
		\max\left\{ 
		\frac{3}{2} \,, 
		180 nd \frac{K_2^2}{\alpha^2} 
		\frac{ \sqrt{ \log \frac{2nK_2}{\epsilon\delta} } }{ \delta^2 \sqrt{\epsilon} }
		\right\} 
		\left( 2 \log \frac{2}{\delta} + \log H_\nu(\rho_0) \right) 
		\,, 
\end{aligned}
\end{equation}
such that the Langevin algorithm $\{X_k\}_{k\geq 1}$ 
defined in \cref{eq:langevin_algorithm} 
with distribution $X_k \sim \rho_k$ satisfies 
\begin{equation}
	\rho_k\left( F - \min_{y \in M} F(y) \geq \epsilon \right) 
	\leq \delta \,. 
\end{equation}
In other words, $X_k$ finds an $\epsilon$-global minimum 
with probability $1 - \delta$. 
\end{corollary}
The full proof can be found in \cref{subsec:runtime_lsi}. 
When naively viewing only in terms of the error tolerance $\epsilon>0$, 
\cref{cor:runtime_complexity_lsi} implies that it is sufficient to use 
$k \geq \wt{\Omega}\left( \frac{ nd }{\delta^2 \sqrt{\epsilon}} \right)$ 
steps to reach an $\epsilon$-global minimum, 
where $\wt{\Omega}(\cdot)$ here hides dependence 
on other parameters and log factors. 
However, we note that the logarithmic Sobolev constant $\alpha$ 
can depend on $\beta$ implicitly,
and as $\beta$ depends on $n, d, \epsilon$, 
this could lead to a worse runtime complexity. 
In fact, a naive lower bound on $\alpha$ can lead to 
exponential dependence on these parameters, 
hence leading to an exponential runtime. 
We will discuss this in further detail in \cref{sec:lsi_overview}, 
where we establish explicit dependence of $\alpha$ on $n,d,\beta$ 
that lead to an explicit runtime complexity.

\section{LSI and Runtime Complexity}
\label{sec:lsi_overview}

In this section, we provide sufficient conditions 
for LSI($\alpha$)
and a lower bound for the constant $\alpha$. 
While estimates for LSI constants 
are well studied \citep{bakry2013analysis,wang1997logarithmic,wang1997estimation,bakry1985diffusions,holley1987logarithmic,block2020fast,bakry2008simple,cattiaux2010note}, 
we note that straight forward approaches in Euclidean spaces 
via the Bakry--\'Emery criterion will not work here, 
as strong convexity is not possible on compact manifolds. 
Notably, when there are no spurious local minima, 
Menz and Schlichting \cite{menz2014poincare} 
introduced a perturbation method to 
establish the bound $\alpha^{-1} \leq O(\beta)$, 
but depends on dimension exponentially. 
We further improve this result by 
removing the exponential dependence on dimension $n,d$ 
in a similar setting. 

\begin{assumption}
\label{asm:app_c3}
	$F \in C^3(M)$. 
	Without loss of generality, we let
	$\min_{x\in M} F(x) = 0$. 
\end{assumption}

As a consequence of compactness and continuity,
there exist constants $K_1, K_2, K_3 \geq 1$ such that 
$F$ is $K_1$-Lipschitz, $\grad F$ is $K_2$-Lipschitz, 
and $\nabla^2 F$ is $K_3$-Lipschitz. 
Here we remark that assuming $F$ is $K_1$-Lipschitz 
does not contradict LSI on a compact manifold. 

\begin{assumption}
\label{asm:app_unique_min}
	The set of global minima
	$\mathcal{X} := \{ x \in M | F(x) = \inf_{y \in M} F(y) = 0 \}$ 
	is geodesically convex, 
	i.e. for any two points $x,y \in \mathcal{X}$, 
	the minimum distance geodesic path connecting $x,y$ 
	is also contained in $\mathcal{X}$. 
\end{assumption}

Here we remark this is strictly more general than 
a unique global minimum.

\begin{assumption}[Weak Morse]
\label{asm:app_morse}
For all critical points $y$, we have that the Hessian eigenvalues in the escape directions are bounded away from zero. 
More precisely, there exists $\lambda_* \in (0,1]$ such that for all $y$ with $\grad F(y) = 0$, we have:
\begin{enumerate}
	\item for all $v \in \Ker(\nabla^2 F(y)^\sharp)$, we have that $\grad F(\exp_y(v)) = 0$; in other words, $\exp_y(v)$ is still a critical point. 
	\item for all $v \in \Ker(\nabla^2 F(y)^\sharp)^\bot$, we have that $|\nabla^2 F(y)^\sharp[v]|_g \geq \lambda_* |v|_g$. 
\end{enumerate}
Furthermore, for all critical points that is not a global minimum, we have that $\lambda_{\min} (\nabla^2 F(y)) \leq - \lambda_* < 0$; in other words, there are no spurious local minima, 
	and all saddles points have an escape. 
\end{assumption}

We remark this condition is strictly weaker than 
a standard Morse assumption (see for example \cite{mei2018landscape}), 
and we only assume $\lambda_* \leq 1$ to simplify computation of explicit constants. 
Here we note that even if the saddle point has an escape direction, 
having ``flat directions'' will also slow down convergence of Langevin diffusion, 
as it will take longer to ``find the saddle point.''
This is characterized more precisesly using the Lyapunov functional, 
and we discuss this further in \cref{subsub:establish_lyapunov}. 

We can now state the main result of this section. 

\begin{theorem}
[Logarithmic Sobolev Inequality]
\label{thm:lyapunov_log_sobolev}
Suppose $F$ satisfies 
\cref{asm:app_c3,asm:app_morse,asm:app_unique_min} 
and $d \geq 2$. 
Then for all choices of $a,\beta > 0$ such that 
\begin{equation}
\begin{aligned}
	a^2 
	&\geq 
		\max\left( \frac{24 K_2 nd}{ C_F^2 } \,, 
		\frac{160}{\lambda_*} 
		\right)
		\,, \quad 
	\beta 
	\geq
		a^2 
		\max\left( 
			\, \frac{ 4 K_3^2 }{ \lambda_{*}^2 } \,, 
			\, K_3^2 a^4 \, 
		 \right) \,, 
\end{aligned}
\end{equation}
where $\lambda_*$ is defined in \cref{asm:app_morse}, 
we have that the Markov triple $(M, \nu, \Gamma)$ 
satisfies $\LSI(\alpha)$ with constant 
\begin{equation}
	\frac{1}{\alpha} 
	= 
		\begin{cases}
		 	440 \frac{K_2 n \beta}{\lambda_*^2} \,, 
		 	& \text{ if the global minimum is unique, } \\ 
		 	605 \frac{K_2 n \beta^2}{K_* \lambda_*} \,, 
		 	& \text{ otherwise, } 
		\end{cases} 
\end{equation}
where $K_* = \exp \left( \frac{ -2 C_F^2}{ K_2 K_3 } \right) $. 

\end{theorem}

The full proof can be found in \cref{sec:escape_time_proof}. 
Here, we observe that $\alpha^{-1} \leq O(n\beta)$ or $O(n\beta^2)$, 
which is a significant improvement over any standard approach 
to establishing $\LSI(\alpha)$, 
where one typically gets exponential dependence on 
both dimension $nd$ and temperature $\beta$. 
We also remark that a dimension free Poincar\'{e} inequality 
was established in \cref{prop:lyapunov_poincare} 
as an intermediate result; 
if the global minimum is unique, 
the constant is also temperature free. 

The main bottlenecks preventing improvement at the moment are 
(1) the requirement that $\beta \gtrsim (nd)^3$ 
for establishing the LSI, which likely can be reduced to $\beta \gtrsim nd$ 
with a more careful analysis 
(2) the fact that LSI is worse than Poincar\'e by a factor of $n\beta$, 
and we can likely use Poincar\'e  instead to analyze convergence 
(3) the fact that we did not quotient out the symmetry of the problem, 
which would yield a unique minimum and improve the dependence on $\beta$ by one order. 
Putting these together, 
we would be hopeful of an Poincar\'e constant of order $\kappa^{-1} = O(1)$ 
under the requirement of $\beta \gtrsim nd$, 
which would yield a runtime complexity of 
$\widetilde O(\frac{n d}{\delta^2 \epsilon^{0.5}})$ 
from Corollary 2.6. 
While the runtime complexity can still be improved 
via a more careful analysis, 
we note this is first polynomial time result 
for a Langevin based algorithm, 
as existing log-Sobolev inequalities on compact manifolds 
have an exponential dependence on dimension \cite{menz2014poincare}. 

We will also state the current runtime complexity. 

\begin{corollary}
\label{cor:runtime_general_problem}
Let $F:M \to \mathbb{R}$ satisfy 
\cref{asm:app_c3,asm:app_unique_min,asm:app_morse}. 
Let $\{X_k\}_{k\geq 1}$ be the Langevin algorithm 
defined in \cref{eq:langevin_algorithm}, 
with initialization $\rho_0 \in C^1(M)$, 
and $d \geq 3$. 
For all choices of 
$\epsilon \in (0,1]$ and $\delta \in (0,1)$, 
if $\beta$ and $\eta$ satisfy the conditions in 
\cref{cor:runtime_complexity_lsi} 
and \cref{thm:lyapunov_log_sobolev}, 
then choosing $k$ as
\begin{equation}
	k \geq 
	\begin{cases}
		\wt\Omega \left( 
			\frac{n^{9.5} d^{8}}{\epsilon^{2.5} \delta^2} 
		\right)
		\,, 
		& \text{ if the global minimum is unique, } \\
		\wt\Omega \left( 
			\frac{n^{15.5} d^{14}}{\epsilon^{4.5} \delta^2} 
		\right) 
		\,, 
		& \text{ otherwise, }
	\end{cases}
\end{equation}
where $\wt{\Omega}(\cdot)$ hides dependence on 
$\poly\left( K_2, K_3, C_F^{-1}, \lambda_*^{-1}, 
K_*,
\log \frac{ nd \, K_2 }{ \epsilon \, \delta } \,, 
\log H_\nu(\rho_0) \,
\right)$ 
and $K_* = \exp \left( \frac{ -2 C_F^2}{ K_2 K_3 } \right) $, 
we have that the Langevin algorithm $\{X_k\}_{k\geq 1}$ 
defined in \cref{eq:langevin_algorithm} 
with distribution $\rho_k := \mathcal{L}(X_k)$ satisfies 
\begin{equation}
	\rho_k\left( F - \min_{y \in M} F(y) \geq \epsilon \right) 
	\leq \delta \,. 
\end{equation}
In other words, $X_k$ finds an $\epsilon$-global minimum 
with probability $1 - \delta$. 

\end{corollary}

The proof can be found in 
\cref{subsec:runtime_general_problem}. 
Compared to the exponential runtime of \cite{raginsky2017nonconvex}
(which considers the Langevin algorithm for non-convex optimization 
in the Euclidean setting),
both of our results are polynomial runtime. 
This is due to the additional weak Morse 
and connected minima assumptions 
in the manifold setting, 
which yields an LSI constant that has 
polynomial temperature and dimension dependence, 
i.e., \cref{thm:lyapunov_log_sobolev}.

\section{Application to SDP and Max-Cut}
\label{sec:sdp}

In this section, we discuss an application 
of the Langevin algorithm for solving the Max-Cut SDP. 
More specifically, we consider the following SDP 
for a symmetric matrix $A \in \mathbb{R}^{n\times n}$ 
\begin{equation}
\label{eq:sdp}
\begin{aligned}
	\SDP(A) := &\max_{X} \, \langle A, X \rangle \,, \\ 
	&\text{subject to } 
		X_{ii} = 1 \, \text{ for all } \, i \in [n] \,, 
	\text{and } 
		X \succeq 0 \,, 
\end{aligned}
\end{equation}
where $[n] = \{1, 2, \cdots, n\}$ 
and we use $X \succeq 0$ to denote positive semidefiniteness. 
SDPs have a great number of applications in fields 
across computer science and engineering.
See \cite{erdogdu2017inference,majumdar2019survey} for a recent survey of 
methods and applications. 
In a seminal work, Goemans and Williamson \cite{goemans1995improved} 
analyzed the SDP in \cref{eq:sdp} 
as a convex relaxation of the Max-Cut problem on 
an undirected graph $G$ with adjacency matrix $A_G$ 
\begin{equation}
\label{eq:maxcut}
	\MaxCut(A_G) := \max_{x_i \in \{ -1, +1 \}} 
	\frac{1}{4} \sum_{i,j=1}^n A_{G,ij} ( 1 - x_i x_j) \,, 
\end{equation}
where we the relaxation of the SDP arises from choosing $A = - A_G$. 
Most importantly, the authors introduced a rounding scheme 
for optimal solutions of the SDP 
that leads to a $0.878$-optimal Max-Cut. 
Furthermore, the SDP admits a unique solution under strict complementarity, 
a condition that is known to be satisfied for almost every cost matrix $A$ 
(with respect to the Lebesgue on real symmetric matrices) 
\cite{alizadeh1997complementarity}. 
While it is well known that SDPs can be solved to arbitrary accuracy
in polynomial time via interior point methods 
\citep{nesterov2013introductory}, 
when used naively, computational costs of these methods 
scale poorly with the problem dimension $n$. 
This led to a large literature of fast SDP solvers 
\citep{arora2005fast,arora2007combinatorial,steurer2010fast,garber2011approximating,lee2019widetilde}. 

In an alternative approach, 
Burer and Monteiro \cite{burer2003nonlinear} 
introduced a low rank and non-convex reparametrization of the convex problem \cref{eq:sdp} given as
\begin{equation}
\label{eq:bm_sdp}
\begin{aligned}
	\BM(A) := &\max_x \, \langle x, A x \rangle \,, \\ 
	&\text{subject to } 
		x = [ x^{(1)}, \cdots, x^{(n)} ]^\top 
		\in \mathbb{R}^{n \times (d+1)}
	\text{ and } 
		|x^{(i)}| = 1 \text{ for all } i \in [n] \,, 
\end{aligned}
\end{equation}
where $x^{(i)} \in \mathbb{R}^{d+1}$ 
and $|x^{(i)}|$ denotes the Euclidean norm. 
We emphasize this problem is 
contained in our framework
since the manifold is a product of spheres. 
More precisely, we choose the objective function as
\begin{equation}
\label{eq:bm_loss_function}
	F(x) := - \langle x, A x \rangle 
			+ \max_{y \in M} \langle y, A y \rangle \,, 
\end{equation}
where we observe that the constant offset gives us 
$\min_{x \in M} F(x) = 0$. 
Furthermore, if we choose $(d+1)(d+2) \geq 2n$, 
then we also have that the optimal solution to 
the Burer--Monteiro problem \cref{eq:bm_sdp} 
is also an optimal solution to the SDP \cref{eq:sdp} 
\citep{barvinok1995problems,pataki1998rank,burer2003nonlinear}. 

While the Burer--Monteiro problem is non-convex, 
it is empirically observed that simple first order methods 
still perform very well \citep{javanmard2016phase}. 
To explain this phenomenon, 
it was shown that when 
$(d+1)(d+2) > 2n$, 
all second order stationary points (local minima) 
are also global minima 
for almost every cost matrix $A$ 
\cite{boumal2016non,waldspurger2018rank,cifuentes2019burer,pumir2018smoothed,cifuentes2019polynomial}. 
In a different approach by \cite{mei2017solving}, 
the authors used a Grothendieck-type inequality 
to show that all critical points are approximately optimal 
up to a multiplicative factor of $1 - \frac{1}{d}$. 

Finally, putting these discussion together, 
we can verify the assumptions required to establish an LSI, 
and therefore provide a runtime complexity guarantee. 

\begin{corollary}
\label{cor:sdp_maxcut_no_asm}
Let $F$ be the Burer--Monteiro loss function defined in \cref{eq:bm_loss_function}. 
Then for all choices of $d$ such that $(d+1)(d+2)>2n$ 
and almost every cost matrix $A$, 
$F$ satisfies \cref{asm:app_c3,asm:app_morse,asm:app_unique_min}. 

Furthermore, if we choose $d = \left\lceil \sqrt{2n} \right\rceil$, 
then for all $\epsilon \in (0,1]$ and $\delta \in (0,1)$, 
$\beta$ and $\eta$ satisfying the conditions in 
\cref{cor:runtime_complexity_lsi} and \cref{thm:lyapunov_log_sobolev}, 
and choosing $k$ as 
\begin{equation}
	k \geq 
		\wt\Omega \left( 
			\frac{ n^{22.5} }{\epsilon^{4.5} \delta^2} 
		\right) 
		\,, 
\end{equation}
where $\wt{\Omega}(\cdot)$ hides dependence on 
$\poly\left( K_2, K_3, C_F^{-1}, \lambda_*^{-1}, 
K_*
\log \frac{ n \, K_2 }{ \epsilon \, \delta } \,, 
\log H_\nu(\rho_0) \, 
\right)$ and 
$K_* = \exp \left( \frac{ -2 C_F^2}{ K_2 K_3 } \right) $, 
we have that with probability $1 - \delta$, 
the Langevin algorithm $\{X_k\}_{k \geq 1}$ 
defined in \cref{eq:langevin_algorithm} finds 
an $\epsilon$-global solution of the SDP \cref{eq:sdp} 
after $k$ iterations 
for almost every cost matrix $A$. 

Additionally, if we let $\epsilon' := \epsilon / (4 \MaxCut(A_G))$, 
then using $X_k$ and the random rounding scheme of \citep{goemans1995improved}, 
we recover an $0.878 (1-\epsilon')$-optimal Max-Cut 
for almost every adjacency matrix $A_G$. 
\end{corollary}

\begin{remark*}
Here the notion of \textbf{almost every} cost matrix $A$ is 
with respect to the Lebesgue measure on real symmetric matrices 
\cite{alizadeh1997complementarity}. 
\end{remark*}

As discussed in the earlier section, 
this is the first polynomial runtime result for the Langevin algorithm; 
and if the main technical roadblocks of the functional inequalities can be resolved, 
we are hopeful of a significant improvement in the runtime complexity.
For context, we mention some comparable algorithms, 
including a generic solver of SDP with diagonal constraints 
$\widetilde O(m \epsilon^{-3.5})$ \cite{lee2019widetilde}, 
where $m$ is the number of non-sparse entries in the cost matrix $A$; 
the fast Riemannian trust region method 
$O(n^3 d^2 \epsilon^{-2})$ \cite{mei2017solving}, 
but each iteration typically cost $O(n^2 d)$; 
Block coordinate ascent method 
$O(n^3 d^2 \epsilon^{-2})$ \cite{erdogdu2018convergence} 
with each iteration typically cost $O(nd)$.

\section{Discussion}
\label{sec:discussion}

\textbf{Extending to general manifolds}. 
Many of the proof techniques used in this paper 
only depend on knowing an explicit local coordinate system, 
not the specific manifold structure. 
However, it is not yet clear which unifying properties 
are needed to avoid working with known local coordinates. 
Using comparison theory \citep[Section 11]{lee2019riemann} 
on manifolds with bounded (Ricci and scalar) curvature, 
we speculate all analyses can be sandwiched by the edge cases: 
a sphere for maximum curvature 
and a hyperbolic space for minimum curvature. 
Therefore, this work can be viewed as the first step 
in a multiple step program to study 
Langevin algorithms on general manifolds. 

We briefly mention two additional difficulties with extending to 
more general manifolds. 
First, we remark Brownian motion increments are difficult to 
sample on general manifolds.
In fact, even sampling on spheres was a recent result 
\citep{jenkins2017exact,mijatovic2018note}. 
Without a Brownian motion increment, 
it is difficult to relate 
the discrete time algorithm \cref{eq:langevin_algorithm} 
to the continuous time diffusion \cref{eq:langevin_diffusion}, 
since
the algorithm does not lead to 
an infinitesimal generator containing the Laplace-Beltrami operator.
Second, it is generally difficult to 
provide a tight bound on the error terms in de Bruijn's identity without local coordinates.
Luckily on spheres, they can be studied 
via stereographic coordinates (\cref{subsec:app_stereo}) 
and a transformation to Wright-Fisher diffusion (\cref{subsec:app_wright_fisher}).

\vspace{0.2cm}
\noindent
\textbf{Comparison to other samplers}. 
There are known alternative sampling algorithms on manifolds.
Most notably,
Hamiltonian Monte Carlo (HMC) is widely studied
\cite{liu2016stochastic,betancourt2013generalizing,holbrook2018geodesic,cobb2019introducing,byrne2013geodesic}. 
For optimization, simulated annealing is also often used as 
a variant of MCMC \cite{kalai2006simulated,belloni2015escaping}. 
From a theoretical perspective, 
analyzing these algorithms require an orthogonal set of techniques
compared to Langevin algorithms. 
In particular, Metropolis adjusted algorithms 
are often studied in discrete time in terms of total variation distance 
\citep{meyn2009markov,douc2018markov}. 
On the other hand, Langevin algorithms are often 
studied via an approximation by 
the continuous time Langevin diffusion \cref{eq:langevin_diffusion}, 
with many techniques rooted in partial differential equations (PDEs) 
\citep{bakry2013analysis}. Therefore, our paper can be seen as complimentary to this line of work, ultimately solving the same problem but with a different set of tools.

Closer to our proposed algorithm, 
there are variants of projected and proximal Langevin algorithms 
\cite{brosse2017sampling,bubeck2018sampling,wibisono2019proximal}. 
While the projected and proximal algorithms have 
straight forward implementations on the manifold in practice, 
the analysis of convergence are usually quite difficult. 
In particular, the main difficulty arises in comparing 
the randomness of the algorithm to a spherical Brownian motion increment, 
which is not problematic in Euclidean space as Gaussians 
are easily sampled and analyzed. 
As a result, the results of the analyses are not directly comparable. 
That being said, it is plausible that the analysis of 
the proximal Langevin algorithm \cite{wibisono2019proximal} 
can be extended to our setting, 
as both our approaches are based on similar techniques 
from \cite{vempala2019rapid}. 
Several modifications are likely in place to adapt to the manifold setting, 
for example there will likely be a new Riemannian discretization error term 
as we had in Lemma 2.3. 
The proximal algorithm here is particularly interesting as it achieves 
a higher order bias error term of $O(\eta^2)$, 
and therefore yields an improved dependence on the condition number of 
$\frac{K_2^{1.5}}{\alpha^{1.5}} = \kappa^{1.5}$. 
This suggests that our current result in Theorem 2.4 
(which matches \cite{vempala2019rapid}) is likely not tight 
in terms of the condition number. 
However, the proximal algorithm requires an additional smoothness assumption, 
results in a higher order dimension dependence, 
and perhaps most importantly is not directly implementable. 
Once a proximal solver is introduced, additional error terms will 
need to be analyzed as well for a fair comparison.

\section{Overview of Proofs for Main Theorems}
\label{sec:sketch}

\subsection{Convergence of the Langevin Algorithm}

In this subsection, 
we briefly sketch the proof of \cref{thm:finite_iteration_KL_bound}. 
We first write the one step update as a continuous time process. 
In particular, we let 
\begin{equation}
\label{eq:cont_time_rep_overview}
	\widehat{X}_t = \exp( X_0, - t \grad F(X_0) ) \,, 
	\quad 
	X_t = W\left( \widehat{X}_t, \frac{2t}{\beta} \right) \,, 
\end{equation}
where recall $\exp:M \times T_M \to M$ 
is the standard exponential map, 
and $W:M \times \mathbb{R}^+ \to M$ is an increment of 
the standard Brownian motion, 
and further observe that at time $t=\eta$, 
we have that the discrete time update $X_1$ 
defined in \cref{eq:langevin_algorithm}. 

We will attempt to establish a KL divergence bound for this process. 
For convenience, we will define $\gamma_t(x) := \exp(x, -t \grad F(x))$, 
and we will be able to derive a Fokker--Planck equation 
for $\rho_{t|0}$ which we define as the condition density of 
$X_t$ given $X_0 = x_0$ in \cref{lm:disc_fokker_planck} 
\begin{equation}
	\de_t \rho_{t|0}(x|x_0) 
	= \frac{1}{\beta} \Delta \rho_{t|0}(x|x_0) 
		+ \left\langle 
			\grad_x \rho_{t|0}(x|x_0),
			P_{\gamma_{t}(x_0),x} P_{x_0, \gamma_t(x_0)}
			\grad F(x_0)
		\right\rangle_g \,,
\end{equation}
where we add a subscript $x$ to differential operators 
such as $\grad$ to denote the 
derivative with respect to the $x$ variable as needed, 
and we use $P_{x,y}:T_x M \to T_y M$ 
to denote the parallel transport 
along the geodesic between $x,y$ when it is unique, 
and zero otherwise. 

Using the Fokker--Planck equation, 
we can derive a de Bruijn identity for $\rho_t$ 
in \cref{lm:de_bruijn_disc_main} 
\begin{equation}
\begin{aligned}
	\de_t H_\nu( \rho_t ) 
		&= - I_\nu(\rho_t) 
	+ 
		\mathbb{E} \left\langle
			\grad F(X_t) - 
			b(t, X_0, X_t) \,, \,
			\grad \log \frac{\rho_t(X_t)}{\nu(X_t)}
		\right\rangle_g 
		\\ 
	&\quad\,\,
		- \mathbb{E} \div_{X_t} b(t,X_0,X_t) 
			\log \frac{\rho_t(X_t)}{\nu(X_t)} 
			\,, 
\end{aligned}
\end{equation}
where we define 
$b(t, x_0, x) := P_{\gamma_t(x_0), x} 
P_{x_0, \gamma_t(x_0)} \grad F(x_0)$. 

Here we remark if $X_t$ is exactly the Langevin diffusion, 
then the classical de Bruijn's identity 
$\de_t H_\nu( \rho_t ) = - I_\nu(\rho_t)$ 
corresponds to exactly the first term 
\citep[Proposition 5.2.2]{bakry2013analysis}. 
We also observe that the second term involving the inner product 
correspond to the term arising from discretizing Langevin diffusion 
in $\mathbb{R}^{nd}$ from \cite[Equation 31]{vempala2019rapid}, 
hence we call it the \emph{Euclidean discretization error}. 
Finally, being on the manifold $M$, 
we will also get a final term involving $\div_x b(t,x_0,x)$, 
which requires a careful analysis to give a tight bound. 

We will next control the two latter terms. 
Starting with the inner product term, 
since it matches the Eucldiean case, 
we can use the same argument as \cite{vempala2019rapid}. 
In particular, we will need a similar application of 
Talagrand's inequality in \cref{lm:exp_grad_bound}, 
then via basic Cauchy--Schwarz calculations in 
\cref{lm:inner_prod_term_bound}, 
we can get 
\begin{equation}
	\mathbb{E} \left\langle
		\grad F(X_t) - 
		b(t, X_0, X_t) \,, \,
		\grad \log \frac{\rho_t(X_t)}{\nu(X_t)}
	\right\rangle_g 
	\leq 
		\frac{1}{8} I_\nu(\rho_t) 
		+ 8 tnd K_2^2 ( 1 + t K_2 ) 
		+ \frac{ 16 \beta t^2 K_2^4 }{\alpha} H_\nu(\rho_0) \,. 
\end{equation}

Now the divergence term from Riemannian discretization error 
is far more tricky. 
Using a stereographic local coordinate, 
we can show a couple of important identities for 
the divergence of the vector field 
in \cref{lm:stereo_div_symmetry_main}. 
Here we use the superscript $(i)$ to denote 
the $i^\text{th}$ sphere component of the product manifold $M$, 
and use $g'$ to denote the Riemannian metric on $S^d$. 
Then we can show 
\begin{equation}
\begin{aligned}
	\mathbb{E} \left[ \left. 
			\div_{X_t} 
			b(t, X_0, X_t) \right| X_0 = x_0 \right] 
	&= 
		0 \,, \\ 
	\mathbb{E} \left[ \left. \left( \div_{X_t} 
			b(t, X_0, X_t) 
		\right)^2 \right| X_0 = x_0 \right]
	&= 
		| \grad F(x_0) |_{g}^2 \, 
		\left( \frac{2}{d} + 1 \right) 
		\mathbb{E} \tan\left( 
			\frac{1}{2} d_{g'}( \gamma_t(x_0)^{(i)}, X_t^{(i)} ) 
			\right)^2 \,, 
\end{aligned}
\end{equation}
where $i \in [n]$ is arbitrary. 

Observe that $d_{g'}(\gamma_t(x_0)^{(\ell)}, x^{(\ell)})$ 
is exactly the geodesic distance traveled by a Brownian motion 
in time $2t / \beta$. 
This radial process can be transformed to the well known 
Wright--Fisher diffusion, 
which we have a density formula in the form of 
an infinite series \cref{thm:wf_transition_density}. 
Using this formula, we can give an explicit bound on 
the expected value of the squared divergence term 
in \cref{lm:tan_L2_bound} 
\begin{equation}
	\mathbb{E} \tan\left( 
		\frac{1}{2} d_{g'}( \gamma_t(x_0)^{(i)}, X_t^{(i)} ) 
		\right)^2 
	\leq 
		\frac{4td}{\beta} 
		\,. 
\end{equation}
Here we remark that $\frac{2td}{\beta}$ is the variance of 
the Brownian motion term in Euclidean space $\mathbb{R}^d$, 
and that $\tan(\theta)\approx \theta$ when $\theta$ is small, 
therefore this result is tight up to a universal constant. 

Returning to the original divergence term 
in de Bruijn's identity, 
we can now split it into two using Young's inequality. 
Since the only randomness is a Brownian motion on $M$ 
when conditioned on $X_0$, 
we can use a local Poincar\'{e} inequality 
with constant $\frac{1}{2t}$, 
and control the log density ratio term 
by Fisher information in \cref{prop:exp_div_term_bound} 
\begin{equation}
	\mathbb{E} \div b(t, X_0, X_t) 
		\log \frac{ \rho_t(X_t) }{ \nu(X_t) }
	\leq 
		\frac{1}{8} I_{\nu}(\rho_t) 
		+ \frac{128 t^2 d K_2^2}{\alpha \beta} H_\nu(\rho_0)
		+ \frac{64 t^2 nd^2 K_2}{\beta^2} 
		\,. 
\end{equation}

Putting everything back together into the de Bruijn's identity, 
we can get a KL divergence bound using a Gr\"{o}nwall type inequality 
in \cref{thm:one_step_KL_bound} 
\begin{equation}
	H_\nu( \rho_t ) 
	\leq 
		e^{-\alpha t} H_\nu( \rho_0 ) 
		+ 
		30 nd K_2^2 t^2 
		\,. 
\end{equation}

The main result \cref{thm:finite_iteration_KL_bound} 
follows as a straight forward corollary, 
where we use a geometric sum bound to to control 
the KL divergence of $\rho_k := \mathcal{L}(X_k)$ 
\begin{equation}
	H_\nu(\rho_k) 
	\leq 
		H_\nu(\rho_0) \,  
		e^{-\alpha k \eta} 
		+ 45 nd K_2^2 
		\frac{\eta}{\alpha} 
		\,. 
\end{equation}

\subsection{Suboptimality of the Gibbs Distribution}

Recall the Gibbs distribution 
$\nu = \frac{1}{Z} e^{-\beta F}$, 
and we want to provide an upper bound on 
$\nu(F \geq \epsilon)$ 
given a sufficiently large $\beta > 0$. 
Without loss of generality, 
we assumed that $\inf_x F(x) = 0$, 
and let $x^*$ achieve the global min. 

The first observation we will make is that 
$\nu(F \geq \epsilon)$ can be written as a fraction 
\begin{equation}
	\nu(F \geq \epsilon) 
	= 
		\frac{ \int_{F \geq \epsilon} e^{-\beta F} dV_g(x) 
			}{
			\int_{F \geq \epsilon} e^{-\beta F} dV_g(x) 
			+ 
			\int_{F < \epsilon} e^{-\beta F} dV_g(x) 
			} 
	\leq  
		\frac{ \int_{F \geq \epsilon} e^{-\beta F} dV_g(x) 
			}{ 
			\int_{F < \epsilon} e^{-\beta F} dV_g(x) 
			} \,. 
\end{equation}

Now observe that the numerator can be upper bounded by 
\begin{equation}
	 \int_{F \geq \epsilon} e^{-\beta F} dV_g(x) 
	 \leq 
	 	e^{-\beta \epsilon} \Vol(M) \,, 
\end{equation}
and the denominator can be lower bounded by 
a quadratic approximation 
(\cref{lm:sub_opt_quad_approx})
\begin{equation}
	\int_{F < \epsilon} e^{-\beta F} dV_g(x) 
	\geq 
		\int_{B_R} e^{-\beta \frac{K_2}{2} d_g(x,x^*)^2 } dV_g(x) \,, 
\end{equation}
where $B_R :=\{ x \in M | \frac{K_2}{2} d_g(x,x^*)^2 < \epsilon \} $, 
this result follows from 
$F \leq \frac{K_2}{2} d_g(x,x^*)^2$ on $B_R$. 

Next we will use a similar normal coordinates approximation, 
such that when $R>0$ is sufficiently small, 
we can approximate the numerator as a Gaussian bound 
after normalizing the integral 
(\cref{lm:sub_opt_integral_approx} and 
\cref{eq:lb_by_gaussian}) 
\begin{equation}
	\int_{B_R} e^{-\beta \frac{K_2}{2} d_g(x,x^*)^2 } dV_g(x) 
	\geq 
		C \beta^{-nd/2} \mathbb{P}[ |X| \leq R ] \,, 
\end{equation}
where $X \sim N(0, (\beta K_2)^{-1} I_{nd})$, 
$C$ is a constant independent of $\beta$ and $\epsilon$. 
Here $\beta^{-nd/2}$ arises from the normalizing constant 
of a Gaussian integral. 
We also observe that we can choose $\beta$ sufficiently large 
such that $\mathbb{P}[ |X| \leq R ] \geq 1/2$. 

We now return to the original quantity, 
and observe that it is sufficient to show 
\begin{equation}
	\nu(F \geq \epsilon) 
	\leq 
		C e^{-\beta \epsilon} \beta^{ nd/2 } 
	\leq 
		\delta \,, 
\end{equation}
for some new constant $C$. 
With some additional calculation, 
we will see that it is sufficient to choose 
\begin{equation}
	\beta \geq \frac{C}{\epsilon} \log \frac{1}{\epsilon\delta} \,, 
\end{equation}
which gives us the desired result 
in \cref{thm:gibbs_high_prob_bound}.

\subsection{Logarithmic Sobolev Constant}

\subsubsection{Existing Results on Lyapunov Conditions}

In the context of Markov diffusions, 
a Lyapunov function is a map $W:M \to [1,\infty)$ 
that satisfies the following type of condition on 
a subset $B \subset M$ 
\begin{equation}
	L W \leq - \theta W \,, 
\end{equation}
where $L$ is the It\^{o} generator of the Langevin diffusion 
in \cref{eq:langevin_diffusion}, 
and $\theta$ is either a positive constant 
or a function depending on $x$. 
Here the set $B$ is chosen usch that the Gibbs measure $\nu$
satisfies a Poincar\'e inequality when restricted to $M\setminus B$. 
Under these two conditions, 
we have that $\nu$ satisfies a Poincar\'{e} inequality on $M$ 
\cite{bakry2008simple,cattiaux2013poincare}. 
A slight strengthening of the condition along 
with a Poincar\'{e} inequality 
also implies a logarithmic Sobolev inequality 
\cite{cattiaux2010note}. 
The advantage of this approach is that 
both of these results have made their constants explicit, 
which allows us to compute explicit dependence 
on dimension and temperature. 

Our work builds on \cite{menz2014poincare}, 
where the authors developed a careful perturbation 
to remove exponential dependence on $\beta$ 
due to the Holley--Stroock perturbation. 
However, after careful calculations, 
we found this approach still introduced 
an exponential dependence on dimensions $nd$. 
To get a polynomial runtime in terms of dimension, 
we need to develop a new technique without perturbation. 
We will defer to \cref{subsec:app_lyapunov} 
for more details on the existing Lyapunov methods. 

Lyapunov function is often interpreted as an energy quantity 
that decreases with time, 
hence controlling a desired process. 
Alternatively, we can interpret the Lyapunov function 
as the moment generating function of an escape time, 
which is well known in the Markov chain literature 
(see e.g. \cite[Theorem 14.1.3]{douc2018markov}), 
and first introduced for Markov diffusions 
in \cite{cattiaux2013poincare,cattiaux2017hitting}. 
Let $\{Z_t\}_{t\geq 0}$ be the Langevin diffusion 
defined in \cref{eq:langevin_diffusion}, 
and we know from the Feynman--Kac formula 
\citep[Theorem 7.15]{bovier2016metastability} 
that the unique solution to $LW = - \theta W$ is 
\begin{equation}
	W(x) = \mathbb{E} [ e^{\theta \tau} | Z_0 = x ] \,, 
\end{equation}
where $\tau := \{ t\geq 0 | Z_t \notin B \}$ 
is the first escape time. 
In fact, the existence of the Feynman--Kac solution 
is equivalent to the existence of a Lyapunov function 
\cite{cattiaux2013poincare}. 
Here we remark that the boundary needs to be handled carefully 
to an explicit estimate of the Poincar\'e constant 
\cite{cattiaux2013poincare}. 

The proof will come in several steps. 
First, we will establish an escape time method 
to finding a Lyapunov function. 
Secondly, we will prove a Poincar\'{e} inequality 
using the Lyapunov method from \cite{bakry2008simple}. 
Finally, we will prove the logarithmic Sobolev inequality 
by adapting HWI inequality approach of \cite{cattiaux2010note}.

\subsubsection{Establishing the Lyapunov Conditions}
\label{subsub:establish_lyapunov}

We will first establish a Lyapunov condition 
away from all saddle points, 
i.e. a point $y \in M$ such that 
$\grad F(y) = 0$ and the Hessian $\nabla^2 F(y)$ 
has at least one negative eigenvalue. 
Building on \cite{menz2014poincare}, 
we will also choose the Lyapunov function 
$W(x) = \exp\left( \frac{\beta}{2} F(x) \right)$, 
and then the Lyapunov condition is reduced to 
\begin{equation}
	\frac{LW}{W} 
	= 
		\frac{1}{2} \Delta F 
		- \frac{\beta}{4} |\grad F|_g^2 
	\leq 
		- \theta \,. 
\end{equation}
Here we emphasize that \cref{asm:app_morse} is in a sense necessary 
to upper bound the Lyapunov term by an negative constant, 
as we require $|\grad F|_g$ to grow as we move away from the saddle point. 
This can be interpretted as the ``flat directions'' of the saddle point 
can also slow down convergence of Langevin diffusion. 
We also note this is not unique to Langevin, 
as even gradient flow alone will yield a similar $- |\grad F|_g^2$ term, 
and upper bounding this is similar in spirit to 
the establishing exponential convergence using 
the Polyak-\L ojasiewicz inequality 
\cite{polyak1963gradient,lojasiewicz1963topological}. 

Therefore choosing a sufficiently large $\beta>0$ 
will meet the Lyapunov condition (\cref{lm:lyapunov_away_from_crit}). 
Additionally, we observe that near the unique global minimum, 
we have either (1) strong convexity of the function $F$, or 
(2) a positive Ricci curvature on $S^d$, 
hence we get a Poincar\'{e} inequality 
with the classical Bakry--\'{E}mery condition. 
Together this implies a Poincar\'{e} inequality away 
from saddle points (\cref{prop:lyapunov_poincare}). 

Next we will choose a second Lyapunov function 
that satisfies the condition only near saddle points. 
In particular, we will choose the escape time representation of 
$W(x) = \mathbb{E} [ e^{\theta \tau} | Z_0 = x ]$, 
where $\tau$ is the first escape time of $Z_t$ 
away from a neighbourhood from each saddle point. 
Using an equivalent characterization of 
subexponential random variables 
(\cref{thm:subexp_equiv_wainwright}), 
it is equivalent to establish a tail bound of the type 
\begin{equation}
	\mathbb{P}[ \tau > t ] \leq c \, e^{ - \theta t / 2} \,, 
\end{equation}
for some constant $c > 0$ and all $t \geq 0$. 

To study the escape time, 
we will first use a ``quadratic approximation'' 
of the function $F$ near the saddle point. 
Then the squared radial process $\frac{1}{2} d_g(y, Z_t)^2$ 
is lower bounded by a Cox--Ingersoll--Ross (CIR) process 
of the following form 
\begin{equation}
	d \left[ \frac{1}{2} r_t^2 \right] 
	= 
	\left[ 
		c_1 \frac{1}{2} r_t^2 
		+ \frac{c_2}{ \beta } 
		\right] \, dt 
		+ \sqrt{ \frac{2}{\beta} } 
			r_t \, dB_t 
		\,, 
\end{equation}
where $c_1, c_2 > 0$ are constants specified in the proof. 

Since we can compute the density of this process explicitly 
(\cref{cor:cir_density}), 
we can compute a desired escape time bound 
in \cref{prop:local_escape_time_bound}. 
Putting it together with the Lyapunov method, 
we have that $\nu$ satisfies a Poincar\'{e} inequality on $M$. 

Finally, to get a logarithmic Sobolev inequality, 
we will use the HWI inequality and Rothaus tightening from 
\cite{cattiaux2010note}, which first establishes 
\begin{equation}
	H_\nu(\mu) 
	\leq 
		W_2(\mu,\nu) \sqrt{\beta I_\nu(\mu)}
		+ \frac{\beta K_2}{2} W_2(\mu,\nu)^2 
	\leq 
		\frac{\tau \beta}{2} I_\nu(\mu) 
		+ \left( \frac{1}{2\tau} + \frac{\beta K}{2} \right) 
			W_2^2(\mu, \nu) 
	=: 
		A I_\nu(\mu) + B \,, 
\end{equation}
where $\tau > 0$ can be optimized later. 
Then if $\nu$ also satisfies a Poincar\'e inequality with constant $\kappa$, 
we can tighten the defective log-Soboelv inequality via Rothaus Lemma to get 
\begin{equation}
	H_\nu(\mu) \leq \left( A + \frac{B+2}{\kappa} \right) I_\nu(\mu) \,. 
\end{equation}
The logarithmic Sobolev constant is computed in 
\cref{thm:lyapunov_log_sobolev}.

\section{Convergence of the Langevin Algorithm - 
Proof of \cref{thm:finite_iteration_KL_bound}}
\label{sec:conv_proof}

\subsection{The Continuous Time Representation}

Once again we recall the continuous time process representation 
of the Langevin algorithm 
\begin{equation}
\label{eq:cont_time_rep}
	\widehat{X}_t = \exp( X_0, -t \grad F(X_0) ) \,,
	\quad
	X_t = W\left( \widehat{X}_t, \frac{2 t}{\beta} \right) \,,
\end{equation}
where $\exp:M \times TM \to M$ is 
the standard exponential map, 
and $W:M \times \mathbb{R}^+ \to M$ is 
a Brownian motion increment. 
We also recall the notation 
\begin{equation}
	\gamma_t(x) := \exp(x, - t \grad F(x)) \,.
\end{equation} 

Here we observe that at $t=0$, $\gamma_0(x) = x$. 
At the same time, $X_t$ is a Brownian motion 
starting at $\gamma_t(X_0)$, 
therefore the (conditional) density of $X_t$ 
starting at $X_0 = x_0$ is 
\begin{equation}
	\rho_{t|0}(x|x_0) = p(t, \gamma_t(x_0), x) \,,
\end{equation}
where $p(t,y,x)$ is the density of 
a Brownian motion with diffusion coefficient 
$\sqrt{2/\beta}$ and starting at $y$. 
In other words, $p$ is the unique solution to 
the following heat equation 
\begin{equation}
\label{eq:heat}
	\de_t p(t,y,x) = \frac{1}{\beta} \Delta_x p(t,y,x) \,, 
	\quad p(0,y,x) = \delta_y(x) \,, 
	\quad x,y \in M \,, 
\end{equation}
where $\Delta_x$ denotes the Laplacian (Laplace--Beltrami operator) 
with respect to the $x$ variable, 
$\delta_y$ is the Dirac delta distribution, 
and the partial differential equation (PDE) 
is interpreted in the distributional sense. 
See \cite[Theorem 4.1.1]{hsu2002stochastic} 
for additional details on the heat equation 
and Brownian motion density. 

To derive the Fokker--Planck equation, 
we will need an additional symmetry result 
about the density $p(t,y,x)$.

\begin{lemma}
\label{lm:grad_symmetry}
Let $p(t,y,x)$ be the unique solution of \cref{eq:heat}, 
i.e. the density of a Brownian motion on $M$ 
with diffusion coefficient $\sqrt{2/\beta}$. 
For all $x,y \in M$, 
let $P_{x,y}:T_x M \to T_y M$ be the parallel transport along 
the unique shortest geodesic connecting $x,y$ when it exists, 
and zero otherwise. 
Then we have 
\begin{equation}
	\grad_y p(t,y,x) = - P_{x,y} \grad_x p(t,y,x) \,, 
	\quad \forall t \geq 0, 
\end{equation}
where we use the subscript $\grad_y$ to denote 
the gradient with respect to the $y$ variable. 
\end{lemma}

\begin{proof}

Since $M = S^d \times \cdots \times S^d$ is a product, 
it is sufficient to prove the desired result on a single sphere. 
For the rest of the proof, 
we will also drop the time variable $t$ and simply write 
$p(y,x)$ and $p(x,y)$ instead. 

We will start by considering the case when 
there is a unique shortest geodesic connecting $x$ and $y$. 
Due to Kolmogorov characterization of reversible diffusions
\citep[Extend Proposition 4.5 to $M$]{pavliotis2014stochastic}, 
we have the following identity 
\begin{equation}
	p(y,x) = p(x,y) \,, 
	\quad 
	\forall \, x,y \in M \,, 
\end{equation}
therefore it is equivalent to prove 
\begin{equation}
	\grad_y p(x,y) = - P_{x,y} \grad_x p(y, x) \,. 
\end{equation}

Let us view $S^d$ with center $x$, 
such that all points $y \in S^d$ can be seen as 
having a radial component $r = d_g(x, y)$ 
and an angular component $\theta \in T_x S^d$ 
such that $|\theta|_g = 1$. 
Next we observe that since $S^d$ is symmetrical 
in all angular directions, 
the vector $v = \grad_y p(x,y)$ must be 
in the radial direction. 

This implies that $r_t := \exp(y, tv)$ is a geodesic, 
and $r_\tau = x$ for some $\tau > 0$ 
(in fact $\tau$ is either $d_g(y,x) / |v|_g$ 
or $2\pi - d_g(y,x) / |v|_g$). 
Similarly, we can view $S^d$ with center $y$, 
and let $\wt{v} = \grad_x p(y,x)$, 
such that $\wt{r}_t := \exp(x, t \wt{v})$ 
is also a geodesic with $\wt{r}_\tau = y$. 
Due to symmetry of $x,y$, 
we must also have $r_t = \wt{r}_{\tau-t}$ for all $t \in [0,\tau]$. 
Since the time derivative of a geodesic is 
parallel transported along its path, 
we can write 
\begin{equation}
	v 
	= \de_t r_t |_{t=0} 
	= - \de_t \wt{r}_t|_{t = \tau} 
	= - P_{\wt{r}} \wt{v} \,, 
\end{equation}
where we use $P_{\wt{r}}$ to denote the parallel transport 
along the path $\{\wt{r}_t\}_{t=0}^\tau$. 
To get the desired result, it is sufficient to use the fact that 
all parallel transports along geodesics to the same end point 
are equivalent on $S^d$, 
which implies $P_{\wt{r}} \wt{v} = P_{x,y} \wt{v}$. 

Finally, we will consider $x,y$ in the cut locus of each other, 
in other words, $x,y$ are on the opposite poles of $S^d$. 
In this case, we observe that the gradient must remain radial, 
but symmetrical in all angular directions. 
Therefore we must have 
\begin{equation}
	\grad_x p(y,x) = 0 \,, 
\end{equation}
which also gives us the desired result. 

\end{proof}

\begin{lemma}
[Fokker--Planck Equation of the Discretization]
\label{lm:disc_fokker_planck}
Let $\{X_t\}_{t\geq 0}$ be the continuous time representation of 
the Langevin algorithm defined in \cref{eq:cont_time_rep}, 
and let $\rho_{t|0}(x|x_0)$ be 
the conditional density of $X_t|X_0 = x_0$. 
Then we have that $\rho_{t|0}$ satisfies the following 
Fokker--Planck equation 
\begin{equation}
\label{eq:disc_fokker_planck}
	\de_t \rho_{t|0}(x|x_0) 
	= \frac{1}{\beta} \Delta_x \rho_{t|0}(x|x_0) 
		+ \left\langle 
			\grad_x \rho_{t|0}(x|x_0),
			P_{\gamma_{t}(x_0),x} P_{x_0, \gamma_t(x_0)}
			\grad F(x_0)
		\right\rangle_g \,,
\end{equation}
where we add a subscript $\grad_x$ to denote the 
gradient with respect to the $x$ variable as needed, 
and we use $P_{x,y}$ to denote the parallel transport 
along the unique shortest geodesic between $x,y$ when it exists 
and zero otherwise. 
\end{lemma}

\begin{proof}

We will directly compute the time derivative 
\begin{equation}
\begin{aligned}
	\de_t \rho_{t|0}(x|x_0)
	&= \frac{d}{dt} p(t, \gamma_t(x_0), x) \\
	&= \de_t p(t, \gamma_t(x_0), x) + 
		\langle 
			\grad_{\gamma_t(x_0)} p(t, \gamma_t(x_0), x), 
			\de_t \gamma_t(x_0)
		\rangle_g \,, 
\end{aligned}
\end{equation}
where we use $\frac{d}{dt}$ to denote the time derivative 
to all variables depending on $t$, 
and $\de_t$ to denote the derivative with respect to 
the first variable in $p(t,y,x)$. 

Here we observe that 
\begin{equation}
	\de_t p(t, \gamma_t(x_0), x) 
	= \frac{1}{\beta} \Delta_x p(t, \gamma_t(x_0), x) \,, 
\end{equation}
due to being the density of the standard Brownian motion. 

Next we will observe that since $\gamma_t(x_0)$ 
is a geodesic, we simply have 
$\de_t \gamma_t(x_0) = P_{x_0, \gamma_t(x_0)} (-\grad F(x_0))$. 
Furthermore, we can use \cref{lm:grad_symmetry} 
to write 
\begin{equation}
	\grad_{\gamma_t(x_0)} p(t, \gamma_t(x_0), x)
	= - P_{x, \gamma_t(x_0)}
		\grad_x p(t, \gamma_t(x_0), x) \,. 
\end{equation}

Then using the fact that parallel transport preserves 
the Riemannian inner product, we have that 
\begin{equation}
\begin{aligned}
	& \langle 
		\grad_{\gamma_t(x_0)} p(t, \gamma_t(x_0), x), 
		\de_t \gamma_t(x_0)
	\rangle_g \\
	&= 
	\langle 
		- P_{x, \gamma_t(x_0)}
		\grad_x p(t, \gamma_t(x_0), x), 
		P_{x_0, \gamma_t(x_0)} (-\grad F(x_0))
	\rangle_g \\
	&= 
	\langle 
		\grad_x p(t, \gamma_t(x_0), x), 
		P_{\gamma_t(x_0), x} 
		P_{x_0, \gamma_t(x_0)} \grad F(x_0)
	\rangle_g \,. 
\end{aligned}
\end{equation}

Finally, rewriting $p(t, \gamma_t(x_0), x)$ 
as $\rho_{t|0}(x|x_0)$ gives us the desired 
Fokker--Planck equation. 

\end{proof}

\subsection{De Bruijn's Identity for the Discretization}
\label{subsec:de_bruijn_proof}

We will next derive the de Bruijn type identity 
for the discretization, 
which we first stated in \cref{lm:de_bruijn_disc_main}. 
Here we recall that the Markov triple $(M,\nu,\Gamma)$ 
is defined as $\nu = \frac{1}{Z} e^{-\beta F}$ 
and $\Gamma(\phi) = \frac{1}{\beta} | \grad \phi |_g^2$. 
We will also define $\rho_t(x)$ as the density of $X_t$, 
and let $h = \rho_t / \nu$. 
We further recall the KL divergence and 
Fisher information are defined as 
\begin{equation}
\begin{aligned}
	H_\nu(\rho_t) &= \int_M h \log h \, d\nu 
	= 
		\int_M \log \frac{\rho_t(x)}{ \nu(x) } \, \rho_t(x) \, dV_g(x) 
		\,, \\
	I_\nu(\rho_t) 
	&= 
		\int_M \frac{ \Gamma( h ) }{ h } \, d\nu 
	= 
		\frac{1}{\beta} 
		\int_M \left| 
			\grad \log \frac{ \rho_t(x) }{ \nu(x) }
			\right|_g^2 \,
			\rho_t(x) \, dV_g(x) \,, 
\end{aligned}
\end{equation}
where we recall that $dV_g$ is the Riemannian volume form.

\begin{lemma}
[De Bruijn's Identity for the Discretization]
\label{lm:de_bruijn_disc}
Let $\{X_t\}_{t\geq 0}$ be the continuous time representation 
of the Langevin algorithm defined in \cref{eq:cont_time_rep}, 
and let $\rho_{t}(x)$ be the density of $X_t$. 
Then we have the following de Bruijn type identity 
\begin{equation}
\label{eq:de_bruijn_disc}
	\de_t H_\nu( \rho_t ) 
	= - I_\nu(\rho_t) 
	+ \mathbb{E} \left\langle
		\grad F(X_t) - 
		b(t, X_0, X_t) \,, \,
		\grad \log \frac{\rho_t(X_t)}{\nu(X_t)}
	\right\rangle_g 
	- \mathbb{E} \div_{X_t} b(t,X_0,X_t) 
		\log \frac{\rho_t(X_t)}{\nu(X_t)}  
	\,, 
\end{equation}
where we define 
\begin{equation}
	b(t, x_0, x) := P_{\gamma_t(x_0), x} 
		P_{x_0, \gamma_t(x_0)} \grad F(x_0) \,, 
\end{equation}
and recall $P_{x,y}:T_x M \to T_y M$ is 
the parallel transport map along the unique shortest geodesic 
connecting $x,y$ when it exists, 
and zero otherwise. 

\end{lemma}

\begin{proof}

Before we compute the time derivative of 
KL divergence, 
we will first compute $\de_t \rho_t$ 
using \cref{lm:disc_fokker_planck} 
\begin{equation}
\begin{aligned}
	\de_t \rho_t(x) 
	&= \de_t \int_M \rho_{t|0}(x|x_0) \rho_0(x_0) \, dV_g(x_0) \\
	&= \int_M 
		\left[
			\frac{1}{\beta} \Delta_x \rho_{t|0}(x|x_0) 
			+ \left\langle 
				\grad_x \rho_{t|0}(x|x_0),
				b(t, x_0, x)
			\right\rangle_g
		\right] \, \rho_0(x_0) \, dV_g(x_0) \\ 
	&= \frac{1}{\beta} \Delta_x \int_M \rho_{t|0}(x|x_0) 
		\rho_0(x_0) \, dV_g(x_0)
		+ \int_M \left\langle 
			\grad_x \rho_{t|0}(x|x_0) \rho_0(x_0),
			b(t, x_0, x) 
		\right\rangle_g \, dV_g(x_0) \\
	&= \frac{1}{\beta} \Delta_x \rho_t(x)
		+ \int_M \left\langle 
			\grad_x \rho_{t0}(x,x_0),
			b(t, x_0, x) 
		\right\rangle_g \, dV_g(x_0) \,, 
\end{aligned}
\end{equation}
where we denote $\rho_{t0}(x,x_0)$ 
as the joint density of $(X_t, X_0)$. 

Returning to the time derivative of KL divergence, 
we will compute it explicitly using the previous equation 
\begin{equation}
\begin{aligned}
	\de_t H_\nu( \rho_t ) 
	&= \de_t \int_M \rho_t \log \frac{\rho_t}{\nu} \, dV_g(x) \\
	&= \int_M \de_t \rho_t \log \frac{\rho_t}{\nu} 
		+ \rho_t \frac{\de_t \rho_t}{\rho_t} 
		- \rho_t \de_t \log \nu \, dV_g(x) 
		\\
	&= \int_M \de_t \rho_t \log \frac{\rho_t}{\nu} \, dV_g(x) 
		+ 
		\de_t \int_M \rho_t \, dV_g(x) \\
	&= \int_M \left[ 
		\frac{1}{\beta} \Delta_x \rho_t(x)
		+ \int_M \left\langle 
			\grad_x \rho_{t0}(x,x_0),
			b(t, x_0, x) 
		\right\rangle_g \, dV_g(x_0) 
		\right] 
	\log \frac{\rho_t}{\nu} dV_g(x) \\
	&= \int_M 
		\frac{1}{\beta} \Delta_x \rho_t(x) 
		\log \frac{\rho_t}{\nu} dV_g(x) 
		+ \int_M \int_M \left\langle 
			\grad_x \rho_{t0}(x,x_0),
			b(t, x_0, x) 
		\right\rangle_g \, dV_g(x_0) 
	\log \frac{\rho_t}{\nu} dV_g(x) 
	\,. 
\end{aligned}
\end{equation}

Here we will treat the two integrals separately 
using integration by parts (divergence theorem). 
For the first integral, 
we have that 
\begin{equation}
	\int_M 
		\frac{1}{\beta} \Delta_x \rho_t(x) 
		\log \frac{\rho_t}{\nu} dV_g(x) 
	= 
		\int_M 
			- \frac{1}{\beta} 
			\left\langle \grad \rho_t(x) , 
				\grad \log \frac{\rho_t}{\nu} 
			\right\rangle_g 
		dV_g(x) \,. 
\end{equation}

Next we will use the fact that 
\begin{equation}
\begin{aligned}
	\rho_t(x) 
		\left[ 
			\frac{1}{\beta} \grad \log \frac{\rho_t}{\nu}
			- \grad F(x)
		\right] 
	&= 
		\rho_t 
		\left[ 
			\frac{1}{\beta} \frac{ \grad  \rho_t}{ \rho_t }
			- \frac{1}{\beta} \frac{ \grad \nu }{ \nu } 
			- \grad F(x)
		\right] \\ 
	&= 
		\rho_t 
		\left[ 
			\frac{1}{\beta} \frac{ \grad  \rho_t}{ \rho_t }
			+ \grad F(x) 
			- \grad F(x) 
		\right] \\ 
	&= 
		\frac{1}{\beta} \grad \rho_t(x) \,, 
\end{aligned}
\end{equation}
which allows us to write 
\begin{equation}
\begin{aligned}
	\int_M 
		- \frac{1}{\beta} 
		\left\langle \grad \rho_t(x) , 
			\grad \log \frac{\rho_t}{\nu} 
		\right\rangle_g 
		dV_g(x) 
	&= \int_M 
		- \rho_t(x) 
		\left\langle \frac{1}{\beta} \grad \log \frac{\rho_t}{\nu}
			- \grad F(x) , 
			\grad \log \frac{\rho_t}{\nu} 
		\right\rangle_g 
		dV_g(x) \\ 
	&= - I_\nu(\rho_t)
		+ \int_M 
		\rho_t(x) 
		\left\langle 
			\grad F(x) , 
			\grad \log \frac{\rho_t}{\nu} 
		\right\rangle_g 
		dV_g(x) \\ 
	&= 
		- I_\nu(\rho_t)
		+ \mathbb{E} 
			\left\langle 
				\grad F(X_t) , 
				\grad \log \frac{\rho_t(X_t) }{\nu(X_t)} 
			\right\rangle_g 
		\,. 
\end{aligned}
\end{equation}

For the second integral, 
we will also use integration by parts to write 
\begin{equation}
\begin{aligned}
	& \int_M \int_M \left\langle 
			\grad_x \rho_{t0}(x,x_0),
			b(t, x_0, x) 
		\right\rangle_g \, dV_g(x_0) 
	\log \frac{\rho_t}{\nu} dV_g(x) \\ 
	&= 
	- \int_M \int_M 
			\rho_{t0}(x,x_0) 
		\div_x \left[
			b(t, x_0, x) 
			\log \frac{\rho_t}{\nu}
		\right] \, dV_g(x_0) 
		dV_g(x) \\ 
	&= - \int_M \int_M 
			\rho_{t0}(x,x_0) 
		\left\langle 
			b(t, x_0, x) , 
			\grad \log \frac{\rho_t(x)}{\nu(x)}
		\right\rangle_g 
		+ \rho_{t0}(x,x_0) \div_x ( b(t,x_0,x) ) 
		\log \frac{\rho_t}{\nu}
		\, dV_g(x_0) 
		dV_g(x) \\ 
	&= \mathbb{E} 
		\left\langle - b(t, X_0, X_t), 
			\grad \log \frac{\rho_t(X_t)}{\nu(X_t)}
		\right\rangle_g 
		- \mathbb{E} 
			\div_{X_t} ( b(t,X_0,X_t) ) 
			\log \frac{\rho_t(X_t)}{\nu(X_t)} \,. 
\end{aligned}
\end{equation}

Finally, by merging the two integral terms together, 
we get the desired identity. 

\end{proof}

\subsection{Bounding the Inner Product Term}

Here we will first adapt a Talagrand's inequality based on 
the expected gradient from \cite[Lemma 11,12]{vempala2019rapid}. 

\begin{lemma}
\label{lm:exp_grad_bound}
Suppose $(M, \nu, \Gamma)$ satisfies $\LSI(\alpha)$ and 
$\grad F$ is $K_2$-Lipschitz. 
Then for all probability distributions $\rho$, 
we have the bound 
\begin{equation}
	\int \, |\grad F(x)|_g^2 \, d\rho(x) 
	\leq 
		\frac{4K_2}{\alpha} H_\nu(\rho) + \frac{2ndK_2}{\beta} \,. 
\end{equation}
\end{lemma}

\begin{proof}

Firstly, we observe that by integration by parts 
\begin{equation}
	\int \Delta F \, d\nu = - \int \langle \grad F, \grad \nu \rangle_g d\Vol 
	= \beta \int \, |\grad F|_g^2 d\nu \,. 
\end{equation}
Combined with $K_2$-Lipschitzness we get the bound of 
\begin{equation}
	\int \, |\grad F|_g^2 d\nu \leq \frac{nd K_2}{\beta} \,. 
\end{equation}

At the same time, recall that we denote $P_{x,y}: T_x M \to T_y M$ 
as the parallel transport via the unique shortest geodesic if it exists, 
we can write 
\begin{equation}
	| \grad F(x) |_g 
	\leq | \grad F(x) - P_{y,x} \grad F(y) |_g 
			+ | P_{y,x} \grad F(y) |_g 
	\leq K_2 d_g(x,y) + |\grad F(y)|_g \,. 
\end{equation}

Now let $\mu(x,y)$ denote the optimal $W_2$ coupling of $\rho(x),\nu(y)$, 
and using $ab \leq \frac{a^2}{2} + \frac{b^2}{2}$, 
we can write 
\begin{equation}
\begin{aligned}
	\int \, |\grad F(x)|_g^2 \, d \rho(x) 
	&= \int \, |\grad F(x)|_g^2 \, d \mu(x,y) \\ 
	&\leq \int \, \left( K_2 d_g(x,y) + |\grad F(y)|_g \right)^2 d \mu(x,y) \\ 
	&\leq 
		2 K_2^2 W_2(\rho, \nu)^2 
		+ 2 \int \, |\grad F(y)|_g^2 \, d\nu(y) \,, 
\end{aligned}
\end{equation}
where we used the previous inequality on $|\grad F(x)|_g$. 

Using the fact that $\LSI(\alpha)$ implies the transport inequality 
$W_2(\rho,\nu)^2 \leq \frac{2}{\alpha} H_\nu(\rho)$, 
and the previous expected gradient bound, 
we get the desired result 
\begin{equation}
	2 K_2^2 W_2(\rho, \nu)^2 
		+ 2 \int \, |\grad F(y)|_g^2 \, d\nu(y) 
	\leq 
		\frac{4 K_2^2}{\alpha} H_\nu(\rho) 
		+ \frac{2nd K_2}{\beta} \,. 
\end{equation}

\end{proof}

\begin{lemma}
\label{lm:inner_prod_term_bound}
Let $\{X_t\}_{t\geq 0}$ be the continuous time representation 
of the Langevin algorithm defined in \cref{eq:cont_time_rep}. 
Suppose further that $(M,\nu,\Gamma)$ satisfies $\LSI(\alpha)$ 
and $\grad F$ is $K_2$-Lipschitz, 
then we have that 
\begin{equation}
	\mathbb{E} \left\langle
		\grad F(X_t) - 
		b(t, X_0, X_t) \,, \,
		\grad \log \frac{\rho_t(X_t)}{\nu(X_t)}
	\right\rangle_g 
	\leq 
		\frac{1}{8} I_\nu(\rho_t) 
		+ 8 tnd K_2^2 ( 1 + t K_2 ) 
		+ \frac{ 16 \beta t^2 K_2^4 }{\alpha} H_\nu(\rho_0) \,, 
\end{equation}
where we define 
\begin{equation}
	b(t, x_0, x) := P_{\gamma_t(x_0), x} 
		P_{x_0, \gamma_t(x_0)} \grad F(x_0) \,, 	
\end{equation}
and recall $P_{x,y}:T_x M \to T_y M$ is 
the parallel transport map along the unique shortest geodesic 
connecting $x,y$ when it exists, 
and zero otherwise. 

\end{lemma}

\begin{proof}

We start by using the Cauchy-Schwarz and Young's inequality to write
\begin{equation}
	\langle a, b \rangle_g 
	\leq 
		|a|_g |b|_g 
	\leq 
		2 \beta |a|_g^2 + \frac{1}{8 \beta} |b|_g^2 \,. 
\end{equation}

Applying this to the inner product term, we can get 
\begin{equation}
	\mathbb{E} \left\langle
		\grad F(X_t) - 
		b(t, X_0, X_t) \,, \,
		\grad \log \frac{\rho_t(X_t)}{\nu(X_t)}
	\right\rangle_g  
	\leq 
		2\beta \, \mathbb{E} \, 
			|\grad F(X_t) - b(t, X_0, X_t)|_g^2 \,
		+ \frac{1}{8} I_\nu(\rho_t) \,, 
\end{equation}
therefore, it is sufficient to study 
the first expectation term. 

To this end, we add an intermediate term to the gradient difference 
\begin{equation}
\begin{aligned}
	& \grad F(X_t) - b(t, X_0, X_t) \\
	&= \grad F(X_t) - 
		P_{\gamma_t(X_0), X_t} P_{X_0, \gamma_t(X_0)} 
		\grad F(x_0) \\
	&= \grad F(X_t) 
		- P_{\gamma_t(X_0), X_t} \grad F(\gamma_t(X_0)) \\
	&\quad + P_{\gamma_t(X_0), X_t} \grad F(\gamma_t(X_0))
		- P_{\gamma_t(X_0), X_t} P_{X_0, \gamma_t(X_0)} 
		\grad F(x_0) \,. 
\end{aligned}
\end{equation}

Then using triangle inequality and the fact that 
$\grad F$ is $K_2$-Lipschitz,
we can write 
\begin{equation}
\begin{aligned}
	&|\grad F(X_t) - b(t, X_0, X_t)|_g \\
	&\leq | \grad F(X_t) 
		- P_{\gamma_t(X_0), X_t} \grad F(\gamma_t(X_0)) |_g \\
	&\quad + | P_{\gamma_t(X_0), X_t} \grad F(\gamma_t(X_0))
		- P_{\gamma_t(X_0), X_t} P_{X_0, \gamma_t(X_0)} 
		\grad F(x_0) |_g \\ 
	&\leq K_2 d_g(\gamma_t(X_0), X_t) 
		+ K_2 d_g(X_0, \gamma_t(X_0)) \,, 
\end{aligned}
\end{equation}
where $d_g$ is the geodesic distance. 

Now we return to the original expectation term, 
we can write 
\begin{equation}
\begin{aligned}
	2\beta \, \mathbb{E} \, 
		|\grad F(X_t) - b(t, X_0, X_t)|_g^2
	&\leq 
		4 \beta K_2^2 \mathbb{E} 
		[ d_g(\gamma_t(X_0), X_t)^2 
		+ d_g(X_0, \gamma_t(X_0))^2 ] \,, 
\end{aligned}
\end{equation}
where we used the inequality $(a+b)^2 \leq 2a^2 + 2b^2$. 

Since the process between $\gamma_t(X_0)$ to $X_t$ 
is just a Brownian motion, 
we can use the radial comparison theorem 
(\cref{cor:conc_mean_comparison}) to write 
\begin{equation}
	\mathbb{E} d_g(\gamma_t(X_0), X_t)^2 
	\leq 
		\mathbb{E} \frac{2}{\beta} |B_t|^2 
	= 
		\frac{2t nd}{\beta} \,, 
\end{equation}
where $B_t$ is a standard $\mathbb{R}^{nd}$ Brownian motion. 

The second term is simply  
the path between $X_0$ and $\gamma_t(X_0)$ 
is a geodesic, and therefore we have that 
\begin{equation}
	\mathbb{E} d_g(X_0, \gamma_t(X_0))^2 
	= 
		\mathbb{E} |t \grad F(X_0)|_g^2
	\leq 
		t^2 \left( \frac{4K_2^2}{\alpha} H_\nu(\rho_0) + \frac{2ndK_2}{\beta} 
		\right) \,, 
\end{equation}
where we used the bound from \cref{lm:exp_grad_bound}. 

Putting it together, we can write 
\begin{equation}
\begin{aligned}
	2 \beta \mathbb{E} \, 
		|\grad F(X_t) - b(t, X_0, X_t)|_g^2 
	&\leq 
		4 \beta K_2^2 
		\left[ \frac{2t nd}{\beta} 
		+ t^2 \, \left( \frac{4K_2^2}{\alpha} H_\nu(\rho_0) + \frac{2ndK_2}{\beta} 
		\right)
		\right] \\ 
	&= 
		8 tnd K_2^2 ( 1 + t K_2 ) 
		+ \frac{ 16 \beta t^2 K_2^4 }{\alpha} H_\nu(\rho_0) 
		\,, 
\end{aligned}
\end{equation}
which gives us the desired bound. 

\end{proof}

\subsection{Bounding the Divergence Term}

Here we define the notation $x = (x^{(1)}, \cdots, x^{(n)} ) \in M$ 
where $x^{(\ell)}$ represents the coordinate in $\ell^\text{th}$ 
sphere $S^d$, 
and similarly write $v = (v^{(1)}, \cdots, v^{(n)} ) \in T_x M$ 
with $v^{(\ell)} \in T_{x^{(\ell)}} S^d$. 
We will also use $g'$ to denote the metric on a single sphere. 
This allows us to state the a technical result. 

\begin{lemma}
[Rotational Symmetry Identity]
\label{lm:stereo_div_symmetry_main}
Let $\{X_t\}_{t\geq 0}$ be the continuous time representation of 
the Langevin algorithm defined in \cref{eq:cont_time_rep}. 
Then we have that 
\begin{equation}
\begin{aligned}
	\mathbb{E} \left[ \left. 
			\div_{X_t} 
			b(t, X_0, X_t) \right| X_0 = x_0 \right] 
	&= 
		0 \,, \\ 
	\mathbb{E} \left[ \left. \left( \div_{X_t} 
			b(t, X_0, X_t) 
		\right)^2 \right| X_0 = x_0 \right]
	&= 
		| \grad F(x_0) |_{g}^2 \, 
		\left( \frac{2}{d} + 1 \right) 
		\mathbb{E} \tan\left( 
			\frac{1}{2} d_{g'}( \gamma_t(x_0)^{(i)}, X_t^{(i)} ) 
			\right)^2 \,, 
\end{aligned}
\end{equation}
where $i \in [n]$ is arbitrary. 

\end{lemma}

The full proof can be found in 
\cref{subsec:app_stereo} \cref{lm:stereo_div_symmetry}.

Now we will turn to controlling the quantity in expectation. 
The next result will use several a couple of well known results 
regarding the Wright--Fisher diffusion, 
which we include in \cref{subsec:app_wright_fisher}. 

\begin{lemma}
\label{lm:tan_L2_bound}
Let $\{X_t\}_{t\geq 0}$ be the continuous time representation of 
the Langevin algorithm defined in \cref{eq:cont_time_rep}. 
For all $t \geq 0$ and $d \geq 3$, 
we have the following bound  
\begin{equation}
	\mathbb{E} \tan\left( 
		\frac{1}{2} d_{g'}( \gamma_t(x_0)^{(i)}, X_t^{(i)} ) 
		\right)^2 
	\leq 
		\frac{4td}{\beta} \,. 
\end{equation}
\end{lemma}

\begin{proof}

We start by invoking \cref{lm:radial_wright_fisher} 
with the transformation of 
$Y_{T(t)} = \frac{1}{2} ( 1 - \cos d_{g'}( \gamma_t(x_0)^{(i)}, X_t^{(i)} ) )$ 
and the time change $T(t) = 2t/\beta$, 
which implies $Y_T$ is a Wright--Fisher diffusion satisfying the SDE 
\begin{equation}
	dY_T = \frac{d}{4} ( 1 - 2Y_T) \, dT 
		+ \sqrt{ Y_T (1 - Y_T) } \, dB_T \,, 
\end{equation}
where $\{B_t\}_{t\geq 0}$ is a standard Brownian motion in $\mathbb{R}$. 

Furthermore, we also have that by the tangent double angle formula, 
and basic trignometry 
\begin{equation}
\begin{aligned}
	\tan\left( 
		\frac{1}{2} d_{g'}( \gamma_t(x_0)^{(i)}, X_t^{(i)} ) 
		\right)^2 
	&= 
		\tan\left( 
		\frac{1}{2} \arccos( 1 - 2Y_T )
		\right)^2 
		\\ 
	&= 
		\frac{ 1 - \cos \circ \arccos( 1 - 2Y_T ) }{ 
			1 + \cos \circ \arccos( 1 - 2Y_T ) } 
		\\ 
	&= 
		\frac{Y_T}{ 1 - Y_T} \,. 
\end{aligned}
\end{equation}

Using It\^{o}'s Lemma, we have that 
\begin{equation}
	\frac{Y_T}{1 - Y_T} 
	= 
		\frac{Y_0}{1 - Y_0} 
		+ 
			\int_0^T \, 
			\frac{ \frac{d}{4} (1 - 2Y_s) + \frac{1}{2} Y_s }{ (1 - Y_s)^2 }
			\, ds 
		+ 
			\int_0^T \, 
			\frac{ \sqrt{ Y_s (1-Y_s) } }{ (1 - Y_s)^2 }
			\, dB_s \,, 
\end{equation}
and since $Y_0 = 0$ and the It\^{o} integral is a martingale, 
we can write the desired expectation as 
\begin{equation}
	\mathbb{E} \tan\left( 
		\frac{1}{2} d_{g'}( \gamma_t(x_0)^{(i)}, X_t^{(i)} ) 
		\right)^2 
	= 
		\mathbb{E} \int_0^T \, 
			\frac{ \frac{d}{4} (1 - 2Y_s) + \frac{1}{2} Y_s }{ (1 - Y_s)^2 }
			\, ds \,. 
\end{equation}

We further observe that when $d \geq 3$, 
we can upper bound the integrand as 
\begin{equation}
	\frac{ \frac{d}{4} (1 - 2Y_s) + \frac{1}{2} Y_s }{ (1 - Y_s)^2 } 
	= 
		\frac{d-1}{2} \frac{1}{1 - Y_s} 
		+ \frac{2-d}{4} \frac{1}{(1 - Y_s)^2} 
	\leq 
		\frac{d-1}{2} \frac{1}{1 - Y_s} \,, 
\end{equation}
and since the integrand is positive, 
we can also exchange the order of expectation and integration 
using Tonelli's theorem to write 
\begin{equation}
	\mathbb{E} \tan\left( 
		\frac{1}{2} d_{g'}( \gamma_t(x_0)^{(i)}, X_t^{(i)} ) 
		\right)^2 
	\leq 
		\int_0^T \, 
			\mathbb{E} \, \frac{d-1}{2} \frac{1}{1 - Y_s} 
			\, ds \,. 
\end{equation}

At this point we will invoke \cref{thm:wf_transition_density} 
to write the expectation as an integral 
\begin{equation}
	\mathbb{E} \, \frac{1}{1-Y_s} 
	= 
		\int_0^1 
			\frac{1}{1-y}
			\sum_{m=0}^\infty q_m(s) \, 
			\frac{ y^{d/2 - 1} (1-y)^{d/2 + m - 1} }{ 
				B(d/2, d/2 + m) } \, dy \,, 
\end{equation}
where $\{q_m(s)\}_{m\geq 0}$ is a probability distribution 
over $\mathbb{N}$ for each $s \geq 0$ 
and $B(\theta_1, \theta_2)$ is the beta function. 

Since the integrand is positive, 
we will once again exchange the order of integration and sum 
to compute the beta integral 
\begin{equation}
\begin{aligned}
	\mathbb{E} \, \frac{1}{1-Y_s} 
	&= 
		\sum_{m=0}^\infty q_m(s) \, 
			\int_0^1 \, 
				\frac{ y^{d/2 - 1} (1-y)^{d/2 + m - 2} }{ 
				B(d/2, d/2 + m) } \, dy 
				\\ 
	&= 
		\sum_{m=0}^\infty q_m(s) \, 
			\frac{ B(d/2, d/2 + m - 1) }{ 
				B(d/2, d/2 + m) } 
				\\ 
	&= 
		\sum_{m=0}^\infty q_m(s) \, 
			\frac{ d-1 }{ d/2 - 1 } \,. 
\end{aligned}
\end{equation}

Using the fact that $d \geq 3$ once again, we have 
\begin{equation}
	\sum_{m=0}^\infty q_m(s) \, 
		\frac{ d-1 }{ d/2 - 1 } 
	= 
		\sum_{m=0}^\infty q_m(s) 
		\frac{2}{ 1 - 1/(d-1) } 
	\leq 
		4 \,,
\end{equation}
which implies 
\begin{equation}
	\mathbb{E} \tan\left( 
		\frac{1}{2} d_{g'}( \gamma_t(x_0)^{(i)}, X_t^{(i)} ) 
		\right)^2 
	\leq 
		\int_0^{2t/\beta} \, 
			2(d-1) \, ds 
	= 
		\frac{4(d-1)t}{\beta} \,, 
\end{equation}
which is the desired result. 

\end{proof}

\begin{corollary}
\label{cor:div_l2_bound}
Let $\{X_t\}_{t\geq 0}$ be the continuous time representation of 
the Langevin algorithm defined in \cref{eq:cont_time_rep}. 
For all $t \geq 0$ and $d \geq 3$, 
we have the following bound 
\begin{equation}
	\mathbb{E} (\div b(t, X_0, X_t) )^2 
	\leq 		
		\frac{8td}{\beta} 
		\mathbb{E} |\grad F(X_0)|_g^2 \,. 
\end{equation}
\end{corollary}

\begin{proof}

We simply combine the results of 
\cref{lm:stereo_div_symmetry_main,lm:tan_L2_bound} 
to write 
\begin{equation}
\begin{aligned}
	\mathbb{E} (\div b(t, x_0, x) )^2 
	&= 
		\mathbb{E} | \grad F(x_0) |_g^2 
			\left( \frac{2}{d} + 1 \right)
		\mathbb{E} \tan\left( 
			\frac{1}{2} d_{g_\ell}(\gamma_t(x_0)_\ell, x_\ell) 
			\right)^2 
			\\ 
	&\leq 
		\left( \frac{2}{d} + 1 \right) 
		\frac{4td}{\beta} 
		\mathbb{E} | \grad F(x_0) |_g^2 
		\,, 
\end{aligned}
\end{equation}
and observe that when $d\geq 3$ 
we have that $2/d + 1 \leq 2$, 
which gives us the desired result. 

\end{proof}

\begin{proposition}
\label{prop:exp_div_term_bound}
Let $\{X_t\}_{t\geq 0}$ be the continuous time representation of 
the Langevin algorithm defined in \cref{eq:cont_time_rep}. 
Suppose $F$ satisfies \cref{asm:c2} 
and $(M,\nu,\Gamma)$ satisfies \cref{asm:lsi}. 
Then for all $t\geq 0$ and $d \geq 3$, 
we have the following bound 
\begin{equation}
	- \mathbb{E} \div b(t, X_0, X_t) 
		\log \frac{ \rho_t(X_t) }{ \nu(X_t) }
	\leq 
		\frac{1}{8} I_{\nu}(\rho_t) 
		+ \frac{128 t^2 d K_2^2}{\alpha \beta} H_\nu(\rho_0)
		+ \frac{64 t^2 nd^2 K_2}{\beta^2} 
			\,. 
\end{equation}
\end{proposition}

\begin{proof}

We start by observing that 
\begin{equation}
	\mathbb{E} \log \frac{ \rho_t(X_t) }{ \nu(X_t) } 
	= H_\nu( \rho_t ) \,, 
\end{equation}
then we can use the fact 
$\mathbb{E} \div b(t, X_0, X_t) = 0$ from 
\cref{lm:stereo_div_symmetry_main} 
to write 
\begin{equation}
\begin{aligned}
	\mathbb{E} \div b(t, X_0, X_t) 
		\log \frac{ \rho_t(X_t) }{ \nu(X_t) }
	&=  
		\mathbb{E} \div b(t, X_0, X_t) 
			\left( \log \frac{ \rho_t(X_t) }{ \nu(X_t) } - H_\nu( \rho_t ) 
			\right) 
		+ 
			H_\nu( \rho_t ) \, 
			\mathbb{E} \div b(t, X_0, X_t) 
			\\ 
	&= 
		\mathbb{E} \div b(t, X_0, X_t) 
			\left( \log \frac{ \rho_t(X_t) }{ \nu(X_t) } - H_\nu( \rho_t ) 
			\right) 
			\,. 
\end{aligned}
\end{equation}

We will then split the first term into two more 
using Young's inequality 
\begin{equation}
	- \mathbb{E} \div b(t, X_0, X_t) 
		\left( 
			\log \frac{ \rho_t(X_t) }{ \nu(X_t) }
				- H_\nu( \rho_t ) 
		\right) 
	\leq 
		\frac{\epsilon}{2} 
			\mathbb{E} \div b(t,X_0,X)^2 
		+ \frac{1}{2\epsilon}
			\mathbb{E} \left(
				\log \frac{ \rho_t(X_t) }{ \nu(X_t) } - H_\nu( \rho_t ) 
				\right)^2 \,, 
\end{equation}
where $\epsilon > 0$ will be chosen later. 

Using \cref{cor:div_l2_bound}, we can control the first term as 
\begin{equation}
	\frac{\epsilon}{2} \mathbb{E} \div b(t,X_0,X)^2 
	\leq 
		\frac{\epsilon}{2} \frac{8td}{\beta} 
			\mathbb{E} | \grad F(x_0) |_g^2  
	\leq 
		\frac{\epsilon}{2} \frac{8td}{\beta} 
			\left( 
				\frac{4 K_2^2}{\alpha} H_\nu(\rho_0) 
				+ \frac{2nd K_2}{\beta} 
			\right) \,, 
\end{equation}
where we used the bound from \cref{lm:exp_grad_bound}. 

For the second term, we start be rewriting expectation 
as a double integral over the conditional density first 
\begin{equation}
	\int_M f(x) \, \rho_t(x) \, dV_g(x)
	= \int_M \int_M f(x) \, \rho_{t|0}(x | x_0) \, dV_g(x) \, dV_g(x_0)
	\,. 
\end{equation}

This way we can use the fact that 
$\rho_{t|0}(x|x_0) = p(t, \gamma_t(x_0), x)$ 
is the density of a Brownian motion starting at $\gamma_t(x_0)$.  
Hence using \cref{thm:local_poincare}, 
$\rho_{t|0}(x|x_0)$ satisfies a local Poincar\'{e} inequality 
with constant $\frac{1}{2t}$. 
More precisely, we have that 
\begin{equation}
	\mathbb{E} \left(
				\log \frac{ \rho_t(X_t) }{ \nu(X_t) } - H_\nu( \rho_t ) 
				\right)^2
	\leq 
		\frac{ 2t }{\beta}
		\mathbb{E} \left(
				\grad \log \frac{ \rho_t(X_t) }{ \nu(X_t) }
				\right)_g^2 
		= 2t \, I_\nu(\rho_t) \,. 
\end{equation}

To get the desired coefficient on $I_\nu(\rho_t)$, 
we can simply choose $\epsilon = 8t$, 
and this gives us the bound 
\begin{equation}
	- \mathbb{E} \div b(t, X_0, X_t) 
		\left( 
			\log \frac{ \rho_t(X_t) }{ \nu(X_t) }
				- H_\nu( \rho_t ) 
		\right) 
	\leq 
		\frac{32 t^2 d}{\beta} 
			\left( 
				\frac{4 K_2^2}{\alpha} H_\nu(\rho_0) 
				+ \frac{2nd K_2}{\beta} 
			\right) 
		+ \frac{1}{8} I_{\nu}(\rho_t) \,, 
\end{equation}
which is the desired result. 

\end{proof}

\subsection{Main Result - KL Divergence Bounds}
\label{subsec:conv_kl_bounds}

\begin{theorem}
[One Step KL Divergence Bound]
\label{thm:one_step_KL_bound}
Let $\{X_t\}_{t\geq 0}$ be the continuous time representation of 
the Langevin algorithm defined in \cref{eq:cont_time_rep}. 
Assume $F$ satisfies \cref{asm:c2}
and that $(M,\nu,\Gamma)$ satisfies \cref{asm:lsi}. 
Then for all $0 \leq t \leq \min\left( \frac{2}{3\alpha} , 
\frac{\alpha}{ 24 K_2 \sqrt{(\beta+d)d} } \right)$ 
and $d \geq 3$, 
we have the following growth bound on the KL divergence 
for $\rho_t := \mathcal{L}(X_t)$ 
\begin{equation}
	H_\nu( \rho_t ) 
	\leq 
		e^{-\alpha t} H_\nu( \rho_0 ) 
		+ 
		30 nd K_2^2 t^2 
		\,. 
\end{equation}
\end{theorem}

\begin{proof}

We start by writing down the de Bruijn's identity 
for the discretization process 
from \cref{lm:de_bruijn_disc_main} 
\begin{equation}
	\de_t H_\nu( \rho_t ) 
	= 
	- I_\nu(\rho_t) 
	+ \mathbb{E} \left\langle
		\grad F(X_t) - 
		b(t, X_0, X_t) \,, \,
		\grad \log \frac{\rho_t(X_t)}{\nu(X_t)}
	\right\rangle_g 
	- \mathbb{E} \div b(t,X_0,X_t) 
		\log \frac{\rho_t(X_t)}{\nu(X_t)} \,. 
\end{equation}

Then using the results \cref{lm:inner_prod_term_bound} 
and \cref{prop:exp_div_term_bound}, 
we can get the following bound 
\begin{equation}
\begin{aligned}
	\de_t H_\nu( \rho_t ) 
	&\leq 
		- I_\nu(\rho_t) \\
	&\quad 
		+ \frac{1}{8} I_\nu(\rho_t) 
		+ 16 \frac{\beta t^2 K_2^2}{\alpha} H_\nu(\rho_0) 
		+ 8 tnd K_2^2 (1 + tK_2) \\ 
	&\quad 
		+ \frac{1}{8} I_{\nu}(\rho_t) 
		+ 128 \frac{t^2 d K_2^2}{\alpha \beta} H_\nu(\rho_0) 
		+ 64 \frac{t^2nd^2 K_2}{\beta^2} 
			\,. 
\end{aligned}
\end{equation}

Plugging in the Sobolev inequality for $\nu$
and $\beta \geq 1$, 
we can get 
\begin{equation}
\begin{aligned}
	\de_t H_\nu( \rho_t ) 
	&\leq 
		- \frac{3}{2} \alpha 
			H_\nu( \rho_t ) 
		+ t ( 8ndK_2^2  ) 
		+ t^2 (8 K_2^2 + 64 d)  nd K_2
		+ t^2 \frac{K_2^2}{\alpha} \left( 16 \beta + 128 d 
		\right) H_\nu(\rho_0) 
		\\ 
	&=: 
		- a H_\nu(\rho_t) 
		+ c_1 t 
		+ c_2 t^2 
		+ c_3 t^2 H_\nu(\rho_0) 
		\,. 
\end{aligned}
\end{equation}

We observe the above is a Gr\"{o}nwall type 
differential inequality, 
and using a Gr\"{o}nwall type argument 
in \cref{lm:gronwall}, 
we can get the following bound 
\begin{equation}
\begin{aligned}
	H_\nu( \rho_t ) 
	&\leq 
		e^{-a t} H_\nu( \rho_0 ) 
		+ 
		\int_0^t \, e^{a(s-t)} \, (c_1 s + c_2 s^2
			+ c_3 s^2 H_\nu(\rho_0) ) \, ds 
		\\ 
	&\leq  
		e^{-a t} H_\nu( \rho_0 ) 
		+ 
		(c_1 t + c_2 t^2 + c_3 t^2 H_\nu(\rho_0) ) 
		\int_0^t \, e^{a(s-t)} \, ds 
		\\ 
	&= 
		e^{-a t} H_\nu( \rho_0 ) 
		+ (c_1 t + c_2 t^2 + c_3 t^2 H_\nu(\rho_0) ) 
		\frac{1-e^{-at}}{a} 
		\\ 
	&= 
		e^{-at} \left( 1 + \frac{e^{at} - 1}{a} c_3 t^2 \right) H_\nu(\rho_0) 
		+ (c_1 t + c_2 t^2) \frac{1-e^{-at}}{a} 
		\,. 
\end{aligned}
\end{equation}

There are a few simplifications in place. 
Firstly, we have that $1 - e^{-at} \leq at$ for all $t\geq 0$. 
Then using the fact that $t \leq \frac{2}{3\alpha} = \frac{1}{a}$, 
we also have that $e^{at} \leq 1 + 2at$, which gives us 
\begin{equation}
	H_\nu( \rho_t ) 
	\leq 
		e^{-\frac{3}{2} \alpha t}( 1 + 2 c_3 t^2 ) H_\nu(\rho_0) 
		+ c_1 t^2 + c_2 t^3 \,. 
\end{equation}

Now we will use the fact that $t \leq \frac{\alpha}{24 K_2 \sqrt{(\beta+d)d}}$ 
to get that $2c_3 t^2 = 2 \frac{K_2^2}{\alpha} \left( 16 \beta + 128 d \right) t^2 
\leq \frac{1}{2}\alpha t$, which implies the following 
\begin{equation}
	1 + 2 c_3 t^2 \leq 1 + \frac{\alpha t}{2} \leq e^{\frac{\alpha}{2} t} \,, 
\end{equation}
and at the same time we also have that 
\begin{equation}
	c_2 t^3 
	\leq 
		8 nd K_2^2 \left( \frac{K_2 + 64 d}{ 24 K_2 \sqrt{\beta} d } \right) t^2 
	\leq 
		22 nd K_2^2 t^2 \,. 
\end{equation}

Putting these together, we get the desired result of 
\begin{equation}
	H_\nu( \rho_t ) 
	\leq 
		e^{-\alpha t} H_\nu( \rho_0 ) 
		+ 
		30 nd K_2^2 t 
		\,. 
\end{equation}

\end{proof}

Here we will restate the main result, 
which was originally \cref{thm:finite_iteration_KL_bound}. 

\begin{theorem}
[Finite Iteration KL Divergence Bound]
Let $F$ satisfy \cref{asm:c2}, 
$(M,\nu,\Gamma)$ satisfy \cref{asm:lsi}, 
and suppose $d\geq 3$. 
Let $\{X_k\}_{k\geq 1}$ be the Langevin algorithm 
defined in \cref{eq:langevin_algorithm}, 
with initialization $\rho_0 \in C^1(M)$. 
If we choose $\beta \geq 1$ and 
$0 \leq \eta \leq \min\left( \frac{2}{3\alpha}, 
\frac{\alpha}{ 24 K_2 \sqrt{(\beta+d)d} } \right)$, 
then we have the following bound on the KL divergence 
of $\rho_k := \mathcal{L}(X_k)$ 
\begin{equation}
	H_\nu(\rho_k) 
	\leq 
		H_\nu(\rho_0) \,  
		e^{-\alpha k \eta} 
		+ 45 nd K_2^2 
		\frac{\eta}{\alpha} 
		\,. 
\end{equation}
\end{theorem}

\begin{proof}

From \cref{thm:one_step_KL_bound}, 
we get that for all positive integers $k$ 
\begin{equation}
	H_\nu( \rho_{k} ) 
	\leq e^{ -\alpha \eta } H_\nu( \rho_{k-1} ) 
		+ C \eta^2 \,, 
\end{equation}
where $C = 30 nd K_2^2$. 

We can then continue to expand the KL divergence term 
for iteration $k-1$ and smaller to get 
\begin{equation}
	H_\nu( \rho_k ) 
	\leq e^{-\alpha k \eta} H_\nu( \rho_0 ) 
		+ \sum_{\ell=0}^{k-1} e^{-\alpha \ell \eta} C \eta^2 \,. 
\end{equation}

We can upper bound the second term 
by an infinite geometric series, 
hence leading to 
\begin{equation}
	\sum_{\ell=0}^{k-1} e^{-\alpha \ell \eta} C \eta^2
	\leq C \eta^2 \frac{1}{ 1 - e^{-\alpha \eta} } 
	\leq C \eta^2 \frac{1}{\alpha\eta( 1 - \alpha \eta / 2 )} \,, 
\end{equation}
where we used the fact that $1 - e^{-x} \geq x-x^2/2$. 
Using the constraint that $\eta \leq \frac{2}{3\alpha}$, 
we further get that 
\begin{equation}
	\sum_{\ell=0}^{k-1} e^{-\alpha \ell \eta} C \eta^2
	\leq C \frac{3\eta}{2\alpha} \,. 
\end{equation}

Finally, putting everything together, 
we can get the desired bound for all $k > 0$ 
\begin{equation}
	H_\nu( \rho_k ) 
	\leq 
		e^{-\alpha k \eta} H_\nu( \rho_0 ) 
		+ 45 nd K_2^2 
		\frac{\eta}{\alpha} \,. 
\end{equation}

\end{proof}

\section{Suboptimality of the Gibbs Distribution - 
Proof of \cref{thm:gibbs_high_prob_bound}}
\label{sec:subopt_proof}

In this subsection, 
we attempt to prove results about optimization using 
the Gibbs distribution 
\begin{equation}
	\nu(x) = \frac{1}{Z} e^{ -\beta F(x) } \,, 
\end{equation}
where $Z$ is the normalizing constant. 

Without loss of generality, 
we shall assume $\min_x F(x) = 0$, 
since constant shifts do not affect 
the algorithm or the distribution. 
We start with the following approximation Lemma. 

\begin{lemma}
\label{lm:sub_opt_quad_approx}
Let $F$ be $K_2$-smooth on $(M,g)$, 
and let $x^*$ be any global minimum of $F$. 
Then for all $\epsilon > 0$, 
we have the following bound 
\begin{equation}
	\nu( F \geq \epsilon ) 
	\leq \frac{
		e^{-\beta \epsilon} \Vol(M)
		}{
		\int_{B_\radius(x^*)} 
			e^{ -\beta \frac{K_2}{2} d_g(x^*, x)^2 } dV_g(x)
		} \,, 
\end{equation}
where $\radius = \sqrt{ 2\epsilon / K_2 }$, 
$B_\radius(x^*)$ is the geodesic ball of radius $\radius$ 
centered at $x^*$, 
and $d_g$ is the geodesic distance.  
\end{lemma}

\begin{proof}

We start by defining the unnormalized measure 
\begin{equation}
	\widehat{\nu}_F( A ) 
	:= \int_A e^{-\beta F(x)} dV_g(x) \,. 
\end{equation}

Then letting $D_\epsilon(F) := \{ x : F(x) < \epsilon \}$, 
we can rewrite the desired probability as 
\begin{equation}
	\nu( F \geq \epsilon ) 
	= \frac{ \widehat{\nu}_F( D_\epsilon(F)^c ) 
		}{
			\widehat{\nu}_F( D_\epsilon(F) )
			+ \widehat{\nu}_F( D_\epsilon(F)^c )
		} \,. 
\end{equation}

Observe that right hand side is now a function of $F$, 
and to achieve an upper bound, 
it is sufficient to modify $F$ such that 
either increase $\widehat{\nu}_F( D_\epsilon(F)^c )$ 
and/or decrease $\widehat{\nu}_F( D_\epsilon(F) )$. 

More precisely, for any function $G$ such that 
\begin{equation}
	\widehat{\nu}_F( D_\epsilon(F) )
	\geq 
	\widehat{\nu}_G( D_\epsilon(G) ) \,, 
	\quad 
	\widehat{\nu}_F( D_\epsilon(F)^c ) 
	\leq 
	\widehat{\nu}_G( D_\epsilon(G)^c ) \,, 
\end{equation}
we then also have 
\begin{equation}
	\frac{ \widehat{\nu}_F( D_\epsilon(F)^c ) 
		}{
			\widehat{\nu}_F( D_\epsilon(F) )
			+ \widehat{\nu}_F( D_\epsilon(F)^c )
		}
	\leq 
	\frac{ \widehat{\nu}_G( D_\epsilon(G)^c ) 
		}{
			\widehat{\nu}_G( D_\epsilon(G) )
			+ \widehat{\nu}_G( D_\epsilon(G)^c )
		} \,. 
\end{equation}

To this goal, we will upper bound $F$ near $x^*$ 
using a ``quadratic'' function, i.e. 
\begin{equation}
	G(x) = \frac{K_2}{2} d_g(x^*, x) \,, 
	\quad x \in B_\radius(x^*)^2 \,. 
\end{equation}

To see that $G(x) \geq F(x)$ on $B_\radius(x^*)$, 
we will let $\gamma(t)$ be the unit speed geodesic 
such that $\gamma(0) = x^*$ and $\gamma(d_g(x^*,x)) = x$. 
Then using the fact that $F(x^*) = 0$ and $dF(x^*) = 0$, 
we can write 
\begin{equation}
	F(x) = \int_0^{d_g(x^*,x)} dF( \dot{\gamma}(s_1) ) \, ds_1
		= \int_0^{d_g(x^*,x)} \int_0^{s_1} 
			\nabla^2 F( \dot{\gamma}(s_2), \dot{\gamma}(s_2) )
			ds_2 \, ds_1 \,. 
\end{equation}

Using the fact that $F$ is $K_2$-smooth 
and $\gamma(t)$ is unit speed, 
we can upper bound the integrand by $K_2$, 
and therefore we can write 
\begin{equation}
	\int_0^{d_g(x^*,x)} \int_0^{s_1} 
			\nabla^2 F( \dot{\gamma}(s_2), \dot{\gamma}(s_2) )
			ds_2 \, ds_1
	\leq 
		\int_0^{d_g(x^*,x)} \int_0^{s_1} 
			K_2 
			ds_2 \, ds_1
	= \frac{K_2}{2} d_g(x^*, x)^2 \,, 
\end{equation}
which is the desired bound. 

For $x \notin B_\radius(x^*)$, 
we will simply define $G(x) := \epsilon$. 
Then we must have that 
\begin{equation}
	\widehat{\nu}_F( D_\epsilon(F) ) 
	= \int_{D_\epsilon(F)} e^{-\beta F(x)} dV_g(x)
	\geq \int_{D_\epsilon(G)} e^{-\beta G(x)} dV_g(x) 
	= \widehat{\nu}_G( D_\epsilon(G) ) 
	\,, 
\end{equation}
since $D_\epsilon(G) \subset D_\epsilon(F)$ 
and $e^{-\beta F} \geq e^{-\beta G}$. 

The other direction is obtained similarly 
\begin{equation}
	\widehat{\nu}_F( D_\epsilon(F)^c ) 
	= \int_{ D_\epsilon(F)^c } e^{ -\beta F(x) } dV_g(x) 
	\leq \int_{ D_\epsilon(F)^c } e^{ -\beta \epsilon } dV_g(x) 
	\leq \int_{ D_\epsilon(G)^c } e^{ -\beta \epsilon } dV_g(x) 
	= \widehat{\nu}_G( D_\epsilon(G)^c ) 
	\,. 
\end{equation}

Putting everything together, 
we get 
\begin{equation}
\begin{aligned}
	\frac{ \widehat{\nu}_F( D_\epsilon(F)^c ) 
		}{
			\widehat{\nu}_F( D_\epsilon(F) )
			+ \widehat{\nu}_F( D_\epsilon(F)^c )
		}
	&\leq 
	\frac{ \widehat{\nu}_G( D_\epsilon(G)^c ) 
		}{
			\widehat{\nu}_G( D_\epsilon(G) )
			+ \widehat{\nu}_G( D_\epsilon(G)^c )
		} \\ 
	&= \frac{
		e^{-\beta \epsilon} \Vol( M \setminus B_\radius(x^*) ) 
		}{
		e^{-\beta \epsilon} \Vol( M \setminus B_\radius(x^*) ) 
		+ \int_{B_\radius(x^*)} 
			e^{ -\beta \frac{K_2}{2} d_g(x^*, x)^2 } dV_g(x)
		} \,. 
\end{aligned}
\end{equation}

We get the desired result by further upper bounding 
this term by 
\begin{equation}
	\frac{
		e^{-\beta \epsilon} \Vol( M \setminus B_\radius(x^*) ) 
		}{
		e^{-\beta \epsilon} \Vol( M \setminus B_\radius(x^*) ) 
		+ \int_{B_\radius(x^*)} 
			e^{ -\beta \frac{K_2}{2} d_g(x^*, x)^2 } dV_g(x)
		}
	\leq 
	\frac{
		e^{-\beta \epsilon} \Vol(M) 
		}{
		\int_{B_\radius(x^*)} 
			e^{ -\beta \frac{K_2}{2} d_g(x^*, x)^2 } dV_g(x)
		} \,. 
\end{equation}

\end{proof}

\begin{lemma}
\label{lm:sub_opt_integral_approx}
For $\radius \leq \pi / 2$ and 
$\beta \geq \frac{nd}{\radius^2 K_2}$, 
we have the following lower bound 
\begin{equation}
	\int_{B_\radius(x^*)} 
			e^{ -\beta \frac{K_2}{2} d_g(x^*, x)^2 } dV_g(x)
	\geq
		\left( \frac{2}{\pi} \right)^{n(d-1)} 
		\left( \frac{2\pi}{\beta K_2} \right)^{nd/2} 
		\left[ 1 - \exp\left( 
			-\frac{1}{2}( \radius^2 \beta K_2 / 2 - nd )
			\right)
		\right] \,. 
\end{equation}
\end{lemma}

\begin{proof}

We begin by rewriting the integral in normal coordinates, 
and observe that a geodesic ball $B_\radius(x^*) \subset M $ 
is then $B_\radius(0) \subset \mathbb{R}^{nd}$. 
Then we get 
\begin{equation}
	\int_{B_\radius(x^*)} 
			e^{ -\beta \frac{K_2}{2} d_g(x^*, x)^2 } dV_g(x)
	= 
	\int_{B_\radius(0)} 
			e^{ -\beta \frac{K_2}{2} |x|^2 } \sqrt{\det g} \, dx \,, 
\end{equation}
where $g$ is the Riemannian metric on 
$M = S^d \times \cdots S^d$ ($n$-times). 

At this point, we observe 
the block diagonal structure of $g$ 
gives us 
\begin{equation}
	\det g = ( \det g_{S^d} )^n \,, 
\end{equation}
where $g_{S^d}$ is the Riemannian metric of $S^d$ 
in normal coordinates. 
Finally, we can also have the lower bound 
from \cref{lm:normal_coord_vol_est} 
whenever $|y| \leq \pi / 2$ 
\begin{equation}
	\det g_{S^d} \geq \left( \frac{2}{\pi} \right)^{2(d-1)} \,. 
\end{equation}

This then implies the lower bound on the volume form 
\begin{equation}
	\sqrt{ \det g } \geq \left( \frac{2}{\pi} \right)^{n(d-1)} \,, 
\end{equation}
leading to the following integral lower bound 
\begin{equation}
	\int_{B_\radius(x^*)} 
			e^{ -\beta \frac{K_2}{2} d_g(x^*, x)^2 } dV_g(x)
	\geq 
		\left( \frac{2}{\pi} \right)^{n(d-1)}
		\int_{B_\radius(0)} 
				e^{ -\beta \frac{K_2}{2} |x|^2 } dx \,. 
\end{equation}

We next observe the above integral is an unnormalized 
Gaussian integral in $\mathbb{R}^{nd}$ 
with variance $(\beta K_2)^{-1}$. 
So we can rewrite the integral as 
\begin{equation}
\label{eq:lb_by_gaussian}
	\int_{B_\radius(0)} 
				e^{ -\beta \frac{K_2}{2} |x|^2 } dx
	= \left( \frac{2\pi}{\beta K_2} \right)^{nd/2} 
		\mathbb{P}[ |Z| \leq \radius \sqrt{\beta K_2} ] \,, 
\end{equation}
where $Z \sim N(0, I_{nd})$ is a standard Gaussian. 

At this point, it is sufficient to provide 
a Gaussian tail bound, since 
\begin{equation}
	\mathbb{P}[ |Z| \leq \radius \sqrt{\beta K_2} ] 
	= 1 - \mathbb{P}[ |Z| \geq \radius \sqrt{\beta K_2} ]
	= 1 - \mathbb{P}[ |Z| - \mathbb{E} |Z| 
		\geq \radius \sqrt{\beta K_2} - \mathbb{E} |Z| ] \,. 
\end{equation}

Using the Cauchy-Schwarz inequality, we get that 
\begin{equation}
	\mathbb{E} |Z| \leq \sqrt{ \mathbb{E} |Z|^2 } = \sqrt{nd} \,, 
\end{equation}
and therefore we can bound 
\begin{equation}
	\mathbb{P}[ |Z| - \mathbb{E} |Z| 
		\geq \radius \sqrt{\beta K_2} - \mathbb{E} |Z| ]
	\leq 
		\mathbb{P}[ |Z| - \mathbb{E} |Z| 
		\geq \radius \sqrt{\beta K_2} - \sqrt{nd} ] \,. 
\end{equation}

Using the assumption $\beta \geq \frac{nd}{\radius^2 K_2}$, 
we ensure that $\radius \sqrt{\beta K_2} - \sqrt{nd} \geq 0$, 
hence we can use the $1$-Lipschitz concentration bound 
to get 
\begin{equation}
	\mathbb{P}[ |Z| - \mathbb{E} |Z| 
		\geq \radius \sqrt{\beta K_2} - \sqrt{nd} ]
	\leq 
		\exp\left( -\frac{1}{2} 
			\left( \radius \sqrt{\beta K_2} - \sqrt{nd}
			\right)^2
			\right) \,. 
\end{equation}

Here we will use Young's inequality to write 
$ab \leq \frac{a^2}{4} + b^2$, 
which leads to 
\begin{equation}
	(a - b)^2 = a^2 + b^2 - 2ab \geq a^2 + b^2 - \frac{1}{2} a^2 - 2b^2
	= \frac{1}{2} a^2 - b^2 \,, 
\end{equation}
and applying to $a = \radius \sqrt{\beta K_2}, b = \sqrt{nd}$, 
we get that 
\begin{equation}
	\exp\left( -\frac{1}{2} 
			\left( \radius \sqrt{\beta K_2} - \sqrt{nd}
			\right)^2
			\right)
	\leq 
	\exp\left( -\frac{1}{2} 
			\left( \radius^2 \beta K_2 /2 - nd
			\right)
			\right) \,. 
\end{equation}

Putting everything together, 
we get the desired lower bound 
\begin{equation}
	\int_{B_\radius(x^*)} 
			e^{ -\beta \frac{K_2}{2} d_g(x^*, x)^2 } dV_g(x)
	\geq
		\left( \frac{2}{\pi} \right)^{n(d-1)} 
		\left( \frac{2\pi}{\beta K_2} \right)^{nd/2} 
		\left[ 1 - \exp\left( 
			-\frac{1}{2}( \radius^2 \beta K_2 / 2 - nd )
			\right)
		\right] \,. 
\end{equation}

\end{proof}

We will now restate our main result on suboptimality, 
originally \cref{thm:gibbs_high_prob_bound}. 

\begin{theorem}
[Gibbs Suboptimality Bound]
Let $F$ satisfy \cref{asm:c2} 
and suppose $d \geq 3$. 
Then for all 
$\epsilon \in (0,1]$ and $\delta \in (0,1)$, 
when we choose 
\begin{equation}
	\beta 
	\geq 
		\frac{3 nd }{\epsilon} 
		\log \frac{ n K_2 }{ \epsilon \, \delta }
		\,, 
\end{equation}
then we have that the Gibbs distribution 
$\nu(x) := \frac{1}{Z} e^{ -\beta F(x) }$ 
satisfies the following bound 
\begin{equation}
	\nu\left( F - \min_{y \in M} F(y) \geq \epsilon \right) 
	\leq \delta \,. 
\end{equation}
In other words, 
$\nu$ finds an $\epsilon$-approximate 
global minimum of $F$ 
with probability $1-\delta$. 

\end{theorem}

\begin{proof}

Given the results of \cref{lm:sub_opt_quad_approx} 
and \cref{lm:sub_opt_integral_approx}, 
we have the following upper bound 
by choosing $R = \sqrt{ 2\epsilon / K_2 }$ 
\begin{equation}
	\nu( F \geq \epsilon ) 
	\leq 
	\frac{ e^{-\beta \epsilon} \Vol(M) }{
		\left( \frac{2}{\pi} \right)^{n(d-1)} 
		\left( \frac{2\pi}{\beta K_2} \right)^{nd/2} 
		\left[ 1 - \exp\left( 
			-\frac{1}{2}( \epsilon \beta - nd )
			\right)
		\right] 
	} \,, 
\end{equation}
when $0 < \epsilon \leq 1 \leq \frac{\pi^2 K_2}{8} $, 
and $\beta \geq \frac{nd}{2 \epsilon}$, 
which is satisfied when 
$\beta \geq \frac{1}{\epsilon} ( nd + 2 \log 2 )$. 

Therefore it is sufficient to compute 
a lower bound condition on $\beta$ such that 
the right hand side term is less than $\delta$. 
To this end, we start by 
computing the requirement for $\beta$ such that 
\begin{equation}
	1 - \exp\left( 
			-\frac{1}{2}( \epsilon \beta - nd )
			\right)
	\geq \frac{1}{2} \,. 
\end{equation}

Rearranging and taking $log$ of both sides, 
it is equivalent to satisfy the condition 
\begin{equation}
	\frac{nd}{2} - \frac{\beta \epsilon}{2}
	\leq - \log 2 \,, 
\end{equation}
hence it is sufficient to have 
$\beta \geq \frac{1}{\epsilon} ( nd + 2 \log 2 )$. 

Therefore we have the upper bound 
\begin{equation}
	\nu( F \geq \epsilon ) 
	\leq 
		e^{-\beta \epsilon} \beta^{nd / 2} 
		\frac{
			\Vol(M)
		}{ \left( \frac{2}{\pi} \right)^{n(d-1)} 
			\left( \frac{2\pi}{K_2} \right)^{nd/2} \frac{1}{2} 
		} \,. 
\end{equation}

For now, let us define 
\begin{equation}
	C := \frac{
			\Vol(M)
		}{ \left( \frac{2}{\pi} \right)^{n(d-1)} 
			\left( \frac{2\pi}{K_2} \right)^{nd/2} \frac{1}{2} 
		} \,, 
\end{equation}
then to get the desired result, 
it is sufficient to show 
\begin{equation}
	e^{-\beta \epsilon} \beta^{nd / 2} C \leq \delta \,.  
\end{equation}

Taking $log$ and rearranging, we have the equivalent condition of 
\begin{equation}
	\beta \epsilon - \frac{nd}{2} \log \beta 
	\geq 
	\log C + \log \frac{1}{\delta} \,. 
\end{equation}

Here we will first observe that 
$\frac{\beta \epsilon}{2} - \frac{nd}{2} \log \beta$ 
will be an increasing function in terms of $\beta$
after a local minimum, 
therefore we compute this local minimum by differentiating 
with respect to $\beta$ to get 
\begin{equation}
	0 = \frac{\epsilon}{2} - \frac{nd}{2 \beta} \,, 
\end{equation}
therefore we have the local minimum is at 
$\beta = \frac{nd}{\epsilon}$, 
and therefore for all $\beta \geq \frac{nd}{\epsilon}$, 
we have that 
\begin{equation}
	\frac{\beta \epsilon}{2} - \frac{nd}{2} \log \beta 
	\geq 
		\frac{nd}{2} \left( 
			1 - \log \frac{nd}{\epsilon} 
		\right) \,. 
\end{equation}

Since we already require 
$\beta \geq \frac{1}{\epsilon} ( nd + 2 \log 2 )$, 
we have it is sufficient to show 
\begin{equation}
	\frac{ \beta \epsilon }{2} 
		+ \frac{nd}{2} \left( 
			1 - \log \frac{nd}{\epsilon} 
			\right)
	\geq 
	\log C + \log \frac{1}{\delta} \,, 
\end{equation}
and we rearrange to get 
\begin{equation}
	\beta 
	\geq 
		\frac{1}{\epsilon} 
		\left[ 2 \log C + 2 \log \frac{1}{\delta}
			+ nd \left(
				\log \frac{nd}{\epsilon} - 1
				\right)
		\right] \,. 
\end{equation}

Therefore it simply remains to compute 
an upper bound for $\log C$. 
We start by writing out the exact term 
\begin{equation}
	\log C
	= \log \Vol(M) 
		+ n(d-1) \log \frac{\pi}{2} 
		+ \frac{nd}{2} \log \frac{ K_2 }{ 2 \pi } 
		+ \log 2 \,. 
\end{equation}

We will group the $nd$ terms by upper bounding 
\begin{equation}
	n(d-1) \log \frac{\pi}{2} 
		+ \frac{nd}{2} \log \frac{ K_2 }{ 2 \pi } 
	\leq nd \log \frac{\pi}{2} \sqrt{ \frac{K_2}{2\pi} } 
	= nd \log \sqrt{ \frac{\pi K_2}{ 8 } } 
	\leq \frac{nd}{2} \log \frac{K_2}{2} \,, 
\end{equation}
where we used the fact that $\frac{\pi}{8} < \frac{1}{2}$. 

Then we observe that 
\begin{equation}
	\log \Vol(M) 
	= n \log \Vol(S^d) 
	= n \log \frac{ (d+1) \pi^{ (d+1)/2 } }{ \Gamma( \frac{d+1}{2} + 1 ) } \,. 
\end{equation}

Now we will use a Stirling's approximation bound 
from \cref{lm:gamma_bounds} to write 
\begin{equation}
	\Gamma\left( \frac{d+1}{2} + 1 \right)
	\geq 
		\sqrt{2\pi} \left( \frac{d+1}{2} \right)^{ d/2 + 1 }
			e^{ -(d+1)/2 } \,, 
\end{equation}
which gives us the bound on volume 
\begin{equation}
	\log \Vol(M)
	\leq 
		n \log \left[ (e\pi)^{ (d+1)/2 } 
			\left( \frac{d+1}{2} \right)^{-d/2} 
			\sqrt{ \frac{2}{\pi} }
			\right] 
	= n \left[ \frac{-d}{2}
			\left( \log \frac{d+1}{2} 
				- 1 - \log \pi 
			\right)
				+ \frac{1}{2} (1 + \log 2)
		\right] \,. 
\end{equation}

Since we have $d \geq 3$, 
we group the constant in for simplicity by bounding 
\begin{equation}
	\frac{1}{2} (1 + \log 2) \leq \frac{d}{6} (1 + \log 2) \,, 
\end{equation}
which gives us a cleaner form 
\begin{equation}
	\log \Vol(M)
	\leq 
		n d \left[ \frac{-1}{2}
			\log \frac{d+1}{2} 
				+ \frac{1 + \log \pi }{2} 
				+ \frac{1}{6} (1 + \log 2)
		\right]
	\leq 
		n d \left[ \frac{-1}{2}
			\log \frac{d+1}{2} 
				+ 2 
		\right] \,, 
\end{equation}
where we used the fact that 
$\frac{1 + \log \pi }{2} + \frac{1}{6} (1 + \log 2) \approx 1.35 < 2$. 

Putting the $\log C$ terms together, 
we get the bound 
\begin{equation}
	\log C 
	\leq 
		nd \left( 
			\frac{1}{2} \log \frac{K_2}{2}
			- \frac{1}{2}
			\log \frac{d+1}{2} 
			+ 2 
			\right) 
		+ \log 2 
		\,, 
\end{equation}
and putting the above expression back to 
the lower bound on $\beta$, we get 
\begin{equation}
\begin{aligned}
	& \frac{1}{\epsilon} 
		\left[ 
			2 nd \left( 
			\frac{1}{2} \log \frac{K_2}{2}
			- \frac{1}{2}
			\log \frac{d+1}{2} 
			+ 2 
			\right) 
			+ 2 \log 2
			+ 2 \log \frac{1}{\delta}
			+ nd \left(
				\log \frac{nd}{\epsilon} - 1
				\right)
		\right] \\
	&= \frac{1}{\epsilon} 
		\left[ 
			nd \left( 
			\log \frac{2nd}{\epsilon(d+1)}
			+ \log \frac{K_2}{2}
			+ 3 
			\right) 
		+ 2 \log \frac{1}{\delta} 
		+ 2 \log 2 
		\right] \\
	&\leq \frac{1}{\epsilon} 
		\left[ 
			nd \left( 
			\log \frac{2n}{\epsilon}
			+ \log \frac{K_2}{2}
			+ 3 
			\right) 
		+ 2 \log \frac{1}{\delta} 
		+ 2 \log 2 
		\right] \,. 
\end{aligned}
\end{equation}

To get the desired sufficient lower bound for $\beta$, 
we clean up the constants using the fact that $K_2 \geq 1$, 
$\log 2 < 0$ to get 
\begin{equation}
	\beta 
	\geq 
		\frac{3 nd }{\epsilon} 
		\log \frac{ n K_2 }{ \epsilon \, \delta }
		\,. 
\end{equation}

\end{proof}

\section{Escape Time Approach to Lyapunov - Proof of \cref{thm:lyapunov_log_sobolev}}
\label{sec:escape_time_proof}

\subsection{Equivalent Escape Time Formulation}

In this section, we will study a construction of 
Lyapunov function using escape time. 
In particular, 
we consider a partition of $M = B \sqcup B^c$, 
such that $B$ is open, 
$\de B \in C^1(M)$, 
and $\nu$ satisfies a Poincar\'{e} inequality 
when restricted to $B^c$. 
With this in mind, we will introduce the following definition. 

\begin{definition}
	$W \in C^2(M)$ is said to be a \textbf{Lyapunov function} 
	on $B \subset M$ with parameters $\theta > 0, b \geq 0$
	if $W \geq 1$ and 
		\begin{equation}
			LW \leq - \theta W 
			\,, \quad x \in B \,. 
		\end{equation}
\end{definition}

Recall \cref{thm:lyapunov_poincare_bakry} 
adapted from \cite{bakry2008simple}, 
if $W$ is a Lyapunov function under the above conditions, 
then $\nu$ satisfies a Poincar\'{e} inequality on $M$. 
For the rest of this section, 
we will construct a Lyapunov function $W$ 
in terms of an escape time of $Z_t$. 

To start, we recall the Langevin diffusion on $M$ 
\begin{equation}
	dZ_t = - \grad F(Z_t) \, dt + \sqrt{\frac{2}{\beta}} \, dW_t \,, 
\end{equation}
and it has the generator 
\begin{equation}
	L \phi = \langle - \grad \phi, \grad F \rangle_g 
		+ \frac{1}{\beta} \Delta \phi \,, 
	\quad \forall \phi \in C^2(M) \,. 
\end{equation}

With this in mind, we define the first escape time outside of $B$ 
\begin{equation}
	\tau_{B^c} := \inf \{ t \geq 0 | Z_t \notin B \} \,, 
\end{equation}
and we also define 
$\mathbb{E}_x [\, \cdots \,] 
 := \mathbb{E} [ \, \cdots \, | Z_0 = x ]$. 

Then we will recall the classic Feynman--Kac theorem 
for the bounded domain. 

\begin{theorem}
\citep[Theorem 7.15]{bovier2016metastability}
\label{thm:feynman_kac_bounded_domain}
Suppose $u$ is a unique solution to 
the following Dirichlet problem 
\begin{equation}
	\begin{cases}
		- Lu + k(x) u = h(x) \,, & x \in B \,, \\ 
		u(x) = \overline{h}(x) \,, & x \in \de B \,, 
	\end{cases}
\end{equation}
and further assume 
\begin{equation}
	\mathbb{E}_x \left[ \tau_{B^c} 
			\exp \left( 
				\inf_{x\in B} k(x) \tau_{B^c} 
			\right) 
		\right]
	< \infty \,, 
	\quad x \in B \,. 
\end{equation}
Then we have the following stochastic representation 
\begin{equation}
\begin{aligned}
	u(x) &= \mathbb{E}_x \bigg[
		\overline{h}( Z_{\tau_{B^c}} )
		\exp \left( 
			- \int_0^{\tau_{B^c}} k(Z_s) \, ds 
			\right)
		+ \int_0^{\tau_{B^c}} h( Z_{\tau_{B^c}} )
		\exp \left( 
			- \int_0^t k(Z_s) \, ds 
			\right)
		\, dt 
		\bigg] \,, 
	\quad \forall x \in B \,. 
\end{aligned}
\end{equation}
\end{theorem}

As a consequence, 
if $\tau_{B^c}$ is exponentially integrable, 
then we have an explicit formula for a Lyapunov function. 

\begin{corollary}
\label{cor:escape_time_lyapunov}
Suppose there exists $\theta > 0$ such that 
\begin{equation}
	\mathbb{E}[ \exp( \theta \tau_{B^c} ) | Z_0 = x ] < \infty \,, 
	\quad \forall x \in B \,. 
\end{equation}
Then for $W(x) := \mathbb{E} [ \exp( \theta \tau_{B^c}) | Z_0 = x ]$, 
we have that 
\begin{equation}
	L W = -\theta W \,, \quad \forall x \in B \,, 
\end{equation}
and hence $W$ is a Lyapunov function. 
\end{corollary}

\begin{proof}

We will first state the desired form of 
the Dirichlet problem 
\begin{equation}
\begin{cases}
	LW = - \theta W \,, & x \in B \,, \\ 
	W = 1 \,, & x \in \de B \,. 
\end{cases}
\end{equation}

Using a standard existence and uniqueness result 
for elliptic equations \citep[Section 6.2, Theorem 5]{evans2010partial}, 
we can now invoke the Feynman-Kac representation 
\cref{thm:feynman_kac_bounded_domain} 
to get that 
\begin{equation}
	W(x) = \mathbb{E} [ \exp( \theta \tau_{B^c}) | Z_0 = x ] \,, 
\end{equation}
where we weakened the boundedness condition 
since we removed the integral term. 

\end{proof}

Next we will also observe that it is equivalent to 
show that $\tau_{B^c}$ is a sub-exponential 
random variable using the following equivalent 
characterization result. 

\begin{theorem}
\citep[Theorem 2.13, Equivalent Characterization of 
Sub-Exponential Variables]{wainwright2019high}
\label{thm:subexp_equiv_wainwright}
For a zero mean random variable $X$, 
the following statements are equivalent:
\begin{enumerate}
	\item there exist positive numbers $(\nu,\alpha)$ such that 
	\begin{equation}
		\mathbb{E} e^{\lambda X} 
		\leq e^{\nu^2 \lambda^2 / 2} \,, 
		\quad \forall |\lambda| < \frac{1}{\alpha} \,. 
	\end{equation}
	\item there exist a positive number $c_0 > 0$ 
	such that 
	\begin{equation}
		\mathbb{E} e^{c X} < \infty \,, 
		\quad \forall |c| \leq c_0 \,. 
	\end{equation}
	\item there are constants $c_1, c_2 > 0$ such that 
	\begin{equation}
		\mathbb{P}[ \, |X| \geq t ]
		\leq c_1 e^{-c_2 t} \,, 
		\quad \forall t \geq 0 \,. 
	\end{equation}
	\item we have 
	\begin{equation}
		\sup_{k \geq 2} \left[
			\frac{ \mathbb{E} [X^k] }{ k! } 
			\right]^{1/k} < \infty \,. 
	\end{equation}
\end{enumerate}
\end{theorem}

In particular, we have from the proof in 
\cite[Section 2.5]{wainwright2019high} that, 
if we have $c_1, c_2 > 0$ such that 
\begin{equation}
		\mathbb{P}[ \, |X| \geq t ]
		\leq c_1 e^{-c_2 t} \,, 
		\quad \forall t \geq 0 \,, 
\end{equation}
then we have that for $c_0 = c_2 / 2$ 
\begin{equation}
	\mathbb{E} e^{c X} < \infty \,, 
	\quad \forall |c| \leq c_0 \,. 
\end{equation}

So it is sufficient to compute the constant $c_2$ 
in the above theorem, 
which is the decay rate in the tail bound. 
We summarize this discussion in the next result. 

\begin{proposition}
[Escape Time to Lyapunov Condition]
\label{prop:escape_lyapunov}
Suppose there exists $c_1,c_2 > 0$ such that 
\begin{equation}
	\mathbb{P}[ \, \tau_{B^c} \geq t | Z_0 = x ]
		\leq c_1 e^{-c_2 t} \,, 
		\quad \forall t \geq 0 \,, 
		\forall x \in B \,. 
\end{equation}
Then for $W(x) := \mathbb{E} [ \exp( \theta \tau_{B^c}) | Z_0 = x ]$, 
we have that 
\begin{equation}
	L W \leq - \frac{c_2}{2} W \,, \quad x \in B \,, 
\end{equation}
and hence $W$ is a Lyapunov function. 

\end{proposition}

\subsection{Local Escape Time Bounds}

In this subsection, we consider escaping from 
a neighbourhood near a saddle point $\mathcal{S}$. 
Recall $B := \left\{ x \in M | 
d_g(x,\mathcal{S})^2 \leq \frac{a^2}{\beta} \right\}$, 
and from \cref{prop:escape_lyapunov} 
we know it is sufficient to show that 
\begin{equation}
	\mathbb{P}[ \tau_{B^c} \geq t | Z_0 = x ] 
	\leq c_1 e^{ -c_2 t } 
	\,, \quad \forall \, t \geq 0, x \in B \,, 
\end{equation}
for some constants $c_1, c_2 > 0$, 
which implies that $W(x) := \mathbb{E}[ e^{-\theta \tau_{B^c}} | Z_0 = x ]$ 
is a Lyapunov function with $\theta = \frac{c_2}{2}$. 
We will next establish the following escape time bound.

\begin{proposition}
[Local Escape Time Bound]
\label{prop:local_escape_time_bound}
Let \cref{asm:app_c3,asm:app_morse} hold, 
and let $\{Z_t\}_{t\geq 0}$ be the Langevin diffusion 
defined in \cref{eq:langevin_diffusion}. 
If we choose  
\begin{equation}
	\beta 
	\geq 
		\max\left( a^2 \,, 
			K_3^2 a^6 
		\right) \,, 
\end{equation}
then we have that 
\begin{equation}
	\mathbb{P}[ \tau_{B^c} \geq t | Z_0 = x ] 
	\leq 
		C e^{ - \lambda_* t / 2 } 
		\,, \quad \forall \, t \geq 0, \, x \in B \,, 
\end{equation}
where $C := C(a,\beta,\lambda_*) > 0$ 
is a constant independent of $t,x$. 
Hence $W(x) = \mathbb{E} [ \exp( \theta \tau_{B^c}) | Z_0 = x ]$ 
is a Lyapunov function on $B$ with parameters 
$\theta = \frac{\lambda_*}{4}$. 

\end{proposition}

\begin{proof}

For each saddle point $y \in \mathcal{S}$, 
let $v_y \in T_y M$ be the eigenvector of $\nabla^2 F(y)$ 
that corresponds to the minimum eigenvalue of $\nabla^2 F(y)$. 
And by \cref{asm:app_morse}, we have that 
$\nabla^2 F(y)[v_y, v_y] \leq - \lambda_*$ 
for all $y \in \mathcal{S}$. 
Let us fix a $y \in \mathcal{S}$ such that 
$d_g(y,x)^2 \leq \frac{a^2}{\beta}$.

Let $r(Z_t) := d_g(y, Z_t)$ be the radial process, 
and we can compute the gradient in normal coordinates 
centered at $y$ as 
\begin{equation}
	\grad r(x) 
	= \frac{ - \log_x y }{ |\log_x y| } 
	= \frac{ x }{ |x| } \,, 
\end{equation}
where $\log_x : M \to T_x M$ 
as the inverse of the standard exponential map 
$\exp_x : T_x M \to M$ 
defined outside of the cut-locus of $x$, 
and we treat $x \in \mathbb{R}^{nd}$ in coordinate. 

We will also define 
\begin{equation}
	\wt{r}(x) 
	:= | \langle v_y, x \rangle | \,, 
\end{equation}
where the calculation is done in normal coordinates, 
Observe that $\wt{r}(x)$ is the radial process 
after ``projection'' onto the submanifold (curve) 
$\wt{B} := \{ \exp(y, tv_y) \}_{|t| \leq \frac{a^2}{\beta}}$, 
therefore we have $r(x) \geq \wt{r}(x)$. 

For convenience, we will define $P:B \to \wt{B}$ 
as the ``projection'' in normal coordinates 
\begin{equation}
	P x := \langle v_y, x \rangle v_y \,, 
\end{equation}
and therefore we also have that 
\begin{equation}
	\grad \wt{r}(x) 
	= \frac{ - P \log_x y }{ | P \log_x y | } 
	= \frac{ P x }{ | P x | } \,, 
\end{equation}
where we treat $P$ as a projection in the tangent space as well. 

Then we can compute using It\^{o}'s formula to get 
\begin{equation}
\begin{aligned}
	d \left[ \frac{1}{2} \wt{r}(Z_t)^2 \right]
	=& 
		\left[ 
			\left\langle - \grad F(Z_t) \,, 
				\wt{r}(Z_t) \grad \wt{r}(Z_t) 
			\right\rangle_g 
			+ 
			\frac{1}{\beta} ( |\grad \wt{r}(Z_t) |_g^2 
				+ \wt{r}(Z_t) \Delta \wt{r}(Z_t) )
		\right] \, dt 
		\\
	&+ 
		\sqrt{ \frac{2}{\beta} } 
		\left\langle \wt{r}(Z_t) \grad \wt{r}(Z_t)  \,, 
			dW_t 
		\right\rangle_g \,, 
\end{aligned}
\end{equation}
where we observe that since $\grad \wt{r}(x)$ is a unit vector, 
then $\langle \grad \wt{r}(Z_t) \,, dW_t \rangle_g$ 
is a standard one-dimensional Brownian motion
independent of $Z_t$, 
which we can denote $dB_t$. 

Next we will attempt to approximate $F$ 
by a ``quadratic potential'', 
more precisely we will define the vector field in 
normal coordinates centered at $y$ 
\begin{equation}
	H(x) := g^{ij}(x) \de_{jk} F(0) x^k \de_i \,, 
\end{equation}
and since the Hessian is $K_3$-Lipschitz, 
we can use a standard quadratic approximation result 
from \cite[Proposition 10.55]{boumal2020introduction} to get 
\begin{equation}
	| \grad F(x) - H(x) |_g 
	\leq 
		\frac{K_3}{2}
		|x|^2 \,. 
\end{equation}

This allows us to write 
\begin{equation}
\begin{aligned}
	\left\langle - \grad F(Z_t) \,, 
 		\grad \wt{r}(Z_t) 
 	\right\rangle_g 
 	&\geq 
 		\left\langle - H(Z_t) \,, 
		 		\grad \wt{r}(Z_t) 
		 	\right\rangle_g 
		 - 
		 	| \grad F(Z_t) - H(Z_t) |_g \\
	&\geq 
		\left\langle - H(Z_t) \,, 
		 		\grad \wt{r}(Z_t) 
		 	\right\rangle_g 
		 - 
		 	\frac{K_3}{2} 
			|Z_t|^2 \,. 
\end{aligned}
\end{equation}

Using the definition of $H(x)$ in normal coordinates, 
we can further write 
\begin{equation}
	\left\langle - H(x) \,, 
 		\grad \wt{r}(x) 
 	\right\rangle_g 
 	= 
 		- \de_i \de_j F(0) x_i (P x_j) \frac{1}{| Px |} 
 	\geq 
 		\lambda_* \, \wt{r}(x) \,, 
\end{equation}
where we used the fact that $Px$ is in the same direction as $v_y$, 
and therefore an eigenvector of $\nabla^2 F(y)$. 

Using the Laplacian comparison theorem 
\citep[Corollary 3.4.4]{hsu2002stochastic}, 
we also have the lower bound 
\begin{equation}
	\Delta \wt{r}(x) \geq (\dim \wt{B} - 1) K \cot (K \wt{r}(x)) = 0 \,, 
\end{equation}
since $\dim \wt{B} = 1$, 
and $K$ is the maximum scalar curvature. 
This further implies that 
\begin{equation}
	\Delta \left[ \frac{1}{2} \wt{r}(x)^2 \right] 
	= 
		| \grad \wt{r}(x) |_g^2 
		+ \wt{r}(x) \Delta \wt{r}(x) 
	\geq 
		1 \,. 
\end{equation}

Using the one-dimensional SDE comparison 
from \cref{thm:SDE_weak_comparison}, 
we have that $\frac{1}{2} \wt{r}(Z_t)^2$ is lower bounded 
by the process $\frac{1}{2} (r^{(1)}_t)^2$, 
which is defined 
\begin{equation}
	d \left[ \frac{1}{2} (r^{(1)}_t)^2 \right]
	= \left[ 
		\lambda_* ( r^{(1)}_t )^2
		- \frac{K_3}{2} 
			|Z_t|^2 \, r^{(1)}_t
		+ \frac{1}{\beta} 
		\right] \, dt 
		+ \sqrt{ \frac{2}{\beta} } 
			r^{(1)}_t 
			\, dB_t \,. 
\end{equation}

At this point, we will further bound 
$ |Z_t|^2 \leq \frac{a^2}{\beta} $ 
since we are only concerned with $Z_t \in B$. 
Furthermore, we can use the fact that 
$r^{(1)}_t \leq \wt{r}(Z_t) \leq \frac{a^2}{\beta}$, 
and using these bounds and \cref{thm:SDE_weak_comparison} again, 
we have yet another lower bounded process
\begin{equation}
\begin{aligned}
	d \left[ \frac{1}{2} (r^{(2)}_t)^2 \right]
	&= 
		\left[ 
		\lambda_* ( r^{(2)}_t )^2 
		- \frac{K_3}{2} \frac{a^3}{\beta^{3/2}} 
		+ \frac{1}{\beta} 
		\right] \, dt 
		+ \sqrt{ \frac{2}{\beta} } 
			r^{(2)}_t 
			\, dB_t 
		\\ 
	&= \left[ 
		\lambda_* ( r^{(2)}_t )^2 
		+ \frac{1}{\beta} \left( 
			1
			- 
			\frac{K_3}{2} \frac{a^3}{\beta^{1/2}} 
		\right) 
		\right] \, dt 
		+ \sqrt{ \frac{2}{\beta} } 
			r^{(2)}_t \, dB_t 
		\,. 
\end{aligned}
\end{equation}

Using the fact that 
$\beta \geq K_3^2 a^6 $, 
we also get the bound 
\begin{equation}
	1 - \frac{K_3}{2} \frac{a^3}{\beta^{1/2}} 
	\geq 
		\frac{1}{2} 
		\,, 
\end{equation}
which gives us a further lower bounded process 
\begin{equation}
	d \left[ \frac{1}{2} (r^{(3)}_t)^2 \right] 
	= 
	\left[ 
		2 |\lambda_\ell| \frac{1}{2} ( r^{(3)}_t )^2 
		+ \frac{1}{ 2 \beta } 
		\right] \, dt 
		+ \sqrt{ \frac{2}{\beta} } 
			r^{(3)}_t \, dB_t 
		\,. 
\end{equation}

Now choosing $Y_t = \frac{1}{2} (r^{(3)}_t)^2$, 
we can use \cref{cor:cir_density} to get 
the density of this Cox-Ingersoll-Ross (CIR) process to be 
\begin{equation}
	f(x;t) = 
		2^{ (1/4 - 3) / 2 } 
		x^{ (1/4 - 1) / 2 } 
		\frac{ \lambda_* \beta }{ 
			e^{ \lambda_* t /4 }
			\sinh ( \lambda_*t ) 
			} 
		\exp\left[ 
			\frac{ \lambda_* \beta \left( 
				x e^{ -2 \lambda_* t } - \frac{1}{2}
				\right) }{
				1 - e^{ -2 \lambda_* t }
				} 
		\right] 
		I_{(1/4 - 1)}\left( 
			\frac{ \lambda_* \beta }{ \sinh( \lambda_* t ) } 
			\sqrt{ \frac{x}{2} }
		\right) \,. 
\end{equation}

And when $x \leq \frac{a^2}{2 \beta}$ we have the following bound 
via \cref{lm:cir_density_est} 
\begin{equation}
	f(x;t) \leq C e^{ -  \lambda_* t / 2 } \,, 
\end{equation}
for all $t \geq 0$ and some constant 
$C := C(a,\beta,\lambda_*) > 0$ independent of $t$. 

At this point, we will translate the escape time problem 
to studying the density 
\begin{equation}
\begin{aligned}
	\mathbb{P}[ \tau_{B^c} \geq t | Z_0 = x ] 
	&= 
		\mathbb{P}\left[ \left. 
		\sup_{s \in [0,t]} \frac{1}{2} r(Z_s)^2 \leq \frac{a^2}{2\beta}
		\right| Z_0 = x \right] 
		& 
		\\ 
	&\leq 
		\mathbb{P} \left[ \left. 
			\frac{1}{2} r(Z_t)^2 \leq \frac{a^2}{2\beta} 
			\right| Z_0 = x \right] 
			& 
			\text{ including the escaped and returned }
			\\ 
	&\leq 
		\mathbb{P} \left[ \left. 
			\frac{1}{2} (r^{(3)}_t)^2 \leq \frac{a^2}{2\beta} 
			\right| r^{(3)}_t = \wt{r}(x) \right] 
			& 
			\text{ using comparison theorems }
			\\ 
	&= 
		\int_0^{\frac{a^2}{2\beta}} 
			f(y;t)
		\, dy 
		& 
		\text{ using \cref{cor:cir_density} }
		\\ 
	&\leq 
		\frac{a^2}{2 \beta} C e^{ - \lambda_* t / 2 } 
		& 
		\text{ using \cref{lm:cir_density_est} } \,, 
\end{aligned}
\end{equation}
which is the desired result. 
Furthermore, 
\cref{prop:escape_lyapunov} implies 
$W(x)$ is a Lyapunov function on $B$ 
with $\theta = \frac{\lambda_*}{4}$.

\end{proof}

\subsection{Lyapunov Condition Away from Critical Points}

Starting from this section, 
we will establish the Lyapunov condition in different regions. 
We will first study the region away from all critical points. 

\begin{lemma}
[Morse Gradient Estimate]
\label{lm:morse_grad_est}
Suppose $F$ satisfies 
\cref{asm:app_c3,asm:app_morse}. 
Let $\mathcal{C} := \mathcal{S} \cup \mathcal{X}$ 
be the set of critical points of $F$. 
Then there exists a constant $0 < C_F \leq 1$ such that 
\begin{equation}
	| \grad F(x) |_g \geq C_F \, d_g(x, \mathcal{C}) \,, 
\end{equation}
where $d_g(x, \mathcal{C})$ denotes 
the smallest geodesic distance between $x$ 
and a critical point $y \in \mathcal{C}$. 

In particular, we can write the constant explicitly as 
\begin{equation}
	C_F := \min \left( \, 
			1 \,, \, 
			\frac{ \lambda_* }{2} \,, \, 
			\inf_{ x : d_g(x, \mathcal{C}) > \frac{\lambda_*}{K_3} } 
				\frac{ |\grad F|_g }{ d_g(x, \mathcal{C}) } 
				\, 
			\right) \,, 
\end{equation}
where $\lambda^*$ is from assumption \cref{asm:app_morse}. 

\end{lemma}

\begin{proof}

Intuitively, if we have a quadratic function 
$F(x) = \frac{1}{2} x^\top A x$, then the gradient is zero at $x=0$. 
For a full rank Hessian $A$, the magnitude of the gradient is 
$|\grad F(x)| = |Ax| \geq \lambda_* |x|$, 
where $\lambda_* > 0$ is the gap between zero and the nearest eigenvalue. 
We basically extend this linear lower bound to our setting in two steps. 

First, observe that since 
\begin{equation}
	\inf_{ x : d(x, \mathcal{C}) > \frac{\lambda_*}{K_3} } 
		\frac{ |\grad F|_g }{ d(x, \mathcal{C}) } 
	> 0 \,, 
\end{equation}
due to the fact that $M$ is compact 
and $\grad F$ is non-zero away from critical points. 
Therefore it is sufficient to consider only 
$x \in M$ such that 
$d(x,\mathcal{C}) \leq \frac{\lambda_*}{K_3}$. 

Let $y \in \mathcal{C}$ be a critical point such that 
$d(x,y) = d(x, \mathcal{C})$. 
At this point we consider the normal coordinates centered at $y$, 
and we use the quadratic approximation for Hessian 
from \cite[Proposition 10.55]{boumal2020introduction}, 
which says if $\nabla^2 F$ is Lipschitz then 
\begin{equation}
	| P_{x,y} \grad F(x) - \grad F(y) - \nabla^2 F(y)^\sharp[ \log_y x ] |_g 
	\leq \frac{K_3}{2} d_g(x,y)^2 \,. 
\end{equation}

In our case, we have that $\grad F(y) = 0$ and 
$\nabla^2 F(y)^\sharp[\log_y x] = \de_{ij} F(0) x^j \de_i$ 
which is the same as the Euclidean Hessian. 
Since the direction $\log_y x = x$ is orthogonal to the kernel of 
$\nabla^2 F(y)^\sharp$ 
(see \cite[Chapter 9 Problem 1]{do1992riemannian}), 
we can let $(\lambda_i, v_i)_{i\in[nd]}$ 
be the eigenvalues and (orthonormal) eigenvectors of the Hessian matrix $\de_{ij} F(0)$, 
and write 
\begin{equation}
	x = \sum_{i=1}^{nd} c_i v_i \,, \quad 
	\text{ where } c_i = 0 
	\text{ for all } i \text{ such that } \lambda_i = 0 \,. 
\end{equation}

Using this decomposition we can write 
\begin{equation}
	|\nabla^2 F(y)^\sharp[\log_y x]|_g^2 
	= 
		\left| \nabla^2 F(y)^\sharp \sum_{i=1}^{nd} c_i v_i 
		\right|^2 
	= 
		\left| \sum_{i=1}^{nd} c_i \lambda_i v_i \right|^2 
	= 
		\sum_{i=1}^{nd} c_i^2 \lambda_i^2 
	\geq 
		\lambda_*^2 
		\sum_{i=1}^{nd} c_i^2 
	= \lambda_*^2 d_g(x,y) \,, 
\end{equation}
and combining with the quadratic approximation we get 
\begin{equation}
\begin{aligned}
	|\grad F(x)|_g 
	&= 
		| P_{x,y} \grad F(x) |_g \\ 
	&\geq 
		|\nabla^2 F(y)^\sharp[ \log_y x ]|_g 
		- |P_{x,y} \grad F(x) - \nabla^2 F(y)^\sharp[ \log_y x ]|_g 
		\\ 
	&\geq 
		\lambda_* d_g(x,y) 
		- \frac{K_3}{2} d_g(x,y)^2 
		\\ 
	&\geq 
		\frac{\lambda_*}{2} d_g(x,y) \,, 
		\quad 
		\text{ whenever } 
		d_g(x,y) \leq \frac{\lambda_*}{K_3} \,. 
\end{aligned}
\end{equation}

At this point, we can recover the desired $C_F > 0$ 
by taking the minimum of the two constants.

\end{proof}

\begin{lemma}
[Lyapunov Condition Away From Critical Points]
\label{lm:lyapunov_away_from_crit}
Suppose $F$ satisfies 
\cref{asm:app_c3,asm:app_morse,asm:app_unique_min}. 
Let $\mathcal{C}$ be the set of critical points. 
Then for all choices of $a, \beta > 0$ such that 
\begin{equation}
\begin{aligned}
	a^2 
	&\geq 
		\frac{ 24 K_2 nd }{ C_F^2 } 
		\,, 
\end{aligned}
\end{equation}
we have the following Lyapunov condition 
\begin{equation}
	\frac{\Delta F(x)}{2} 
	- \frac{\beta}{4} | \grad F(x) |_g^2 
	\leq 
		- K_2 nd 
		\,, \quad 
	\forall \, x \in M : d(x, \mathcal{C})^2 
		\geq \frac{a^2}{4\beta} \,. 
\end{equation}
\end{lemma}

\begin{proof}

We begin by directly computing the Lyapunov condition 
using \cref{lm:morse_grad_est}
\begin{equation}
	\frac{\Delta F(x)}{2} 
	- \frac{\beta}{4} | \grad F(x) |_g^2 
	\leq 
		\frac{ K_2 nd }{ 2 } 
		- \frac{\beta}{4} C_F^2 d(x, \mathcal{C})^2 \,, 
\end{equation}
and using the choice 
$d(x, \mathcal{C})^2 \geq \frac{a^2}{4\beta}$, 
we can get 
\begin{equation}
	\frac{ K_2 nd }{ 2 } 
		- \frac{\beta}{4} C_F^2 d(x, \mathcal{C})^2
	\leq 
		\frac{ K_2 nd }{ 2 } 
		- \frac{\beta}{4} C_F^2 \frac{a^2}{4\beta} \,. 
\end{equation}

Now we will use the choice that 
\begin{equation}
	a^2 
	\geq 
		\frac{ 24 K_2 nd }{ C_F^2 } \,, 
\end{equation}
which gives us the desired result of 
\begin{equation}
	\frac{ K_2 nd }{ 2 } 
		- \frac{\beta}{4} C_F^2 \frac{a^2}{4\beta} 
	\leq 
		- K_2 nd \,. 
\end{equation}

\end{proof}

\subsection{Poincar\'{e} Inequality Near Global Minimum}

Even if the global minima set itself $\mathcal{X}$ is convex, 
the neighbourhood around $\mathcal{X}$ defined by 
$U := \left\{ x \in M \, | \, d_g(x,\mathcal{X})^2 \leq \frac{a^2}{\beta} \right\} $
may not be convex on positively curved manifolds. 
For a simple counterexample, we can consider any greater circle on $S^2$, which forms a convex set, however, any $\epsilon>0$ expansion of it becomes immediately nonconvex. 

As a result, we will need to adapt the classical Bakry--Emery criterion on convex sets to the slightly non-convex manifolds. 
In particular, we will mostly base our results on \cite[Chapter 3]{wang2014analysis} and \cite{cheng2017functional} where we need require some control on the second fundamental form of $\de U$ (with respect to the inward pointing normal). 
Fortunately, distance functions $r:M \to \mathbb{R}_+$ generate a very simple scalar second fundamental form where $\sff = \nabla^2 r$ is the Hessian of the distance function \cite[Section 3.2.2]{petersen2006riemannian}. 
In our case, the distance function is naturally $r(x) := d_g(x, \mathcal{X})$, and it remains to control the eigenvalues of the Hessian $\nabla^2 r$.

\begin{lemma}
[Bounding the Second Fundamental Form]
\label{lm:sff_bound}
Under \cref{asm:app_unique_min,asm:app_morse} and $\frac{a}{\sqrt{b}} \leq \frac{\pi}{4}$, we have that either 
$\de U$ is convex, or 
\begin{equation}
	|\sff|_g = |\nabla^2 r|_g \leq \frac{2a}{\sqrt{\beta}} \,. 
\end{equation}
\end{lemma}

\begin{proof}

The first case follows from if $\mathcal{X}$ is a point, which then small geodesic balls around a pole do not cross the greater equator. 
Therefore $\frac{\pi}{4}$ is sufficient to prevent this crossing. 

Therefore we consider the second case under \cref{asm:app_morse}, where the boundaries of $\mathcal{X}$ are all given by geodesics, 
therefore the second fundamental form of $\mathcal{X}$ is identically zero. 
We observe that the second fundamental form can be written as the shape operator $s:\mathfrak{X}(\de U) \to \mathfrak{X}(\de U)$ by 
\begin{equation}
	sX = \pi_\top \left( \nabla_X \grad r \right) \,, 
\end{equation}
where $\pi_\top$ is the tangential projection from $TU$ to $T\de U$. 
Therefore we can let $\gamma_t$ be a unit speed geodesic curve in $\de U$ (not necessarily a geodesic in $M$), and compute the parallel transport of a unit vector $V_t$, more precisely if $\dot \gamma_t |_{t=0} = X$ then 
\begin{equation}
	sX = \pi_\top (D_t V_t) |_{t=0} \,, 
\end{equation}
where $V_0 = \grad r$ and $D_t = \nabla_{\dot \gamma_t}$. 

We will first consider a simple case of $n=1,d=2$ where $\mathcal{X}$ is exactly the greater circle. 

\vspace{0.2cm}
\noindent
\textbf{Simple Case: $n=1,d=2$.} 
In this case, we can define without loss of generality $V_0 = [\cos \epsilon, \sin \epsilon, 0]^\top$ where $\epsilon = \frac{a}{\sqrt{\beta}}$, and that 
\begin{equation}
	V_t = [ \cos \epsilon, \sin \epsilon \cos \theta_t , \sin \epsilon \sin \theta_t ]^\top \,, 
\end{equation}
where $\theta_t = \frac{t}{\cos \epsilon}$ is the angle parameterization of a unit speed geodesic on the circle with radius $\cos \epsilon$. 
Therefore we can calculate that 
\begin{equation}
	|D_t V_t|_g^2 
	= \tan^2 \epsilon ( \sin^2 \theta_t + \cos^2 \theta_t ) 
	= \tan^2 \epsilon 
	\leq 4 \epsilon^2 \,, 
\end{equation}
where we used the fact that $\epsilon = \frac{a}{\sqrt{\beta}} \leq \frac{\pi}{4}$ to get the bound $\tan \epsilon \leq 2 \epsilon$ (since $\tan \frac{\pi}{4} = 1 < 2 \frac{\pi}{4} \approx 1.57$ and $\tan$ is convex increasing on $[0,\pi/4]$). 
Consequently, we can conclude that $|\sff|_g \leq 2 \epsilon = \frac{2a}{\sqrt{\beta}}$.

\vspace{0.2cm}
\noindent
\textbf{Higher Dimensional Sphere: $n=1$, general $d$.} 
Without loss of generality, we can rotate the sphere such that we have 
\begin{equation}
	V_t = [\cos \epsilon, \sin \epsilon \cos \theta_t, \sin \epsilon \sin \theta_t, 0, \cdots, 0]^\top\ ,, 
\end{equation}
where $\theta_t$ is the angle parameterizing the unit speed geodesic on $\de U$, which must also be $\theta_t = \frac{t}{\cos \epsilon}$. 
Consequently the rest of the calculations follow similarly by padding zeros and we again get $|\sff|_g \leq \frac{2a}{\sqrt{\beta}}$ as desired. 

\vspace{0.2cm}
\noindent
\textbf{Product of Spheres: General $n,d$.} 
In this case, for any point on $\de U$, we let $\epsilon_i := d_g(x^{(i)}, \mathcal{X}^{(i)})$ be the distance from the minimum on the $i$-th sphere. 
This implies that we have 
\begin{equation}
	\frac{a^2}{\beta} = \sum_{i=1}^n \epsilon_i^2 \,, 
\end{equation}
and we have that on each sphere $|\sff^{(i)}|_g \leq 2 \epsilon_i$, which gives us 
\begin{equation}
	|\sff|_g \leq 2 \max_i \epsilon_i \leq \frac{2a}{\sqrt{\beta}} \,, 
\end{equation}
which is the desired result. 

\end{proof}

We will recover a Poincar\'e inequality on $U$ using the following two results.

\begin{proposition}
[Corollary 3.5.2 of \cite{wang2014analysis}]
\label{prop:wang_poincare_nonconvex}
If $\phi \in C^2_b(U)$ is such that $\phi \geq 1$ and $\sff \geq -N \log \phi$, then the first non-zero Neumann eigenvalue of $L$ (the inverse Poincar\'e constant) satisfies 
\begin{equation}
	\lambda_1 
	\geq 
	\frac{ 1 }{\|\phi\|_\infty^4} 
	\left( \frac{\pi^2}{D^2} + \frac{K_\phi}{2} 
	\right) \,, 
\end{equation}
where $D = \text{diam}(U)$ and 
\begin{equation}
	K_\phi := \inf_U \left\{ 
		\phi^2 K + \frac{1}{2} L \phi^2 - |\grad \phi|_g^2 \, |\grad F|^2 
		- (nd-2) |\grad \phi|_g 
	\right\} \,. 
\end{equation}
\end{proposition}

\begin{proposition}
[Theorem 3.2 of \cite{cheng2017functional}]
\label{prop:cheng_constants}
Suppose $-\sigma \leq |\sff|_g \leq \theta$, $\rho_\de = d_g(x, \de U)$ be smooth on $\de_{r_0} M = \{ x \in U | \rho_\de(x) \leq r_0 \}$ for some constant $r_0$, $|\grad F|_g \leq \delta$, the sectional curvature $\text{sect} \leq k$, and $\nabla^2 F + \Ric_g \geq K g$. 
We will also define 
\begin{equation}
	r_1 = \min\left( r_0, \frac{1}{\sqrt{k}} \arcsin \sqrt{\frac{k}{k+\theta^2}} \right) \,, 
\end{equation}
then we can replace $K_\phi$ with 
\begin{equation}
	K_1 = K - \sigma \left(\delta + \frac{nd}{r_1} \right) - \sigma^2 \,, 
\end{equation}
and replace $\|\phi\|_\infty$ with $e^{\frac{1}{2} \sigma nd r_1}$. 
\end{proposition}

Now we return to deriving the Poincar\'e inequality on $U$. 
Here we define $\nu|_{U}$ 
as the probability measure $\nu$ restricted 
to the set $U$, more precisely 
\begin{equation}
\label{eq:nu_restriction_ua}
	\nu|_{U}(x) 
	:= 
		\frac{ \nu(x) }{ 
			\int_{U} \, \nu(y) \, dy 
		} \mathds{1}_{U}(x) \,. 
\end{equation}

\begin{lemma}
[Poincar\'e Inequality on $U$]
\label{lm:poincare_near_min}
Suppose $F$ satisfies 
\cref{asm:app_c3,asm:app_morse,asm:app_unique_min} 
and $d \geq 2$. 
Then for all choices of $a,\beta > 0$ such that 
\begin{equation}
\begin{aligned}
	a^2 
	&\geq 
		\frac{6 K_2 nd}{ C_F^2 } 
		\,, \quad 
	\beta 
	\geq
		a^2 
		\max\left( 
			\, \frac{ 4 K_3^2 }{ \lambda_{*}^2 } \,, 
			\, K_3^2 a^4 \, 
		 \right) \,, 
\end{aligned}
\end{equation}
where $\lambda_*$ is defined in \cref{asm:app_morse}, 
we have that the Markov triple $(U ,\nu|_{U},\Gamma)$, 
with $\nu|_{U}$ defined in \cref{eq:nu_restriction_ua}, 
satisfies a Poincar\'{e} inequality with constant 
\begin{equation}
	\kappa_{U} 
	= 
		\begin{cases}
		 	\frac{\lambda_*}{2} \,, 
		 	& \text{ if the global minimum $\mathcal{X}$ is a unique point,} \\ 
		 	\frac{K_*(d-1)}{4 \beta} \,, 
		 	& \text{ otherwise.}
		\end{cases} \,, 
\end{equation} 
where $K_* = \exp \left( \frac{ -2 C_F^2}{ K_2 K_3 } \right) $. 
\end{lemma}

\begin{proof}

We will first consider the unique global minimum case, 
where $\mathcal{X} = \{x^*\}$, 
then we can write 
$U := \{ x \in M | d_g(x, x^*)^2 \leq \frac{a^2}{\beta} \}$. 
We start by observing that eigenvalues $\lambda_i$ is a $1$-Lipschitz function of the Hessian $\nabla^2 F$, which is in turn a $K_3$-Lipschitz function, therefore we have that 
\begin{equation}
	| \lambda_i (\nabla^2 F(x)) - \lambda_i (\nabla^2 F(x^*)) | 
	\leq K_3 d_g(x, x^*) \,, 
\end{equation}
which means we can write 
\begin{equation}
	\lambda_{\min} (\nabla^2 F(x)) 
	\geq 
		\lambda_{\min} (\nabla^2 F(x^*)) 
		- K_3 d_g(x, x^*) 
	\geq 
		\frac{\lambda_*}{2} \,, 
		\quad
		\text{ whenever } 
		d_g(x, x^*) \leq \frac{\lambda_*}{2K_3} \,, 
\end{equation}
where we also observe this is satisfied since 
$\frac{a^2}{\beta} \leq \frac{\lambda_*^2}{4 K_3^2}$. 

At the same time, we can also consider the case 
when the global minimum is not unique. 
Here we will directly compute the constants in \cref{prop:wang_poincare_nonconvex,prop:cheng_constants} where 
\begin{equation}
	\kappa_U 
	= \lambda_1^{-1} 
	= 
	\|\phi\|_\infty^4 
	\left( \frac{\pi^2}{D^2} + \frac{K_\phi}{2} \right)^{-1} \,. 
\end{equation}

Starting with $\|\phi\|_\infty^4 = \exp( 2 \sigma nd r_1 )$, where we observe that using the fact that $k=1$ on a sphere and $\sigma = \theta = \frac{2a}{\sqrt{\beta}}$ from \cref{lm:sff_bound} to get that 
\begin{equation}
	r_1 = \arcsin \frac{1}{\sqrt{1 + \frac{4a^2}{\beta}}} \,. 
\end{equation}
And since that $0 < \frac{2a}{\sqrt{\beta}} \leq 1$ we must have that $r_1 \in [\frac{\pi}{4}, \frac{\pi}{2}]$. 
Using the conditions on $a,\beta$ we have that 
\begin{equation}
	\frac{2a}{\sqrt{\beta}} 
	\leq 
		2 \max \left( \frac{2K_3}{\lambda_*}, K_3 a^2 \right)^{-1} 
	= 
		\min \left( \frac{\lambda_*}{K_3} \,, \frac{C_F^2}{ 3K_2 K_3 nd } \right) \,, 
\end{equation}
which further implies that we have 
\begin{equation}
	\| \phi \|_\infty^4 
	\leq 
	\exp \left( \pi \min \left( \frac{\lambda_*}{K_3} \,, \frac{C_F^2}{ 3K_2 K_3 nd } \right) \right) 
	\leq 
	\exp \left( \frac{ 2 C_F^2}{ K_2 K_3 } \right) 
	= K_*^{-1} \,.
\end{equation}

Next we observe that $D^2 \leq \pi^2 n$ and therefore it's sufficient to bound $K_\phi \leq \frac{1}{2n}$. 
Now we write out the three terms in $K_\phi$ 
\begin{equation}
	K_\phi = K - \sigma \left( \delta + \frac{nd}{r_1} \right) - \sigma^2 \,, 
\end{equation}
where we can plug in $K=d-1, \sigma = \frac{2a}{\sqrt{\beta}}, \delta = K_1, r_1 = \frac{\pi}{4}$ for an upper bound. 
Observe it's sufficient to bound each of the three terms to be less than $\frac{1}{3}$ which we observe 
\begin{equation}
\begin{aligned}
	\frac{a}{\sqrt{\beta}} K_1 
	&\leq \frac{C_F^2}{6 K_1 K_2 K_3 nd} \leq \frac{1}{3} \,, \\ 
	\frac{a}{\sqrt{\beta}} \frac{nd\pi}{2} 
	&\leq \frac{2C_F^2}{6 K_2 K_3} 
	\leq \frac{1}{3} \,, \\ 
	\frac{4 a^2}{\beta} 
	&\leq \frac{4 C_F^4}{ 36 K_2^2 K_3^3 n^2 d^2} 
	\leq \frac{1}{3} \,. 
\end{aligned}
\end{equation}

Finally to adjust for the temperature change, we need to add a factor of $\beta$ to get 
\begin{equation}
	\kappa_U = \frac{K_*}{\beta} \left( \frac{1}{n} + \frac{1}{2} \left( 
	d-1 - 1
	\right) \right) 
	\geq \frac{K_*(d-1)}{4\beta} \,, 
\end{equation}
where we used $d\geq 2$. 

\end{proof}

\subsection{Poincar\'{e} Inequality}

In this section, we will establish a Poincar\'{e} inequality 
for the Markov triple $(M,\nu,\Gamma)$ through the Lyapunov approach. 
We will define the following partition of the manifold $M$ with $r^2 = \frac{a^2}{4\beta}$ 
\begin{equation}
\begin{aligned}
	B_r &:= 
		\left\{ x \in M \, | \, d_g(x,\mathcal{S})^2 \leq r \right\} 
		\,, \\ 
	B_{2r} 
	&:= \left\{ x \in M \, | \, d_g(x, \mathcal{S}) < 2 r \right\} 
		\,, \\ 
	U &:= \left\{ x \in M \, | \, d_g(x,\mathcal{X})^2 \leq 2r \right\} 
		\,, \\ 
	A &:= M \setminus (U \cup B_{r}) \,, 
\end{aligned}
\end{equation}
where $\mathcal{X}$ is the set of global minima, 
and $\mathcal{S}$ is the set of critical points that are not global minima. 
Observe here that $B_{2r}$ overlaps with $A$ by exactly the shell $B_{2r} \setminus B_r$.

We will next establish the Poincar\'{e} inequality on the entire manifold $M$ 
using a combination of the Lyapunov methods in \cite{bakry2008simple,cattiaux2013poincare}, which are made more precise in \cref{thm:lyapunov_poincare_bakry,prop:poincare_adapted}. 
Before we start, we will briefly summarize the results of the previous sections up to this point as follows 
\begin{equation}
\label{eq:lyapunov_poincare_sets}
\begin{aligned}
    &LW_2 = -\theta W_2 \,, \quad x \in B_{2r} \,, \\ 
    &LW_1 \leq -\lambda W_1 + b \mathds{1}_U \,, \quad x \in B_r^c \,, \\ 
    &(U, \nu|_U, \Gamma|_U) \in \text{PI}(\kappa_U) \,, 
\end{aligned}
\end{equation}
where $\theta = \frac{\lambda_*}{4}, \lambda = K_2 nd$, and $b = \frac{3}{2} K_2 nd$ using \cref{asm:app_c3}. 
Furthermore 
\begin{equation}
	\kappa_{U} 
	= 
		\begin{cases}
		 	\frac{\lambda_*}{2} \,, 
		 	& \text{ if the global minimum $\mathcal{X}$ is a unique point,} \\ 
		 	\frac{K_*(d-1)}{4 \beta} \,, 
		 	& \text{ otherwise.}
		\end{cases} \,, 
\end{equation} 
where $K_* = \exp \left( \frac{ -2 C_F^2}{ K_2 K_3 } \right) $.

Next we will derive a more generic Poincar\'e inequality under these conditions. 

\begin{proposition}
[Generic Poincar\'e Inequality on $M$]
\label{prop:poincare_lyapunov_generic}
Suppose we have the Lyapunov and Poincar\'e conditions in \cref{eq:lyapunov_poincare_sets}, and we choose 
\begin{equation}
    a^2 \geq 32 \left( \frac{1}{\lambda} + \frac{1}{\theta} \right) \,,  
\end{equation}
then $(M,\nu,\Gamma)$ satisfies a Poincar\'e inequality with constant 
\begin{equation}
    \kappa^{-1} = 8 \left( \frac{1}{\lambda} + \frac{1}{\theta} + \frac{b}{\lambda \kappa_U} \right) \,. 
\end{equation}
\end{proposition}

\begin{proof}

Firstly, we note that 
\begin{equation}
    \Var_\nu(\wt{f}) \leq \int (\wt{f} + c)^2 \, d\nu \,, 
\end{equation}
for all constant offsets $c \in \mathbb{R}$. 
Therefore, without loss of generality we will choose the offset $f = \wt{f} + c$ such that 
\begin{equation}
\label{eq:offset_u}
    \int_U f \,d\nu = 0 \,, 
\end{equation}
and we choose to bound $\Var_\nu(f)$ with $\int f^2 \, d\nu$.

We start by decomposing the function into two components using a smooth partition function $\chi(x) = \psi(d_g(x, \mathcal{S}))$, where $\psi$ is a smooth step function constructed in \cref{lm:smooth_partition} with parameters $r_1 = r, r_2 = 2r$ 
\begin{equation}
\begin{aligned}
    \int f^2 \, d\nu 
    &= 
    \int (f(1-\chi) + f\chi)^2 \, d\nu \\ 
    &\leq 
    2\int_{B_{2r}} f^2 (1-\chi)^2 \, d\nu 
    + 2\int_{B_r^c} f^2 \chi^2 \, d\nu \,. 
\end{aligned}
\end{equation}

For the first integral, we will use the Lyapunov function $W_2$ and the fact that $1-\chi = 0$ on $\de B_{2r}$ to get 
\begin{equation}
	\int_{\de B_{2r}} \frac{f^2 (1-\chi)^2}{W_2} \nu \langle \grad W_2, dn \rangle_g = 0 \,, 
\end{equation}
where $n$ is the outward pointing normal. 
Therefore we can using Green's identity and drop the boundary integral to write 
\begin{equation}
\begin{aligned}
    2\int_{B_{2r}} f^2 (1-\chi)^2 \, d\nu 
    &= 
    \frac{2}{\theta} \int_{B_{2r}} \frac{-LW_2}{W_2} f^2 (1-\chi)^2 \, d\nu \\ 
    &= 
    \frac{2}{\theta} \int _{B_{2r}} \Gamma( f(1-\chi) ) \, d\nu \,, 
\end{aligned}
\end{equation}
where we write $\Gamma(f) = \Gamma(f,f) = \frac{1}{\beta} |\grad f|_g^2$. 

Using the inequality $\Gamma(fg, fg) \leq 2( f^2 \Gamma(g, g) + g^2 \Gamma(f,f) )$ we can also write 
\begin{equation}
    \frac{2}{\theta} \int _{B_{2r}} \Gamma( f(1-\chi) ) \, d\nu
    \leq 
    \frac{4}{\theta} \int \Gamma(f) \, d\nu 
    + \frac{4 \|\Gamma(\chi)\|_\infty }{\theta} 
    \int f^2 \, d\nu \,. 
\end{equation}

For the second integral, we will use the Lyapunov function $W_1$, the fact that $\chi$ vanishes on $\de B_r$, and the Poincar\'e inequality on $U$ (note $\int_U f \, d\nu = 0$) to write 
\begin{equation}
\begin{aligned}
    2 \int_{B_r^c} f^2 \chi^2 \, d\nu 
    &\leq 
    \frac{2}{\lambda} \int_{B_r^c} \frac{-LW_1}{W_1} f^2 \chi^2 \, d\nu 
    + \frac{2}{\lambda} \int_U \frac{b}{W_1} f^2 \, d\nu 
    \\ 
    &\leq 
    \frac{2}{\lambda} \int \Gamma(f\chi) \, d\nu 
    + \frac{2b}{\lambda \kappa_U} \Gamma(f) \, d\nu \,. 
\end{aligned}
\end{equation}

Similarly, we will use $\Gamma(fg, fg) \leq 2( f^2 \Gamma(g, g) + g^2 \Gamma(f,f) )$ to write 
\begin{equation}
    2 \int_{B_r^c} f^2 \chi^2 \, d\nu 
    \leq 
    \frac{4 \|\Gamma(\chi)\|_\infty}{\lambda} \int f^2 \, d\nu 
    + \left( \frac{4}{\lambda} + \frac{4b}{\lambda \kappa_U} \right) \int \Gamma(f) \, d\nu \,. 
\end{equation}

Putting this together, we have that 
\begin{equation}
    \int f^2 \, d\nu 
    \leq 
    4 \| \Gamma(\chi) \|_\infty \left( \frac{1}{\lambda} + \frac{1}{\theta} \right) \int f^2 \, d\nu 
    + 4 \left( \frac{1}{\lambda} + \frac{1}{\theta} + \frac{b}{\lambda \kappa_U} \right) \int \Gamma(f) \, d\nu \,,  
\end{equation}
which we can manipulative to write the Poincar\'e inequality 
\begin{equation}
    \int f^2 \, d\nu 
    \leq 
    \frac{4 \left( \frac{1}{\lambda} + \frac{1}{\theta} + \frac{b}{\lambda \kappa_U} \right) }{ 1 - 4 \| \Gamma(\chi) \|_\infty \left( \frac{1}{\lambda} + \frac{1}{\theta} \right) } 
    \int \Gamma(f) \, d\nu \,. 
\end{equation}

It remains to compute the constant. 
To this goal, we will use the fact that $\|\Gamma(\chi)\|_\infty \leq \frac{1}{\beta (r+\epsilon)^2}$ for all $\epsilon>0$, and $r = \frac{a}{2\sqrt{\beta}}$ to write  
\begin{equation}
    \| \Gamma(\chi) \|_\infty \leq \frac{1}{\beta r^2} = \frac{4}{a^2} \,. 
\end{equation}

Therefore, choosing $a^2 \geq 32 (\frac{1}{\lambda} + \frac{1}{\theta})$, we get that 
\begin{equation}
    1 - 4 \| \Gamma(\chi) \|_\infty \left( \frac{1}{\lambda} + \frac{1}{\theta} \right)
    \geq 
    1 - 4 \frac{4}{a^2} \left( \frac{1}{\lambda} + \frac{1}{\theta} \right)
    \geq \frac{1}{2} \,, 
\end{equation}
which leads to use the desired Poincar\'e constant 
\begin{equation}
    \kappa^{-1} = 
    8 \left( \frac{1}{\lambda} + \frac{1}{\theta} + \frac{b}{\lambda \kappa_U} \right) \,. 
\end{equation}

\end{proof}

\begin{proposition}
[Poincar\'{e} Inequality on $M$]
\label{prop:lyapunov_poincare}
Suppose $F$ satisfies 
\cref{asm:app_c3,asm:app_morse,asm:app_unique_min} 
and $d \geq 2$. 
Then for all choices of $a,\beta > 0$ such that 
\begin{equation}
\begin{aligned}
	a^2 
	&\geq 
		\max\left( \frac{24 K_2 nd}{ C_F^2 } \,, 
		\frac{160}{\lambda_*} 
		\right)
		\,, \quad 
	\beta 
	\geq
		a^2 
		\max\left( 
			\, \frac{ 4 K_3^2 }{ \lambda_{*}^2 } \,, 
			\, K_3^2 a^4 \, 
		 \right) \,, 
\end{aligned}
\end{equation}
where $\lambda_{*}$ is defined in \cref{asm:app_morse}, 
we have that the Markov triple $(M, \nu, \Gamma)$ 
satisfies a Poincar\'{e} inequality with constant 
\begin{equation}
	\kappa
	= 
		\left[ 
		\frac{25}{\lambda_*} + \frac{15}{2 \kappa_U} 
		\right]^{-1} 
	\geq 
		\begin{cases}
			\frac{\lambda_*}{40} \,, 
			& \text{ if the global minimum is unique, } \\
			\frac{K_* \lambda_* }{55 \beta} \,, 
			& \text{ otherwise. } 
		\end{cases}
		\,, 
\end{equation} 
where $\kappa_U$ is the Poincar\'{e} constant 
from \cref{lm:poincare_near_min}, 
and $K_* = \exp \left( \frac{ -2 C_F^2}{ K_2 K_3 } \right) $. 

\end{proposition}

\begin{proof}

We first recall the constants $\theta = \frac{\lambda_*}{4}, \lambda = K_2 nd$, and $b = \frac{3}{2} K_2 nd$ using \cref{asm:app_c3}. 
Furthermore 
\begin{equation}
	\kappa_{U} 
	= 
		\begin{cases}
		 	\frac{\lambda_*}{2} \,, 
		 	& \text{ if the global minimum $\mathcal{X}$ is a unique point,} \\ 
		 	\frac{K_*(d-1)}{4 \beta} \,, 
		 	& \text{ otherwise.}
		\end{cases} \,, 
\end{equation} 
where $K_* = \exp \left( \frac{ -2 C_F^2}{ K_2 K_3 } \right) $.

Then we use the fact that $\lambda = K_2 nd \geq 1 \geq \lambda_*^{-1}$ to simplify the condition on $a^2 \geq 160 \lambda_*^{-1}$. 

The unique minimum value of $\kappa_U$ follows directly from plugging in the constants, where as the non-unique case follows from the fact that $d-1\geq 1$ and $K_* \geq 1$, so we can write 
\begin{equation}
	\frac{5}{\lambda_*} + \frac{6\beta}{K_* (d-1)} \leq \frac{11\beta}{\lambda_*K_*} \,, 
\end{equation}
which gives us the desired result. 

\end{proof}

\begin{remark*}
It is tempting to consider the extension of $W$ to outside the set $B$ 
and use the easier Lyapunov condition from \cite{bakry2008simple} of 
\begin{equation}
	\frac{LW}{W} \leq - \theta + b \mathds{1}_{B^c} \,. 
\end{equation}
However, as discussed by \cite{cattiaux2013poincare}, 
it is unclear how we can estimate the constant $b$ as 
the extension is only known to qualitatively exist via Whitney extension. 
Therefore, we instead follow \cite{cattiaux2013poincare} 
to use a different route without considering the extension at all, 
using a smooth partition function to handle the boundary. 
The full calculations are included in 
\cref{lm:smooth_partition,prop:poincare_adapted}. 
\end{remark*}

\subsection{Logarithmic Sobolev Inequality}

We will next adapt the results of 
\cite[Theorem 1.9]{cattiaux2010note} 
and \cite[Theorem 3.15]{menz2014poincare} 
but with a significant simplification 
since $W_2$ on $M$ is trivially bounded. 
Before we start, we will recall a couple of standard results.

\begin{theorem}
\citep[Corollary 20.13]{villani2008optimal}
\label{thm:hwi_inequality}
Let $(M,g)$ be a Riemannian manifold, 
$V \in C^2(M)$ such that $\nu = e^{-V} \in \mathcal{P}_2(M)$, 
i.e. $\nu$ is a probability measure and 
\begin{equation}
	\int_M d_g(x,x_0)^2 \, d\nu < \infty \,, \quad 
	\forall x_0 \in M \,. 
\end{equation}
Furthermore assume there exists $K>0$ such that 
\begin{equation}
	\nabla^2 V + \Ric_g \geq -K g \,, 
\end{equation}
then for all $\mu \in \mathcal{P}_2(M)$, 
we have that 
\begin{equation}
	H_\nu( \mu ) 
	\leq 
		W_2( \mu, \nu) \sqrt{ \wt I_\nu(\mu) } 
		+ \frac{K}{2} W_2^2(\mu, \nu) \,, 
\end{equation}
where $\wt I_\nu(\mu) := \int_M \frac{ 1 }{ h } |\grad h|_g^2 \, d\nu$ 
is the Fisher information without adjusting for temperature $\beta$, 
and $h = \frac{d\mu}{d\nu}$. 
\end{theorem}

\begin{theorem}
\citep[Theorem 6.15]{villani2008optimal}
\label{thm:wasserstein_weighted_tv_bound}
Let $\mu,\nu$ be two probability measures on the Polish space 
$(\mathcal{X}, d)$. 
Then for all $p \geq 1$, $p'$ the H\"{o}lder conjugate of $p$, 
and $x_0 \in \mathcal{X}$, 
we have that 
\begin{equation}
	W_p( \mu, \nu ) 
	\leq 
		2^{ \frac{1}{p'} } 
		\left[ 
			\int_{\mathcal{X}} d_g(x,x_0)^p \, 
				d|\mu - \nu|(x)
		\right]^{\frac{1}{p}} \,. 
\end{equation}
\end{theorem}

\begin{proposition}
[Poincar\'e Implies LSI]
\label{prop:poincare_implies_lsi}
Suppose that $(M,g)$ is the product of $n$ unit spheres $S^d$, 
with $F \in C^2(M)$, 
and $\nu = \frac{1}{Z} \exp( -\beta F )$ with $\beta \geq 1$. 
Suppose further that 
\begin{enumerate}
	\item There exists a constant $K > 0$ such that 
		\begin{equation}
			\nabla^2 F + \frac{1}{\beta} \Ric_g 
			\geq - K g \,, 
		\end{equation}
	\item $(M, \nu, \Gamma)$ satisfies a Poincar\'{e} 
		inequality with constant $1 \geq \kappa > 0$. 
\end{enumerate}
Then we have that $(M, \nu, \Gamma)$ satisfies 
a logarithmic Sobolev inequality with constant 
\begin{equation}
	\frac{1}{\alpha} 
	= \frac{11 \beta K n }{\kappa} \,. 
\end{equation}
\end{proposition}

\begin{proof}

We start by defining the ``unadjusted'' Fisher information as 
$\wt I_\nu(\mu) = \beta I_\mu(\mu) = \int |\grad h|^2 / h \, d\nu$, 
where $h = \frac{d\mu}{d\nu}$.  
Then we can apply the HWI inequality 
from \cref{thm:hwi_inequality} 
to $V = \beta F$, 
which gives us 
\begin{equation}
	H_\nu( \mu ) 
	\leq 
		W_2( \mu, \nu ) \sqrt{ \wt I_\nu(\mu) } 
		+ \frac{\beta K}{2} 
			W_2^2(\mu, \nu) \,. 
\end{equation}

Now using the trivial bound 
$W_2(\mu,\nu)^2 \leq \sup_{x,y} d_g(x,y)^2 \leq \pi^2 n$ 
and Young's inequality with 
$ab \leq \frac{\tau}{2} a^2 + \frac{1}{2 \tau} b^2$, 
we can write 
\begin{equation}
\begin{aligned}
	H_\nu( \mu ) 
	&\leq 
		\frac{\tau}{2} \wt I_\nu(\mu) 
		+ \left( \frac{1}{2\tau} + \frac{\beta K}{2} \right) 
			W_2^2(\mu, \nu) 
			\\ 
	&\leq 
		\frac{\tau}{2} \wt I_\nu(\mu) 
		+ \left( \frac{1}{2\tau} + \frac{\beta K}{2} \right) 
		\pi^2 n 
		\\ 
	&\leq 
		\frac{\tau \beta}{2} \wt I_\nu(\mu) 
		+ \left( \frac{1}{2\tau} + \frac{\beta K}{2} \right) 
		\pi^2 n 
		\\ 
	&=: 
		A I_\nu(\mu) + B \,, 
\end{aligned}
\end{equation}
which is a defective logarithmic Sobolev inequality. 

Using a standard tightening argument 
with a Poincar\'{e} inequality via Rothaus' Lemma 
\citep[Proposition 5.1.3]{bakry2013analysis}, 
we have that a tight logarithmic Sobolev inequality 
\begin{equation}
\begin{aligned}
	H_\nu( \mu ) 
	&\leq 
		\left[ A + \frac{B+2}{\kappa} \right] 
		\, I_\nu(\mu) \\ 
	&= 
		\left( \frac{\tau \beta}{2} 
			+ \frac{ \frac{\pi^2 n}{2} \left( \frac{1}{\tau} + \beta K \right) + 2 }{\kappa}
		\right) I_\nu(\mu)
		\\
	&= 
		\left( \frac{\tau \beta}{2} 
			+ \frac{1}{\tau} \frac{\pi^2 n}{2\kappa} 
			+ \frac{\pi^2 n \beta K + 4}{2 \kappa} 
		\right) I_\nu(\mu)
		\,, 
\end{aligned}
\end{equation}
and we can optimize over $\tau$ by choosing 
$\tau = \sqrt{ \frac{\pi^2 n}{\beta \kappa} }$, 
which gives us the final constant of 
\begin{equation}
	\sqrt{ \frac{\pi^2 \beta n}{\kappa} } 
	+ \frac{\pi^2 n \beta K + 4}{2 \kappa} 
	\leq 
		\left(\pi + \frac{\pi^2}{2} + 2\right) \frac{\beta K n }{\kappa} 
	\leq 
		\frac{11 \beta K n }{\kappa} 
		\,. 
\end{equation}

\end{proof}

\begin{theorem}[Logarithmic Sobolev Inequality]
Suppose $F$ satisfies 
\cref{asm:app_c3,asm:app_morse,asm:app_unique_min} 
and $d \geq 2$. 
Then for all choices of $a,\beta > 0$ such that 
\begin{equation}
\begin{aligned}
	a^2 
	&\geq 
		\max\left( \frac{24 K_2 nd}{ C_F^2 } \,, 
		\frac{160}{\lambda_*} 
		\right)
		\,, \quad 
	\beta 
	\geq
		a^2 
		\max\left( 
			\, \frac{ 4 K_3^2 }{ \lambda_{*}^2 } \,, 
			\, K_3^2 a^4 \, 
		 \right) \,, 
\end{aligned}
\end{equation}
where $\lambda_*$ is defined in \cref{asm:app_morse}, 
we have that the Markov triple $(M, \nu, \Gamma)$ 
satisfies $\LSI(\alpha)$ with constant 
\begin{equation}
	\frac{1}{\alpha} 
	= 
		\begin{cases}
		 	440 \frac{K_2 n \beta}{\lambda_*^2} \,, 
		 	& \text{ if the global minimum is unique, } \\ 
		 	605 \frac{K_2 n \beta^2}{K_* \lambda_*} \,, 
		 	& \text{ otherwise, } 
		\end{cases} 
\end{equation}
where $K_* = \exp \left( \frac{ -2 C_F^2}{ K_2 K_3 } \right) $. 
\end{theorem}

\begin{proof}

This result follows directly from \cref{prop:poincare_implies_lsi}, 
where we get $\nabla^2 F \geq - K_2 g$ from \cref{asm:app_c3}, 
$\Ric_g = (d-1)g > 0$, and the Poincar\'e constant $\kappa$ from 
\cref{prop:lyapunov_poincare}. 

\end{proof}

\section{Proof of the Corollaries}
\label{sec:cor_proof}

\subsection{Runtime Complexity Under $\LSI(\alpha)$}
\label{subsec:runtime_lsi}

We start by restating \cref{cor:runtime_complexity_lsi}. 

\begin{corollary}
[Runtime Complexity]
Let $F$ satisfy \cref{asm:c2}, 
and let $(M,\nu,\Gamma)$ satisfy \cref{asm:lsi}. 
Further let the initialization $\rho_0 \in C^1(M)$, 
and $d \geq 3$. 
Then for all choices of 
$\epsilon \in (0,1]$ 
and $\delta \in (0,1)$, 
we can choose 
\begin{equation}
\begin{aligned}
	\beta &\geq 
		\frac{3nd}{\epsilon} 
		\log \frac{ 2 n K_2 }{ \epsilon \, \delta }
			\,, \quad 
	\eta \leq 
		\min\left\{ 
				\frac{2}{3 \alpha} \, , \, 
				\frac{ \alpha \delta^2 \sqrt{\epsilon} }{ 
				180 nd K_2^2 
				\sqrt{ \log \frac{2nK_2}{\epsilon \delta} } } 
			\right\} 
			\,, \\ 
	k &\geq 
		\max\left\{ 
		\frac{3}{2} \,, 
		180 nd \frac{K_2^2}{\alpha^2} 
		\frac{ \sqrt{ \log \frac{2nK_2}{\epsilon\delta} } }{ \delta^2 \sqrt{\epsilon} }
		\right\} 
		\left( 2 \log \frac{2}{\delta} + \log H_\nu(\rho_0) \right) 
		\,, 
\end{aligned}
\end{equation}
such that the unadjusted Langevin algorithm $\{X_k\}_{k\geq 1}$ 
defined in \cref{eq:langevin_algorithm} 
with distribution $\rho_k := \mathcal{L}(X_k)$ satisfies 
\begin{equation}
	\rho_k\left( F - \min_{y \in M} F(y) \geq \epsilon \right) 
	\leq \delta \,. 
\end{equation}
In other words, $X_k$ finds an $\epsilon$-global minimum 
with probability $1 - \delta$. 
\end{corollary}

\begin{proof}

We start by observing that under the conditions we chose, 
we have that by \cref{thm:finite_iteration_KL_bound} 
and $\beta \geq 1$ 
\begin{equation}
	H_\nu( \rho_k ) 
	\leq 
		C_0 \, e^{-\alpha k \eta} 
		+ C_1 \frac{\eta}{\alpha} \,, 
\end{equation}
by \cref{thm:gibbs_high_prob_bound} we have that 
\begin{equation}
		\nu( F \geq \epsilon) \leq \frac{\delta}{2} \,. 
\end{equation}

Next we will observe that 
using the above bound on $\nu(F \geq \epsilon)$ 
and Pinsker's inequality on total variation distance, 
we can write 
\begin{equation}
\begin{aligned}
	\rho_k( F \geq \epsilon ) 
	&\leq 
		\nu( F \geq \epsilon ) 
		+ 
		\sup_{A \in \mathcal{F}} | \rho_k(A) - \nu(A) | 
		\\ 
	&\leq 
		\frac{\delta}{2} 
		+ 
		\sqrt{ 
			\frac{1}{2} 
			H_\nu( \rho_k )
		} \,, 
\end{aligned}
\end{equation}
therefore it is sufficient to establish the bound 
$H_\nu( \rho_k ) \leq \frac{\delta^2}{2}$, 
which reduces to the sufficient condition 
\begin{equation}
	C_0 \, e^{-\alpha k \eta} 
	\leq \frac{ \delta^2 }{4}
	\quad \text{ and } \quad 
	C_1 \frac{\eta}{\alpha} 
	\leq \frac{\delta^2}{4} \,. 
\end{equation}

To satisfy the second condition, 
we simply choose $\eta \leq \frac{ \alpha \delta^2 }{ 4 C_1 }$, 
and the first condition is equivalent to 
\begin{equation}
	k \geq \frac{1}{\alpha \eta } 
		\left[ 2 \log \frac{2}{\delta} + \log C_0 
		\right] \,. 
\end{equation}

Now we can use the condition of $\eta$ 
and plug in $C_1 = 45 nd K_2^2$ 
and the condition on $\beta$ to write 
\begin{equation}
\begin{aligned}
	\eta 
	&\leq 
		\min\left\{ 
			\frac{2}{3 \alpha} \, , \, 
			\frac{ \alpha \delta^2 \sqrt{\epsilon} }{ 
			180 nd K_2^2 
			\sqrt{ \log \frac{2nK_2}{\epsilon \delta} } } 
		\right\} 
	\leq 
		\min\left\{ \frac{2}{3 \alpha} \,, 
			\frac{ \alpha \delta^2 }{ 4 C_1 } \,, 
			\frac{\alpha}{ 24 K_2 \sqrt{(\beta+d)d} } 
		\right\} 
		\,, 
\end{aligned}
\end{equation}
and finally we plug in $C_0 = H_\nu(\rho_0)$ 
to get the desired sufficient condition 
\begin{equation}
	k 
	\geq 
		\max\left\{ 
		\frac{3}{2} \,, 
		180 nd \frac{K_2^2}{\alpha^2} 
		\frac{ \sqrt{ \log \frac{2nK_2}{\epsilon\delta} } }{ \delta^2 \sqrt{\epsilon} }
		\right\} 
		\left( 2 \log \frac{2}{\delta} + \log H_\nu(\rho_0) \right) \,. 
\end{equation} 

\end{proof}

\subsection{Runtime Complexity for General Problems}
\label{subsec:runtime_general_problem}

\begin{corollary}
Let $F:M \to \mathbb{R}$ satisfy 
\cref{asm:app_c3,asm:app_unique_min,asm:app_morse}. 
Let $\{X_k\}_{k\geq 1}$ be the Langevin algorithm 
defined in \cref{eq:langevin_algorithm}, 
with initialization $\rho_0 \in C^1(M)$, 
and $d \geq 3$. 
For all choices of 
$\epsilon \in (0,1]$ and $\delta \in (0,1)$, 
if $\beta$ and $\eta$ satisfy the conditions in 
\cref{cor:runtime_complexity_lsi} 
and \cref{thm:lyapunov_log_sobolev}, 
then choosing $k$ as
\begin{equation}
	k \geq 
	\begin{cases}
		\wt\Omega \left( 
			\frac{n^{9.5} d^{8}}{\epsilon^{2.5} \delta^2} 
		\right)
		\,, 
		& \text{ if the global minimum is unique, } \\
		\wt\Omega \left( 
			\frac{n^{15.5} d^{14}}{\epsilon^{4.5} \delta^2} 
		\right) 
		\,, 
		& \text{ otherwise, }
	\end{cases}
\end{equation}
where $\wt{\Omega}(\cdot)$ hides dependence on 
$\poly\left( K_2, K_3, C_F^{-1}, \lambda_*^{-1}, 
K_*, 
\log \frac{ nd \, K_2 }{ \epsilon \, \delta } \,, 
\log H_\nu(\rho_0) \, 
\right)$ 
and $K_* = \exp \left( \frac{ -2 C_F^2}{ K_2 K_3 } \right) $, 
we have that the Langevin algorithm $\{X_k\}_{k\geq 1}$ 
defined in \cref{eq:langevin_algorithm} 
with distribution $\rho_k := \mathcal{L}(X_k)$ satisfies 
\begin{equation}
	\rho_k\left( F - \min_{y \in M} F(y) \geq \epsilon \right) 
	\leq \delta \,. 
\end{equation}
In other words, $X_k$ finds an $\epsilon$-global minimum 
with probability $1 - \delta$. 

\end{corollary}

\begin{proof}

For this proof, since we only need to track the polynomial dependence on 
$n,d,\alpha,\beta,\delta,\epsilon$ and ignoring $\log$ factors, 
we will work with $\wt{O}(\cdot)$ and $\wt\Omega(\cdot)$ notation 
to denote the polynomial dependence on the desired parameters. 

We start by following the steps of \cref{cor:runtime_complexity_lsi} to get 
\begin{equation}
\begin{aligned}
	k 
	&\geq 
		\wt \Omega \left( \frac{1}{\alpha \eta} \right) 
		\,, \quad 
	\eta 
	\leq 
		\wt O \left( \min\left\{ \frac{\alpha\delta}{nd} \,, 
		\frac{\alpha}{ \sqrt{(\beta+d)d} } \right\} 
		\right) \,. 
\end{aligned}
\end{equation}

We simplify this slightly further to get 
\begin{equation}
	k 
	\geq 
		\wt \Omega\left( 
		\frac{1}{\alpha^2 \delta^2} \max\left\{ nd \,, 
		\sqrt{(\beta+d)d} 
		\right\} 
		\right) \,. 
\end{equation}

At this point, we can plug in 
$\alpha$ from \cref{thm:lyapunov_log_sobolev} to get 
\begin{equation}
	\frac{1}{\alpha^2} 
	\leq 
		\wt O \left( n^2 \beta^{2r} \right) 
		\,, 
\end{equation}
where $r = 1$ when there is a unique minimum, 
and $r = 2$ otherwise. 
This further gives us the expression 
\begin{equation}
	k 
	\geq 
		\wt \Omega \left( \frac{n^2 \beta^{2r} }{\delta^2} 
		\max\left\{ nd \,, \sqrt{(\beta+d)d} 
		\right\} 
		\right)
		\,. 
\end{equation}

Recall also from \cref{thm:gibbs_high_prob_bound,thm:lyapunov_log_sobolev},  
we need 
\begin{equation}
	\beta 
	\geq 
		\wt \Omega \left( \frac{(nd)^3}{ \epsilon} \right) 
			\,.  
\end{equation}

Therefore we can get the final runtime complexity as 
\begin{equation}
	k \geq 
		\wt\Omega \left( \frac{n^2}{\delta^2} 
			\frac{ (nd)^{6r} }{ \epsilon^{2r} } 
			\sqrt{ \frac{(nd)^3}{\epsilon} d } 
		\right) 
		= 
		\wt\Omega \left( \frac{ n^{6r + 3.5} d^{6r+2} 
			}{\delta^2 \epsilon^{2r + 0.5}} 
		\right) 
		\,. 
\end{equation}

\end{proof}

\subsection{Runtime for SDP and Max-Cut}
\label{subsec:sdp_maxcut_lsi}

We will restate the result from \cref{cor:sdp_maxcut_no_asm}. 

\begin{corollary}
Let $F$ be the Burer--Monteiro loss function defined in \cref{eq:bm_loss_function}. 
Then for all choices of $d$ such that $(d+1)(d+2)>2n$ 
and almost every cost matrix $A$, 
$F$ satisfies \cref{asm:app_c3,asm:app_morse,asm:app_unique_min}. 

Furthermore, if we choose $d = \left\lceil \sqrt{2n} \right\rceil$, 
then for all $\epsilon \in (0,1]$ and $\delta \in (0,1)$, 
$\beta$ and $\eta$ satisfying the conditions in 
\cref{cor:runtime_complexity_lsi} and \cref{thm:lyapunov_log_sobolev}, 
and choosing $k$ as 
\begin{equation}
	k \geq 
		\wt\Omega \left( 
			\frac{ n^{22.5} }{\epsilon^{4.5} \delta^2} 
		\right) 
		\,, 
\end{equation}
where $\wt{\Omega}(\cdot)$ hides dependence on 
$\poly\left( K_2, K_3, C_F^{-1}, \lambda_*^{-1}, 
K_*, 
\log \frac{ n \, K_2 }{ \epsilon \, \delta } \,, 
\log H_\nu(\rho_0) \, 
\right)$ and $K_* = \exp \left( \frac{ -2 C_F^2}{ K_2 K_3 } \right) $, 
we have that with probability $1 - \delta$, 
the Langevin algorithm $\{X_k\}_{k \geq 1}$ 
defined in \cref{eq:langevin_algorithm} finds 
an $\epsilon$-global solution of the SDP \cref{eq:sdp} 
after $k$ iterations 
for almost every cost matrix $A$. 

Additionally, if we let $\epsilon' := \epsilon / (4 \MaxCut(A_G))$, 
then using $X_k$ and the random rounding scheme of \citep{goemans1995improved}, 
we recover an $0.878 (1-\epsilon')$-optimal Max-Cut 
for almost every adjacency matrix $A_G$. 
\end{corollary}

\begin{proof}

We start by observing $F\in C^3(M)$ trivially, 
which verifies \cref{asm:app_c3}. 

Before we move to the next assumptions, 
we will make the standard observation that $F$ is symmetric 
up to an $O(d+1)$ orbit. 
Or more precisely, for all $x \in M$, 
we call $xO(d+1) = \{ xQ | Q \in O(d+1) \}$ the orbit of $x$. 
And we can write for all $Q \in O(d+1)$ 
\begin{equation}
	F(x) = \langle x , Ax \rangle 
	= \langle A, x x^\top \rangle
	= \langle A, x Q Q^\top x^\top \rangle 
	= \langle A, (Qx) (Qx)^\top \rangle 
	= F(xQ) \,. 
\end{equation}

Then a standard gradient calculation (see for example \cite{mei2017solving}) gives us 
\begin{equation}
	\grad F(x) = 2 ( A - \ddiag(Axx^\top) ) x \,, 
\end{equation}
where $\ddiag:\mathbb{R}^{n\times n} \to \mathbb{R}^{n\times n}$ 
that zeros out all the non-diagonal entries. 

Observe that $\grad F(x) = 0$ requires $A - \ddiag(Axx^\top)$ to be low rank 
and that all columns of $x$ lie in the null space of $A - \ddiag(Axx^\top)$. 
However, any small perturbation of $A$ will lead to 
the matrix $A - \ddiag(Axx^\top)$ becoming full rank almost surely. 
Therefore for almost every matrix $A$, 
all critical points of $F$ are isolated up to an $O(d+1)$ orbit, 
and therefore the total number of these orbits must be finite. 
This means that \cref{asm:app_morse} is trivially satisfied by 
taking the minimum of absolute Hessian spectral gap 
\begin{equation}
	\lambda_{*} 
	:= \inf \left\{ | \lambda_i (\nabla^2 F(x)) | 
		: \grad F(x) = 0, i \in [nd] , \lambda_i (\nabla^2 F(x)) \neq 0 
	\right\} \,. 
\end{equation}

Finally, since strict complementarity is satisfied for almost every $A$ 
\cite[Lemma 2]{alizadeh1997complementarity}, 
this guarantees that the SDP admits a unique solution, 
and therefore the Burer--Monteiro relaxation can also recover that solution. 
In other words, the global minimum is unique up to an $O(d+1)$ orbit, 
hence satisfying \cref{asm:app_unique_min}. 

Finally, by taking $d = \left\lceil \sqrt{2n} \right\rceil$, 
we can apply the runtime complexity result of the Langevin algorithm directly 
with \cref{cor:runtime_general_problem}, 
and taking $d = \Omega(\sqrt{n})$, we recover the desired runtime complexity of 
\begin{equation}
	k \geq \wt \Omega \left( \frac{ n^{22.5} }{\epsilon^{4.5} \delta^2} 
	\right) \,. 
\end{equation}

\end{proof}

\section*{Acknowledgement}
\addcontentsline{toc}{section}{Acknowledgement}

ML would like to specifically thank Tian Xia 
for many insightful discussions on Riemannian geometry 
and stochastic analysis on manifolds, 
Iosif Lytras for a very helpful discusson on 
carefully handling boundary conditions for the Lyapunov method 
(specifically \cref{prop:poincare_lyapunov_generic}), 
and the anonymous reviewers for helping us significantly 
improve the manuscript. 
We also thank 
Blair Bilodeau, 
Philippe Casgrain, 
Christopher Kennedy, 
Yuri Kinoshita, 
Justin Ko, 
Yasaman Mahdaviyeh, 
Jeffrey Negrea, 
Ali Ramezani-Kebrya, 
Taiji Suzuki, 
and Daniel Roy 
for numerous helpful discussions and draft feedback. 
MAE is partially funded by NSERC Grant [2019-06167], 
Connaught New Researcher Award, 
CIFAR AI Chairs program, 
and CIFAR AI Catalyst grant. 
ML is supported by the Ontario Graduate Scholarship 
and the Vector Institute.

\appendix 

\addcontentsline{toc}{section}{References}
\bibliographystyle{amsalpha}
\bibliography{constrained.bib}

\section{Sampling a Spherical Brownian Motion Increment}
\label{sec:app_brownian_inc}

In this section, we will describe the algorithm for exact sampling of 
Brownian motion increments on the sphere $S^d$, 
where $d \geq 2$. 
We start by assuming we can sample from the Wright--Fisher distribution 
$\WF_{0,t}( \frac{d}{2}, \frac{d}{2} )$ 
using an algorithm by \cite{jenkins2017exact}, 
which we state later. 
Then we have an algorithm from \cite{mijatovic2018note} 
that samples the Brownian motion increment exactly 
in \cref{algo:brownian_increment_mijatovic}. 

\begin{algorithm}[h]
\SetAlgoLined
\SetKwProg{Require}{Require}{}{}
\Require{ Starting point $z \in S^d$ and time horizon $t > 0$ }{
	Simulate the radial component: 
		$X \sim \WF_{0,t}( \frac{d}{2}, \frac{d}{2} )$\;
	Simulate the angular component: 
		$Y$ uniform on $S^{d-1}$\;
	Set $e_d := (0, \cdots, 0, 1)^\top \in S^d$, 
		$u := (e_d - z) / | e_d - z |$, 
		and $O(z) := I - 2uu^\top$\; 
	\textbf{return} $O(z) \left( 2 \sqrt{ X (1-X) } Y^\top, 1 - 2X \right)^\top$ 
}
 \caption{\cite[Algorithm 1]{mijatovic2018note} Simulation of Brownian motion on $S^d$}
 \label{algo:brownian_increment_mijatovic}
\end{algorithm}

Next we will describe the sampling procedure for $\WF_{x,t}(\theta_1,\theta_2)$. 

\begin{algorithm}[h]
\SetAlgoLined
\SetKwProg{Require}{Require}{}{}
\Require{ Mutation parameters $\theta_1,\theta_2$, 
	starting point $x \in [0,1]$, and time horizon $t > 0$ 
	}{
	Simulate 
		$M \overset{d}{=} A_\infty^\theta(t)$ 
		using \cite[Algorithm 2]{jenkins2017exact}, 
		where $\theta = \theta_1 + \theta_2$\; 
	Simulate the angular component: 
		$L \sim \text{Binomial}(M,x)$\; 
	Simulate $Y \sim \text{Beta}(\theta_1 + L, \theta_2 + M - L)$\; 
	\textbf{return} $Y$ 
}
 \caption{\cite[Algorithm 2]{mijatovic2018note} 
 	Simulation of $\WF_{x,t}(\theta_1,\theta_2)$}
 \label{algo:wf_mijatovic}
\end{algorithm}

Finally, to generate a sample of $A_\infty^\theta(t)$, 
we will use the following algorithm. 

\begin{algorithm}[h]
\SetAlgoLined
\SetKwProg{Require}{Require}{}{}
\SetKwRepeat{Repeat}{repeat}{until}
\Require{ Mutation parameter $\theta = \theta_1 + \theta_2$ 
	and time horizon $t > 0$ 
	}{
	Set $m \leftarrow 0, k_0 \leftarrow 0, \mathbf{k} \leftarrow (k_0)$\; 
	Simulate $U \sim \text{Uniform}[0,1]$\; 
	\Repeat{
		false
	}{
		Set $k_m \leftarrow C^{t,\theta}_m / 2$\; 
		\While{ $S_{\mathbf{k}}^-(m) < U < S_{\mathbf{k}}^+(m)$ }{
			Set $\mathbf{k} \leftarrow \mathbf{k} + (1,\cdots,1)$  
		}
		 
		\uIf{ $S_{\mathbf{k}}^-(m) > U$ }{
			\textbf{return} $m$ 
		}
		\uElseIf{ $S_{\mathbf{k}}^+(m) < U$ }{
			Set $\mathbf{k} \leftarrow (k_0, k_1, \cdots, k_m, 0)$  
			Set $m \leftarrow m + 1$  
		}
	}
}
 \caption{\cite[Algorithm 2]{jenkins2017exact} 
 	Simulation of $A_\infty^\theta(t)$}
 \label{algo:wf_jenkins}
\end{algorithm}

For the above algorithm, we define 
\begin{equation}
\begin{aligned}
	a^\theta_{km} 
	&:= 
		\frac{(\theta + 2k - 1)}{m! (k-m)!} 
		\frac{ \Gamma( \theta+m+k-1 ) }{ \Gamma(\theta+m) } 
		\,, \\ 
	b^{(t,\theta)}_k(m) 
	&:= 
		a^\theta_{km} 
		e^{ -k (k+\theta-1) t / 2 } 
		\,, \\ 
	C^{t,\theta}_m 
	&:= 
		\inf \{ i \geq 0 \, | \, 
			b^{(t,\theta)}_{i + m + 1}(m)
			< b^{(t,\theta)}_{i + m}(m)
		\} \,, \\ 
	S_{\mathbf{k}}^-(m) 
	&:= 
		\sum_{m=0}^M \sum_{i=0}^{2k_m+1} 
			(-1)^i b^{(t,\theta)}_{m+i}(m) 
			\,, \\ 
	S_{\mathbf{k}}^+(m) 
	&:= 
		\sum_{m=0}^M \sum_{i=0}^{2k_m} 
			(-1)^i b^{(t,\theta)}_{m+i}(m) 
			\,, 
\end{aligned}
\end{equation}
where $\Gamma$ is the Gamma function. 

Here we note that while the repeat loop 
in \cref{algo:wf_jenkins} 
does not have an explicit terminate condition, 
\cite[Proposition 1]{jenkins2017exact} guarantees 
the algorithm will terminate in finite time.

\section{Background on Bakry--\'{E}mery Theory and Lyapunov Methods}
\label{sec:app_lsi_conc}

\subsection{Classical Bakry--\'{E}mery Theory}
\label{subsec:app_bakry_emery}

In this section, we recall several well known results 
from \cite{bakry2013analysis}, 
and adapt them slightly to our setting. 
First we recall our Markov triple $(M, \nu, \Gamma)$ 
with $M = S^d \times \cdots \times S^d$ an $n$-times product, 
$\nu(dx) = \frac{1}{Z} e^{-\beta F(x)} dx$ 
is the Gibbs measure, 
and the carr\'{e} du champ operator $\Gamma$ 
on $C^2(M) \times C^2(M)$ as 
\begin{equation}
	\Gamma(f,h) := \frac{1}{2} ( L(fh) - fLh - hLf )\,, 
\end{equation}
for $L f := \langle - \grad F, \grad f \rangle_g 
+ \frac{1}{\beta} \Delta f$ 
is the infinitesimal generator. 

In particular, we notice that 
\begin{equation}
	\Gamma(f,h) = \frac{1}{\beta} \langle \grad f, \grad h \rangle_g \,. 
\end{equation}

We will next define the second order 
carr\'{e} du champ operator as 
\begin{equation}
	\Gamma_2(f,h) = 
	\frac{1}{2} ( L \Gamma(f,h) - \Gamma(f, Lh) - \Gamma(Lf, h) ) \,. 
\end{equation}

\begin{lemma}

We have the explicit formula for the second order 
carr\'{e} du champ operator as 
\begin{equation}
	\Gamma_2(f,f) = 
			\frac{1}{\beta^2} |\nabla^2 f|_g^2 
			+ \frac{1}{\beta^2} \Ric_g(\grad f, \grad f)
			+ \frac{1}{\beta} \nabla^2 F( \grad f, \grad f ) 
			\,. 
\end{equation}

\end{lemma}

\begin{proof}

This calculation will follow the ones of 
\cite[Appendix C]{bakry2013analysis} closely, 
but we will add a temperature factor $\beta$, 
and make steps explicit. 

We start by stating the Bochner-Lichnerowicz formula 
\citep[Theorem C.3.3]{bakry2013analysis} 
\begin{equation}
	\frac{1}{2} \Delta ( |\grad f|_g^2 ) 
	= \langle \grad f, \grad \Delta f \rangle_g 
	+ |\nabla^2 f|_g^2 
	+ \Ric_g( \grad f, \grad f ) \,. 
\end{equation}

Then we will directly compute 
\begin{equation}
\begin{aligned}
	\Gamma_2(f,f) 
	&= \frac{1}{2} ( L \Gamma(f,f) - 2 \Gamma(f, Lf) ) \\
	&= \frac{1}{2} \left( 
		\frac{1}{\beta} L |\grad f|_g^2 
		- \frac{2}{\beta} \langle \grad f, \grad L f \rangle_g 
		\right) \,, 
\end{aligned}
\end{equation}
where we then compute one term at a time 
\begin{equation}
\begin{aligned}
	L |\grad f|_g^2 
	&= \langle -\grad F, 2 \nabla^2 f( \grad f )^\sharp \rangle_g 
		+ \frac{1}{\beta} \Delta ( |\grad f|_g^2 ) 
		\\ 
	&= \langle -\grad F, 2 \nabla^2 f( \grad f )^\sharp \rangle_g 
		+ \frac{2}{\beta} 
			( \langle \grad f, \grad \Delta f \rangle_g
			+ |\nabla^2 f|_g^2 
			+ \Ric_g( \grad f, \grad f ) ) \,, \\ 
	\langle \grad f, \grad L f \rangle_g
	&= \left\langle 
		\grad f, 
		- \nabla^2 F(\grad f)^\sharp 
		- \nabla^2 f(\grad F)^\sharp 
		+ \frac{1}{\beta} \grad \Delta f 
		\right\rangle_g \,. 
\end{aligned}
\end{equation}

Putting these terms together, 
we have the desired result 
\begin{equation}
\begin{aligned}
	\Gamma_2(f,f) 
	&= \frac{1}{\beta} 
		\left[ 
			\frac{1}{\beta} |\nabla^2 f|_g^2 
			+ \frac{1}{\beta} \Ric_g(\grad f, \grad f)
			+ \nabla^2 F( \grad f, \grad f ) 
		\right] \,. 
\end{aligned}
\end{equation}

\end{proof}

\begin{definition}
[Curvature Dimension Condition]
\label{defn:cd_infty}
For $\kappa \in \mathbb{R}$, 
we say the Markov triple $(M, \nu, \Gamma)$ 
satisfies the condition $CD(\kappa, \infty)$ 
if for all $f \in C^2(M)$ 
\begin{equation}
	\Gamma_2(f,f) \geq \kappa \, \Gamma(f,f) \,. 
\end{equation}
\end{definition}

We note in this case, 
$CD(\kappa, \infty)$ is equivalent to 
\begin{equation}
	\nabla^2 F + \frac{1}{\beta} \Ric_g 
	\geq \kappa \, g \,. 
\end{equation}

We will additionally state a couple of useful standard results 
before compute the logarithmic Sobolev inequality constant 
for the Gibbs distribution. 

\begin{proposition}
\citep[Proposition 5.7.1]{bakry2013analysis}
\label{prop:cd_log_sobolev}
Under the curvature dimension condition $CD(\kappa, \infty)$ 
for $\kappa > 0$, 
the Markov triple $(E,\mu,\Gamma)$ satisfies 
$\LSI(\kappa)$. 
\end{proposition}

We will also need the following perturbation result 
originally due to \cite{holley1987logarithmic}. 

\begin{proposition}
\citep[Proposition 5.1.6]{bakry2013analysis} 
\label{prop:holley_strook_perturb}
Assume that the Markov triple $(E,\mu,\Gamma)$ 
satisfies $\LSI(\alpha)$. 
Let $\mu_1$ be a probability measure with density $h$ 
with respect to $\mu$ such that 
\begin{equation}
	\frac{1}{b} \leq h \leq b \,, 
\end{equation}
for some constant $b>0$. 
Then $\mu_1$ satisfies a logarithmic Sobolev inequality 
with constant $\alpha / b^2$. 
\end{proposition}

Here we remark that an equivalent statement is that if 
$\nu = \frac{1}{Z} e^{-\beta F}$ 
and $\wt{\nu} = \frac{1}{Z} e^{-\beta \wt{F}}$, 
where $(M, \nu, \Gamma)$ satisfies $\LSI(\alpha)$, 
then $(M, \wt{\nu}, \Gamma)$ satisfies 
$\LSI( \alpha e^{ - \Osc ( F - \wt{F} ) } )$, 
where we define 
\begin{equation}
	\Osc h := \sup h - \inf h \,. 
\end{equation}

\subsection{Local in Time Results}

In this subsection, we will recall and adapt 
several known results about the transition density 
of a Brownian motion on $S^d$. 
Equivalently, we would be studying the semigroup 
$\{P_t\}_{t\geq 0}$ defined by 
\begin{equation}
	P_t \phi(x) := \mathbb{E} [ \phi(X_t) | X_0 = x ] \,, 
\end{equation}
where $\{X_t\}_{t \geq 0}$ is the Brownian motion on $S^d$ 
with diffusion coefficient $\sqrt{2/\beta}$. 

We start by stating an important radial comparison theorem 
on the manifold. 

\begin{theorem}\citep[Theorem 3.5.3]{hsu2002stochastic}
\label{thm:radial_comparison}
let $\{X_t\}_{t\geq 0}$ be a standard Brownian motion 
on a general Riemannian manifold $(M,g)$, 
and define the radial process 
\begin{equation}
	r_t := d_g(X_t, X_0) \,, 
\end{equation}
then the radial satisfies 
\begin{equation}
	r_t = \beta_t + \frac{1}{2} 
		\int_0^t \Delta_M r(X_s) \, ds - L_t \,, 
	\quad t < e(X) \,, 
\end{equation}
where $\beta_t$ is a standard Brownian motion in $\mathbb{R}$, 
$L_t$ is a local time process supported on the cut locus of $X_0$, 
and $e(X)$ is the exit time of the process. 

Let $\kappa:[0,\infty) \to \mathbb{R}$ be such that 
\begin{equation}
	\kappa(r) \geq \max \{ K_M(x) : d_g(x,X_0) = r \} \,, 
\end{equation}
where $K_M(x)$ is the maximum sectional curvature at $x$. 
We then define $G(r)$ as the unique solution to 
\begin{equation}
	G''(r) + \kappa(r) G(r) = 0 \,, 
	\quad G(0) = 0 \,, G'(0) = 1 \,, 
\end{equation}
and $\rho_t$ as the unique nonnegative solution to 
\begin{equation}
	\rho_t = \beta_t + \frac{d-1}{2} 
		\int_0^t \frac{G'(\rho_s)}{G(\rho_s)} \, ds \,. 
\end{equation}

Then we have that $e(\rho) \geq e(X)$ 
and $\rho_t \leq r_t$ for all $t < T_{C_0}$, 
where $T_{C_0}$ is the first hitting time 
of the cut locus. 

\end{theorem}

Using the previous result, 
we can now upper bound the radial process of Brownian motion 
on $S^d$ by the radial process on $\mathbb{R}^d$. 

\begin{corollary}
\label{cor:conc_mean_comparison}
Let $r_t$ be the radial process of a Brownian motion on $\mathbb{R}^d$, 
and $\rho_t$ be the radial process of a Brownian motion on $S^d$, 
then we have that $r_t \geq \rho_t$ 
for all $t < T_{C_0}$ i.e. before hitting the cut locus. 
Consequently for all $0 < \delta < d_g(X_0, C_0)$, 
i.e. less than the distance to the cut locus, 
we have the following bounds 
\begin{equation}
	\mathbb{P}[ \, \rho_t \geq \delta ]
	\leq \mathbb{P}[ \, r_t \geq \delta ] \,, 
	\quad 
	\mathbb{E} \, \rho_t \leq \mathbb{E} \, r_t \,. 
\end{equation}
\end{corollary}

\begin{proof}

Firstly, for $\kappa(r)$ a constant function, 
we can show that the unique solution of ODE for $G$ 
is a generalize sine function 
\begin{equation}
	G(r) = \begin{cases}
		\frac{1}{\sqrt{\kappa}} \sin( \sqrt{\kappa} r )
		\,, & \kappa > 0 \,, \\
		r \,, & \kappa = 0 \,, \\
		\frac{1}{\sqrt{-\kappa}} \sinh(\sqrt{-\kappa} r) 
		\,, & \kappa < 0 \,. 
	\end{cases}
\end{equation}

As a straight forward consequence of 
the Laplacian comparison theorem 
\citep[Theorem 3.4.2]{hsu2002stochastic}
for constant curvature manifolds, 
or many other sources such as 
\cite[p. 269]{ito1974diffusion}, 
we can show that the radial process of 
a Brownian motion on $S^d$ satisfies 
\begin{equation}
	\rho_t = \beta_t + \frac{d-1}{2} 
		\int_0^t \cot( \rho_s ) \, ds \,, 
\end{equation}
and the radial process of an $\mathbb{R}^d$ 
is a Bessel process 
\begin{equation}
	r_t = \beta_t + \frac{d-1}{2} 
		\int_0^t \frac{1}{r_s} \, ds \,. 
\end{equation}

Using the comparison from \cref{thm:radial_comparison}, 
and the fact that on $S^d$ we have $\kappa = 1$, 
we immediately have the desired comparison 
\begin{equation}
	\rho_t \leq r_t \,, \quad t < T_{C_0} \,. 
\end{equation}

The tail and expectation bound 
follows immediately from the comparison. 

\end{proof}

Before we can continue, 
we will cite a local logarithmic Sobolev inequality result. 

\begin{theorem}
\citep[Theorem 5.5.2, Local LSI]{bakry2013analysis}
\label{thm:local_log_sobolev}
Let $(E, \mu, \Gamma)$ be a Markov Triple with semigroup $P_t$. 
The following are equivalent. 
\begin{enumerate}
	\item The curvature dimension condition 
		$CD(\kappa, \infty)$ holds for some $\kappa \in \mathbb{R}$. 
	\item For every $f \in H^1(\mu)$ 
		and every $t \geq 0$, we have that
		\begin{equation}
			P_t ( f \log f ) - (P_t f) \log ( P_t f )
			\leq 
				\frac{ 1 - e^{ - 2 \kappa t } }{ 2 \kappa }
				P_t\left( \frac{ \Gamma(f) }{ f }
					\right) \,, 
		\end{equation}
\end{enumerate}
where when $\kappa = 0$, we will replace 
$\frac{ 1 - e^{ - 2 \kappa t } }{ 2 \kappa }$ with $t$. 

In other words, the Markov triple $(M,p_t,\Gamma)$, 
where $p_t$ is the transition kernel density of $P_t$, 
satisfies $\LSI( \frac{ 2\kappa }{ 1 - e^{ - 2 \kappa t } } )$. 
\end{theorem}

\begin{corollary}
\label{cor:heat_local_log_sobolev}
For a Brownian motion on $S^d$ 
or $M = S^d \times \cdots \times S^d$ ($n$-times) 
with diffusion coefficient $\sqrt{2/\beta}$, 
its transition density $p_t$ satisfies 
$\LSI(\frac{1}{2t})$. 
\end{corollary}

\begin{proof}

We will simply verify the conditions of \cref{thm:local_log_sobolev}. 
For $S^d$, our Markov triple is 
$(S^d, \mu, \Gamma)$ where $\mu$ is uniform on the sphere, 
and $\Gamma(f) = \frac{1}{\beta} |\grad f|_g^2$. 
Hence when our potential $F$ is constant, 
and we can check 
\begin{equation}
	\nabla^2 F + \frac{1}{\beta} \Ric_g 
	= \frac{1}{\beta} \Ric_g 
	= \frac{d-1}{\beta} g \,,
\end{equation}
therefore $(S^d, \mu, \Gamma)$ satisfies 
the curvature dimension condition $CD( (d-1) / \beta, \infty )$. 

On $M = S^d \times \cdots S^d$, 
we observe that $\Ric_g$ remains unaffected, 
and therefore we still satisfy 
the curvature dimension condition $CD((d-1) / \beta, \infty)$. 

Using \cref{thm:local_log_sobolev}, 
we also observe that 
\begin{equation}
	\frac{ 2\kappa }{ 1 - e^{ - 2 \kappa t } } 
	\leq 
		\frac{1}{2t} \,, 
\end{equation}
which is the desired result. 

\end{proof}

We will additionally need a local Poincar\'{e} inequality, 
which we will define as follows. 

\begin{definition*}
\label{defn:poincare}
We say that a Markov triple $(E,\mu,\Gamma)$ satisfies 
a \textbf{Poincar\'{e} inequality} with constant $\kappa>0$, 
denoted $\PI(\kappa)$ if for all $f \in L^2(\mu) \cap C^1(M)$, 
we have that 
\begin{equation}
\label{eq:defn_poincare}
	\mu( f^2 ) - \mu(f)^2 
	\leq 
		\frac{1}{\kappa} \, \mu( \Gamma(f) ) \,, 
\end{equation}
where $\mu(f):= \int_E f \, d\mu$. 
\end{definition*}

\begin{theorem}
\citep[Theorem 4.7.2, Local PI]{bakry2013analysis}
\label{thm:local_poincare}
Let $(E, \mu, \Gamma)$ be a Markov Triple with semigroup $P_t$. 
The following are equivalent. 
\begin{enumerate}
	\item The curvature dimension condition 
		$CD(\kappa, \infty)$ holds for some $\kappa \in \mathbb{R}$. 
	\item For every $f \in H^1(\mu)$ 
		and every $t \geq 0$, we have that
		\begin{equation}
			P_t ( f^2 ) - (P_t f)^2 
			\leq 
				\frac{ 1 - e^{ - 2 \kappa t } }{ \kappa }
				P_t\left( \Gamma(f) 
					\right) \,, 
		\end{equation}
\end{enumerate}
where when $\kappa = 0$, we will replace 
$\frac{ 1 - e^{ - 2 \kappa t } }{ \kappa }$ with $2t$. 

In other words, the Markov triple $(M,p_t,\Gamma)$, 
where $p_t$ is the transition kernel density of $P_t$, 
satisfies $\PI( \frac{ 2\kappa }{ 1 - e^{ - \kappa t } } )$. 
\end{theorem}

\subsection{Bakry-\'{E}mery with Boundary Conditions}

In this subsection, we consider the case when $U$ 
is a convex manifold with boundary $\de U$, 
and we consider the Langevin diffusion with 
a reflecting boundary condition in the same setting as 
\citep{qian1997gradient}. 
More specifically, for the generator 
\begin{equation}
	L \phi = \langle - \grad F, \grad \phi \rangle_g 
		+ \frac{1}{\beta} \Delta \phi \,, 
\end{equation}
$X_t$ is the unique diffusion such that 
\begin{equation}
	f(X_t) - f(X_0) - \int_0^t \, Lf(X_s) \, ds \,, 
\end{equation}
is a martingale for all $f \in C^2_0(U)$ with 
$\frac{\de f}{\de n} = 0$, where $n$ is the outward normal. 

In this case, we can define the semigroup $P_t$ as follows
\begin{equation}
	P_t f(x) = \mathbb{E} [ f(X_t) | X_0 = x ] \,. 
\end{equation}

We will also define the curvature potential as in 
\citep{qian1997gradient}
\begin{equation}
	\rho^B(x) := \inf \left\{ 
		( \Ric_g - \nabla^s_B )(\xi, \xi) : \xi \in T_x U, |\xi|_g = 1
		\right\} \,, 
\end{equation}
where 
\begin{equation}
	\nabla^s_B(\xi, \eta) 
	= 
	\frac{1}{2} \langle \nabla_\xi B, \eta \rangle 
	+ \frac{1}{2} \langle \nabla_\eta B, \xi \rangle \,, 
\end{equation}
and the vector field $B$ is $- \grad F$ in our case. 
Note $\rho^B(x) \geq \kappa$ is equivalent to $CD(\kappa,\infty)$ 
in this setting. 

Now we will state the main result of \citep{qian1997gradient}. 

\begin{theorem}\citep[Theorem 2.1]{qian1997gradient}
\label{thm:qian_grad_est}
Let $U$ be a compact Riemannain manifold with convex boundary. 
Then for any $f \in C^2(U) \cap C^1(U \cup \de U)$, 
we have the following gradient estimate 
\begin{equation}
	|\grad P_t f|_g(x) 
	\leq 
	\mathbb{E}\left[ |\grad f|_g(X_t) \, e^{ -\int_0^t \rho^B(X_s) \, ds }
	\right] \,. 
\end{equation}
\end{theorem}

We observe that under $CD(\kappa, \infty)$, we have that 
\begin{equation}
	|\grad P_t f|_g(x) 
	\leq e^{-\kappa t} P_t |\grad f|_g \,. 
\end{equation}

We will use this to establish a local Poincar\'{e} inequality 
using essentially the same proof as in 
\cite[Theorem 4.7.2]{bakry2013analysis}, 
but for a manifold with boundary. 
Here we will define the Gibbs measure on $U$ as 
$\nu|_U(dx) = \frac{1}{Z_U} e^{-\beta F} dx$, 
where $Z_U = \int_U e^{-\beta F} dx$ is the normalizing constant. 

\begin{proposition}[Local Poincar\'{e} Inequality with Boundary]
\label{prop:local_pi_boundary}
Let $U$ be with a Riemannain manifold with convex boundary, 
and let $(U, \nu|_U, \Gamma|_U)$ satisfy $CD(\kappa, \infty)$. 
Then for all $f \in C^\infty_0(U)$ and $t \geq 0$, 
we have the following local inequality 
\begin{equation}
	P_t (f^2) - ( P_t f )^2
	\leq 
	\frac{1 - e^{-2\kappa t}}{\kappa} 
	P_t \left( \Gamma(f) \right) \,. 
\end{equation}
Furthermore, we also have $(U, \nu|_U, \Gamma|_U) \in \PI(\kappa)$. 
\end{proposition}

\begin{proof}

We start by weakening the gradient estimate of 
\cref{thm:qian_grad_est} 
\begin{equation}
\begin{aligned}
	\Gamma( P_t f ) 
	&= \frac{1}{\beta} | \grad P_t f |_g^2 \\ 
	&\leq \frac{1}{\beta} e^{ -2\kappa t } P_t |\grad f|^2 \\ 
	&\leq e^{-2\kappa t} P_t \Gamma(f) \,, 
\end{aligned}
\end{equation}
where we used the Cauchy-Schwarz inequality in the last step. 

Next we will define $\Lambda(s) := P_s ( P_{t-s} f )^2$ 
and compute 
\begin{equation}
\begin{aligned}
	\Lambda'(s) 
	&= 
		P_s( L( P_{t-s} f )^2 - 2 P_{t-s} f L P_{t-s} f ) 
		\\ 
	&= 
		P_s\left( \frac{2}{\beta} |\grad f|_g^2 \right) 
		\\ 
	&= 
		2 P_s \Gamma( P_{t-s} f ) \,. 
\end{aligned}
\end{equation}

To recover the local inequality, 
we simply need to observe the following integral relation 
\begin{equation}
\begin{aligned}
	P_t (f^2) - (P_t f)^2 
	&= \Lambda(t) - \Lambda(0) \\ 
	&= \int_0^t \Lambda'(s) \, ds \\
	&= \int_0^t 2 P_s \Gamma( P_{t-s} f ) \, ds \\ 
	&\leq 
		\int_0^t 2 P_s e^{-2\kappa(t-s)} P_{t-s} \Gamma(f) \, ds 
		\\ 
	&= 
		\int_0^t 2 e^{-2\kappa(t-s)} \, ds \, 
			P_t \Gamma(f) \\ 
	&= 
		\frac{1 - e^{-2\kappa t}}{\kappa} P_t \Gamma(f) \,. 
\end{aligned}
\end{equation}

Finally, to recover the global Poincar\'{e} inequality, 
we can take $t\to\infty$ and recover 
\begin{equation}
	P_\infty (f^2) - (P_\infty f)^2 
	\leq 
		\frac{1}{\kappa} P_\infty \Gamma(f) \,, 
\end{equation}
which is equivalent to $\PI(\kappa)$ 
if the stationary measure of $\nu|_U$ is unique. 
The uniqueness follows from the fact that $U$ is convex bounded 
and therefore $(U,\nu|_U, \Gamma|_U)$ satisfies 
$\PI( \frac{\pi^2}{4D^2} )$ 
where $D = \text{Diam}(U)$ \citep{li1980estimates}.

\end{proof}

\subsection{Adapting Existing Lyapunov Results}
\label{subsec:app_lyapunov}

We will first adapt a result from \cite{bakry2008simple} 
to include the temperature parameter $\beta$, 
and show the manifold setting introduces no additional complications. 

\begin{theorem}
\citep[Theorem 1.4 Adapted]{bakry2008simple}
\citep[Theorem 3.8 Adapted]{menz2014poincare}
\label{thm:lyapunov_poincare_bakry}
Let $\nu = \frac{1}{Z} e^{- \beta F}$ be 
a probability measure on a Riemannian manifold $(M,g)$. 
Let $U \subset M$ be such that $(U, \nu|_U, \Gamma|_U)$ 
satisfies the Poincar\'{e} inequality 
with constant $\kappa_U > 0$. 
Suppose there exist constants $\theta > 0, b \geq 0$ 
and a Lyapunov function $W \in C^2(M)$ and $W \geq 1$ such that 
\begin{equation}
\label{eq:lyapunov_poincare_cond}
	L W \leq - \theta W + b \, \mathds{1}_{U}(x) \,, 
\end{equation}
where $L \phi = \langle - \grad F, \grad \phi \rangle_g 
+ \frac{1}{\beta} \Delta \phi$ 
is the generator. 
Further suppose either $M$ does not have a boundary, 
or $M$ has a $C^1$ boundary $\de M$ and $W$ satisfies 
the Neumann boundary condition 
\begin{equation}
	- \int \phi \, LW d\nu 
	= 
		\frac{1}{\beta} \int 
		\left\langle \grad \phi , 
			\grad W
		\right\rangle_g \, d\nu \,, 
\end{equation}
for all $\phi \in C^1(M)$. 
Then $(M, \nu, \Gamma)$ satisfies $\PI(\kappa)$ with 
\begin{equation}
	\kappa = \frac{\theta}{1 + b / \kappa_U} \,. 
\end{equation}
\end{theorem}

\begin{proof}

We will follow carefully the proof steps of 
\cite{bakry2008simple} with our conventions. 
We start by observing that for all $h \in L^2(\nu)$ 
and constants $c \in \mathbb{R}$ to be chosen later, 
we have that 
\begin{equation}
	\Var_\nu(h) := \nu(h^2) - \nu(h)^2 \leq \nu( (h - c)^2 ) \,. 
\end{equation}

Next we can rearrange the Lyapunov condition 
in \cref{eq:lyapunov_poincare_cond} to write 
\begin{equation}
	\nu(f^2) 
	\leq 
		\int \frac{-LW}{\theta W} f^2 d\nu 
		+ \int f^2 \frac{b}{\theta W} \mathds{1}_{U} d\nu \,. 
\end{equation}

Using the fact that $L$ is $\nu$-symmetric 
and integration by parts with the given boundary condition, 
we can write 
\begin{equation}
\begin{aligned}
	\int \frac{-LW}{W} f^2 d\nu 
	&= 
		\frac{1}{\beta} \int 
		\left\langle \grad \left( \frac{f^2}{W} \right), 
			\grad W
		\right\rangle_g \, d\nu 
		\\ 
	&= 
		\frac{2}{\beta} \int (f / W) \langle \grad f, \grad W \rangle_g 
		d\nu 
		- \frac{1}{\beta} \int (f^2 / W^2) | \grad W |_g^2 \, d\nu \\
	&= 
		\frac{1}{\beta} \int |\grad f|^2_g \, d\nu 
		- \frac{1}{\beta} 
			\int | \grad f - (f / W) \grad W |_g^2 \, d\nu \\
	&\leq 
		\frac{1}{\beta} \int |\grad f|^2_g \, d\nu \,. 
\end{aligned}
\end{equation}

Now we can study the second term. 
Since $\nu$ satisfies a Poincar\'{e}'s inequality 
on $U$ with constant $\kappa_U$, 
we have that 
\begin{equation}
	\int_{U} f^2 \, d\nu 
	\leq \frac{1}{\kappa_U \, \beta} 
		\int_{U} |\grad f|_g^2 \, d\nu 
		+ \frac{1}{\mu(U)} 
			\left( \int_{U} f \, d\nu
			\right)^2 \,. 
\end{equation}

Now we choose $c = \int_{U} h \, d\nu$, 
the so last term is equal to zero. 
Then using $W \geq 1$, we can get that 
\begin{equation}
	\int_{U} f^2 / W \, d\nu 
	\leq \int_{U} f^2 \, d\nu 
	\leq \frac{1}{\kappa_U \, \beta} \int_{U} |\grad f|_g^2 \, d\nu \,. 
\end{equation}

Putting the two results together, 
we get the desired Poincar\'{e} inequality as follows 
\begin{equation}
	\Var_\nu(h) \leq \int f^2 \, d\nu 
	\leq \frac{1 + b / \kappa_U}{\theta} 
		\frac{1}{\beta} \int |\grad h|_g^2 \,. 
\end{equation}

\end{proof}

Next we will adapt an alternative Lyapunov--Poincar\'e result 
from \cite{cattiaux2013poincare}, 
and add some more precise control over the Poincar\'e constant. 
To start we will need to construct a smooth partition function. 

\begin{lemma}[Smooth Partition Function]
\label{lm:smooth_partition}
For all $r_2 > r_1 > 0$ and $\epsilon > 0$, 
there exists a smooth non-increasing function 
$\psi : \mathbb{R} \to [0,1]$ such that 
\begin{equation}
	\psi(x) 
	= 
	\begin{cases}
		0 \,, & x \leq r_1 \,, \\ 
		\text{Increasing} \,, & x \in (r_1, r_2) \,, \\ 
		1 \,, & x \geq r_2 \,, 
	\end{cases}
\end{equation}
and that $\| \psi^\prime \|_\infty \leq \frac{1}{r_2 - r_1} + \epsilon$. 
\end{lemma}

\begin{proof}

We start by defining the unit smooth bump function 
\begin{equation}
	\phi_b(x) 
	= 
	\begin{cases}
		\frac{ \exp\left( \frac{-1}{b^2 - x^2} \right) }{ 
				\int_{-b}^b \exp\left( \frac{-1}{b^2 - y^2} \right) \, dy 
				} 
				\,, & x \in [-b,b] \,, \\ 
		0 \,, & \text{ otherwise.}
	\end{cases}
\end{equation}

Observe that as $b\to 0$, $\phi_b$ converges to 
the Dirac delta functional with unit mass. 
This allows us to construct the derivative of $\psi$ as follows 
\begin{equation}
\begin{aligned}
	\psi^\prime(x) 
	&= 
		\left( \frac{1}{r_2 - r_1} + \epsilon \right) 
		\int_{-\infty}^x \phi_b(y - r_1 - b) - \phi_b(y - r_2 + b) \, dy 
		\\ 
	&= 
	\begin{cases}
		0 \,, & x \leq r_1 \text{ or } x \geq r_2 \,, \\ 
		\frac{1}{r_2 - r_1} + \epsilon \,, & x \in (r_1 + 2b, r_2 - 2b) \,, 
	\end{cases}
\end{aligned}
\end{equation}
where without loss of generality we can assume $b>0$ 
is sufficiently small for $\psi^\prime$ to be well defined. 

Then we can construct $\psi$ as 
\begin{equation}
	\psi(x) = \int_{-\infty}^x \psi^\prime(y) \, dy \,. 
\end{equation}

Finally, to satisfy the requirement that $\psi$ 
choose $b>0$ sufficiently small such that $\psi(r_2) = 0$, 
or more precisely 
\begin{equation}
	\int_{r_1}^{r_2} \psi^\prime(x) \, dx 
	= 
	\left( \frac{1}{r_2 - r_1} + \epsilon \right) 
	\int_{r_1}^{r_2} 
		\int_{-\infty}^x \phi_b(y - r_1 - b) - \phi_b(y - r_2 + b) \, dy  \, dx 
	= 1 \,. 
\end{equation}

Since this is always possible for $\epsilon$ small, 
therefore it's sufficient for the upper bound. 

\end{proof}

We will now establish a Poincar\'e inequality based on this construction. 

\begin{proposition}[Adapted Theorem 2.3 of \cite{cattiaux2013poincare}]
\label{prop:poincare_adapted}
For all $r > \wt{r} > 0$, let us define 
the following open neighbourhoods of saddle points 
\begin{equation}
	B = \left\{ x \in M \, | \,  d_g(x, \mathcal{S}) < r \right\} \,, \quad 
	\wt B = \left\{ x \in M \, | \, d_g(x, \mathcal{S}) < \wt r \right\} \,. 
\end{equation}
Suppose $(\wt{B}^c, \nu|_{\wt{B}^c}, \Gamma|_{\wt{B}^c})$ 
satisfies a Poincar\'e inequality with constant $\kappa_{\wt{B}^c}$, 
and there exists $W \in C^2(B)$ such that $W \geq 1$ and 
\begin{equation}
	L W \leq - \theta W \,, \quad x \in B \,. 
\end{equation}
Then we have that $(M, \nu, \Gamma)$ satisfies a Poincar\'e inequality 
with constant 
\begin{equation}
	\frac{1}{\kappa} = \frac{4}{\theta} + 
		\left( \frac{4}{\theta \beta (r - \wt{r})^2 } + 2 \right) 
		\frac{1}{\kappa_{\wt{B}^c}} \,. 
\end{equation}
\end{proposition}

\begin{proof}

We will follow the steps of (H1) $\Rightarrow$ (H4) in \cite{cattiaux2013poincare} 
and first compute for any smooth $f$ and use integration-by-parts to write 
\begin{equation}
\begin{aligned}
	\int \frac{-LW}{W} f^2 \, d\nu 
	&= \int \Gamma\left( \frac{f^2}{W} , W \right) \, d\nu 
		\\ 
	&= 2 \int \frac{f}{W} \Gamma(f, W) \, d\nu 
		- \int \frac{f^2}{W^2} \Gamma(W, W) \, d\nu 
		\\ 
	&= - \int \left| \frac{f}{W \sqrt{\beta}} \grad W 
		- \frac{1}{\sqrt{\beta}} \grad f
		\right|^2 \, d\nu 
		+ 
		\int \Gamma(f,f) \, d\nu 
		\\ 
	&\leq \int \Gamma(f,f) \, d\nu \,. 
\end{aligned}
\end{equation}

Next we will introduce a partition function $\chi:M \to [0,1]$ 
such that $\chi = 1$ on $B^c$ and $\chi = 0$ on $\wt B$. 
This function can be explicitly constructed by choosing $\psi$ 
from \cref{lm:smooth_partition} with $r_1 = \wt{r}, r_2 = r$, and any $\epsilon>0$, 
then we can define 
\begin{equation}
	\chi(x) = \psi\left( d_g(x, \mathcal{S}) \right) \,, 
\end{equation}
where we recall $B, \wt{B}$ are neighbourhoods of $\mathcal{S}$. 
Here we note since $\|\psi^\prime\|_\infty \leq \frac{1}{r - \wt{r}} + \epsilon$, 
we also have that 
\begin{equation}
	\| \Gamma(\chi, \chi) \|_\infty 
	= \frac{1}{\beta} \| \psi^\prime \|_\infty^2 
	\leq \frac{1}{\beta} \left( \frac{1}{r - \wt{r}} + \epsilon \right)^2 \,. 
\end{equation}

This allows us to carry on the calculation with 
\begin{equation}
\begin{aligned}
	\int f^2 \, d\nu 
	&= \int (f(1-\chi) + f\chi)^2 \, d\nu \\ 
	&\leq 2 \int f^2 (1-\chi)^2 \, d\nu + 2 \int f^2 \chi^2 \, d\nu \\ 
	&\leq \frac{2}{\theta} \int \frac{-LW}{W} f^2 (1-\chi)^2 \, d\nu 
		+ 2 \int_{\wt B^c} f^2 \, d\nu \\ 
	&\leq \frac{2}{\theta} \int \Gamma( f(1-\chi), f(1-\chi) ) \, d\nu 
		+ 2 \int_{\wt B^c} f^2 \, d\nu \,. 
\end{aligned}
\end{equation}

Next we can use $\Gamma(fg, fg) \leq 2( f^2 \Gamma(g, g) + g^2 \Gamma(f,f) )$ 
to write 
\begin{equation}
\begin{aligned}
	\int f^2 \, d\nu 
	&\leq \frac{4}{\theta} \int \Gamma(f,f) \, d\nu 
		+ \frac{4}{\theta} \int f^2 \Gamma(\chi,\chi) \, d\nu 
		+ 2 \int_{\wt B^c} f^2 \, d\nu 
		\\ 
	&\leq \frac{4}{\theta} \int \Gamma(f,f) \, d\nu 
		+ \left( \frac{4 \|\Gamma(\chi, \chi)\|_\infty }{ \theta } + 2 \right) 
		\int_{\wt B^c} f^2 \, d\nu \,. 
\end{aligned}
\end{equation}

Since $(\wt B^c, \nu|_{\wt B^c}, \Gamma|_{\wt B^c})$ 
satisfies a Poincar\'e inequality, 
we have that for $\wt f = f - \int_{\wt B^c} f \, d\nu$ 
\begin{equation}
	\int_{\wt B^c} \wt f^2 \, d \nu 
	\leq 
	\frac{1}{ \kappa_{\wt B^c} } \int_{\wt B^c} \Gamma(f, f) \, d\nu \,, 
\end{equation}
which implies 
\begin{equation}
	\text{Var}_\nu(f) \leq \int \wt f^2 \, d\nu 
	\leq 
	\left( \frac{4}{\theta} 
	+ \left( \frac{4 \| \Gamma(\chi, \chi) \|_\infty }{ \theta } + 2 
	\right) \kappa_{ \wt B^c } 
	\right) 
	\int \Gamma(f, f) \, d\nu \,. 
\end{equation}

Finally, we explicitly control the constant 
\begin{equation}
	\frac{4}{\theta} 
	+ \left( \frac{4 \| \Gamma(\chi, \chi) \|_\infty }{ \theta } + 2
	\right) \kappa_{ \wt B^c } 
	\leq 
	\frac{4}{\theta} 
	+ \left( \frac{4}{ \theta \beta } 
	\left( \frac{1}{r - \wt r} + \epsilon
	\right)^2 + 2 
	\right) \frac{1}{ \kappa_{\wt B^c} } \,, 
\end{equation}
and since $\psi$ can be constructed for any $\epsilon > 0$, 
we have the Poincar\'e constant  
\begin{equation}
	\frac{1}{\kappa} 
	= \frac{4}{\theta} 
	+ \left( \frac{4}{ \theta \beta (r - \wt r)^2 } + 2 
	\right) \frac{1}{ \kappa_{\wt B^c} } \,. 
\end{equation}

\end{proof}

\section{Technical Lemmas on Local Coordinates}
\label{sec:app_local_coord}

\subsection{Stereographic Coordinates}
\label{subsec:app_stereo}

In this section, we will attempt to prove 
a technical result regarding a divergence term 
via explicit calculations in local coordinates. 
We begin by introducing the stereographic coordinates. 
Let us view $S^d \subset \mathbb{R}^{d+1}$ 
using coordinates $(y_0, y_1, \ldots, y_d)$, 
and define a coordinate on $S^d \setminus (1, 0, \ldots, 0)$ by 
\begin{equation}
	x_i := \frac{y_i}{1 - y_0} \,, \quad \forall i = 1, \ldots, d \,. 
\end{equation}

For this coordinate system, we get from 
\cite[Proposition 3.5]{lee2019riemann} 
that the Riemannian metric has the following form 
\begin{equation}
	g_{ij}(x) = \frac{ 4 \delta_{ij} }{ (|x|^2 + 1)^2 } \,, 
\end{equation}
where $\delta_{ij}$ is the Kronecker delta. 
Since the metric is diagonal, we immediately also get that 
$g^{ij} = \delta_{ij} / g_{ij}$. 
This allows us to compute the Christoffel symbols as 
\begin{equation}
\begin{aligned}
	\Gamma^k_{ij} 
	&= \frac{1}{2} g^{kh} 
		\left( 
			\de_j g_{ih} + \de_i g_{hj} - \de_h g_{ij} 
		\right) \\
	&= \frac{1}{2} 
		\frac{(|x|^2 + 1)^2}{4} \delta_{kh} 
		\left( 
			\frac{-16 x_j \delta_{ih} }{( |x|^2 + 1)^3} 
			+ \frac{-16 x_i \delta_{hj} }{( |x|^2 + 1)^3} 
			- \frac{-16 x_h \delta_{ij} }{( |x|^2 + 1)^3} 
		\right) \\ 
	&= \frac{-2}{ |x|^2 + 1 } 
		\left( 
			x_j \delta_{ik} + x_i \delta_{kj} 
			- x_k \delta_{ij} 
		\right) \,. 
\end{aligned}
\end{equation}

We will first need a couple of technical calculations on 
the stereographic coordinates. 

\begin{lemma}
\label{lm:stereo_geodesic}
Let $u = (u^1, \ldots, u^d) \in \mathbb{R}^d$ 
be a tangent vector at $0$ 
such that $|u| = 1$. 
Then in the stereographic coordinates 
for $S^d$ defined above, 
the unique unit speed geodesic connecting 
$0$ and $y = c u$ for constant 
$c \in \mathbb{R}$ has the form 
\begin{equation}
	\gamma(t) = \tan(t/2) \, u \,.
\end{equation}

\end{lemma}

\begin{proof}

it is sufficient to check that $\gamma(t)$ is 
the unique positive solution to the ODE 
\begin{equation}
	\ddot{\gamma}^k(t) 
	+ \dot{\gamma}^i(t) \dot{\gamma}^j(t) 
		\Gamma^k_{ij}(\gamma(t)) = 0 \,, 
	\quad \forall k = 1, \ldots, d \,,
\end{equation}
with initial conditions $\gamma(0) = 0, |\dot{\gamma}(0)|_g = 1$. 

Firstly, since $S^d$ is radially symmetric, 
the geodesic must be radial curve of the form 
\begin{equation}
	\gamma(t) = f(t) \, u \,. 
\end{equation}

Then the ODE reduces down to 
\begin{equation}
	\ddot{f}(t) u^k + \dot{f}(t)^2 u^i u^j \Gamma^k_{ij}(\gamma(t)) = 0 \,. 
\end{equation}

Next we compute the Christoffel symbol part as 
\begin{equation}
\begin{aligned}
	u^i u^j \Gamma^k_{ij} (\gamma(t)) 
	&= \frac{-2}{ |\gamma(t)|^2 + 1 }
		\sum_{i=1}^d 2 u^i u^k \gamma_i(t) - (u^i)^2 \gamma_k(t) \\
	&= \frac{-2}{ f(t)^2 |u|^2 + 1 } 
		\sum_{i=1}^d 2 u^i u^k u^i f(t) - (u^i)^2 u^k f(t) \\
	&= \frac{ - 2 u^k f(t) }{ f(t)^2 + 1 } \,, 
\end{aligned}
\end{equation}
where we used the fact that $|u|^2 = 1$. 

Returning to the ODE, we now have that 
\begin{equation}
	\ddot{f}(t) u^k + 
	\dot{f}(t)^2 \frac{ - 2 u^k f(t) }{ f(t)^2 + 1 } = 0 \,, 
\end{equation}
which further reduces to 
\begin{equation}
	\ddot{f}(t) - 
	\frac{ 2 f(t) \dot{f}(t)^2 }{ f(t)^2 + 1 } = 0 \,. 
\end{equation}

Using Wolfram Cloud \citep[Wolfram Cloud]{WolframCloud} 
with command 
\begin{lstlisting}[language=Mathematica]
	DSolve[f''[t] - 2 * f[t] *(f'[t])^2/(f[t]^2 + 1)==0,f[t],t]
\end{lstlisting}
the general solution to this equation has the form 
\begin{equation}
	f(t) = \tan( c_1 t + c_2) \,, 
\end{equation}
where using initial condition $\gamma(0) = 0$ we have that $c_2 = 0$. 

Finally, to match the other initial condition 
\begin{equation}
	|\dot{\gamma}(0)|_g^2 
	= \dot{f}(0)^2 u^i u^j g_{ij}(0) 
	= c_1^2 |u|^2 4 
	= 1 \,, 
\end{equation}
therefore we have $c_1 = 1/2$ 
and $\gamma(t) = \tan(t/2) \,u$ as desired. 

\end{proof}

Now we make a computation regarding parallel transport 
along geodesics of the above form. 

\begin{lemma}
\label{lm:div_loc_coord}
Let $v = (v^1, \ldots, v^d) \in \mathbb{R}^d$
be a tangent vector at $0$ using 
stereographic coordinates defined above. 
Then for any point $x \in \mathbb{R}^d$, 
the vector parallel transport of $v$ 
along the geodesic to $x$ has the form 
\begin{equation}
	P_{0,x} v = v (|x|^2 + 1) \,. 
\end{equation}

Furthermore, the vector field defined by $A(x) = P_{0,x} v$ 
has the following divergence 
\begin{equation}
	\div A = 4 \sum_{i=1}^d x_i v^i  - 2 \sum_{i,j=1}^d x_j v^i \,. 
\end{equation}

\end{lemma}

\begin{proof}

We will solve the ODE for parallel transport explicitly 
for a geodesic $\gamma(t)$ 
and the vector field $A(t) := A(\gamma(t))$
\begin{equation}
	\dot{A}^k(t) 
	+ \dot{\gamma}^i(t) A^j(t) \Gamma^k_{ij}(\gamma(t))
	= 0 \,, 
	\quad \forall k = 1,\ldots,d \,, 
\end{equation}
with initial condition $A^k(0) = v^k$. 

We start by first simplifying the term 
\begin{equation}
\begin{aligned}
	&\dot{\gamma}^i(t) A^j(t) \Gamma^k_{ij}(\gamma(t)) \\
	&= \frac{-2}{ |\gamma(t)|^2 + 1 } 
		\sum_{i=1}^d 
		\gamma_i(t) ( \dot{\gamma}^i(t) A^k(t) 
		+ \dot{\gamma}^k A^i(t) )
		- \gamma_k(t) \dot{\gamma}^i(t) A^i(t) \\
	&= \frac{-2}{ \tan(t/2)^2 + 1 } 
		\sum_{i=1}^d 
			\tan(t/2) \, u^i
			\left( \frac{1}{2} \sec(t/2)^2 u^i A^k(t) 
				+ \frac{1}{2} \sec(t/2)^2 u^k A^i(t) 
			\right)
			- \tan(1/4) \, u^k \frac{1}{2} 
				\sec(t/2)^2 u^i A^i(t) \\
	&= \frac{ - \tan(t/2) \sec(t/2)^2 }{ \tan(t/2)^2 + 1  }
		\sum_{i=1}^d u^i u^i A^k(t) + u^i u^k A^i(t) - u^k u^i A^i(t) \\
	&= \frac{ - \tan(t/2) \sec(t/2)^2 }{ \tan(t/2)^2 + 1 } A^k(t) \\ 
	&= -\tan(t/2) A^k(t) 
	\,, 
\end{aligned}
\end{equation}
where we canceled the last two terms, 
and used the fact that $|u|^2 = 1$ 
and $\tan(t/2)^2 + 1 = \sec(t/2)^2$. 

This leads to the following ODE for each $k = 1, \ldots, d$
\begin{equation}
	\dot{A}^k(t) - 
	\tan(t/2) A^k(t)  
	= 0 \,, 
\end{equation}
with the initial condition $A^k(0) = v^k$. 

Once again, using Wolfram Cloud \citep{WolframCloud} 
with command 
\begin{lstlisting}[language=Mathematica]
	DSolve[A'[t]-A[t]*Tan[t/2]==0, A[t], t]
\end{lstlisting}
we get that the general solution is 
\begin{equation}
	A^k(t) = c_1 \sec(t/2)^2 \,, 
\end{equation}
and using the initial condition we get $c_1 = v_k$. 

To get the desired form, 
we will let $u = x/|x|$ and $\tan(t/2) = |x|$, 
then we get 
\begin{equation}
	A(x) 
	= v \sec(t/2)^2 
	= v \sec( \arctan(|x|) ) 
	= v (|x|^2 + 1) \,. 
\end{equation}

Finally, we can compute the divergence in local coordinates 
\begin{equation}
\begin{aligned}
	\div A 
	&= \frac{1}{\sqrt{G}} \de_i ( A^i \sqrt{G} ) \\
	&= \de_i A^i + A^i \de_i \log \sqrt{G} \\
	&= \de_i A^i + A^i \Gamma^j_{ij} \\
	&= v^i (2 x_i) + v^i (|x|^2 + 1) \frac{2}{|x|^2 + 1}
		\left( x_i - \sum_{j=1}^d x_j
		\right) \\
	&= 4 \sum_{i=1}^d x_i v^i 
		- 2 \sum_{i,j=1}^d x_j v^i \,, 
\end{aligned}
\end{equation}
where $G := \det(g_{ij})$ 
and we used the identity $\de_i \log \sqrt{G} = \Gamma^j_{ij}$. 

\end{proof}

Here we recall the notation $x = (x^{(1)}, \cdots, x^{(n)} ) \in M$ 
where $x^{(\ell)}$ represents the coordinate in $\ell^\text{th}$ 
sphere $S^d$, 
and similarly write $v = (v^{(1)}, \cdots, v^{(n)} ) \in T_x M$ 
with $v^{(\ell)} \in T_{x^{(\ell)}} S^d$. 
We will also use $g'$ to denote the metric on a single sphere. 

Recall that $\gamma_t(x) := \exp(x, - t \grad F(x))$ 
and 
\begin{equation}
	b(t, x_0, x) := P_{\gamma_t(x_0), x} 
		P_{x_0, \gamma_t(x_0)} \grad F(x_0) \,, 
\end{equation}
where $P_{x,y}:T_x M \to T_y M$ is the parallel transport along 
the unique shortest geodesic connecting $x,y$ when it exists, 
and zero otherwise. 
We first prove an identity to reduce the complexity of the problem. 

\begin{lemma}
[Rotational Symmetry Identity]
\label{lm:stereo_div_symmetry}
Let $\{X_t\}_{t\geq 0}$ be the continuous time representation of 
the Langevin algorithm defined in \cref{eq:cont_time_rep}. 
Then we have the following identities 
\begin{equation}
\begin{aligned}
	\mathbb{E} \left[ \left. 
			\div_{X_t} 
			b(t, X_0, X_t) \right| X_0 = x_0 \right] 
	&= 
		0 \,, \\ 
	\mathbb{E} \left[ \left. \left( \div_{X_t} 
			b(t, X_0, X_t) 
		\right)^2 \right| X_0 = x_0 \right]
	&= 
		| \grad F(x_0) |_{g}^2 \, 
		\left( \frac{2}{d} + 1 \right) 
		\mathbb{E} \tan\left( 
			\frac{1}{2} d_{g'}( \gamma_t(x_0)^{(i)}, X_t^{(i)} ) 
			\right)^2 \,, 
\end{aligned}
\end{equation}
where $i \in [n]$ is arbitrary. 

\end{lemma}

\begin{proof}

We first recall the result of \cref{lm:div_loc_coord}, 
and let $v = P_{x_0, \gamma_t(x_0)} \grad F(x_0), 
y = \gamma_t(x_0), x = X_t$, 
then we have by natural extension of $\div$ 
to product stereographic coordinates 
centered at $y^{(i)}$ 
\begin{equation}
	\div_x b(t,x_0,x) 
	= 
		\sum_{i=1}^{n} 
		\left[ 4 \sum_{j=1}^d x_j^{(i)} v^{(i), j} 
			- 2 \sum_{j=1}^d x_j^{(i)} \sum_{j=1}^d v^{(i), j} 
		\right] \,. 
\end{equation}

Observe that since each $x^{(i)}$ is 
an independent Brownian motion on $S^d$ starting at $y^{(i)}$, 
we have that each $x_j^{(i)}$ is distributed symmetrically around zero, 
and hence have zero mean. 
Therefore we have that 
\begin{equation}
\begin{aligned}
	\mathbb{E} \left[ \left. 
			\div_{X_t} 
			b(t, X_0, X_t) \right| X_0 = x_0 \right] 
	&= 
		0 \,, \\ 
	\mathbb{E} \left[ \left. \left( \div_{X_t} 
			b(t, X_0, X_t) 
		\right)^2 \right| X_0 = x_0 \right] 
	&= 
		\sum_{i=1}^n 
		\mathbb{E} \left[ 4 \sum_{j=1}^d (X_t)_j^{(i)} v^{(i), j} 
			- 2 \sum_{j=1}^d (X_t)_j^{(i)} \sum_{j=1}^d v^{(i), j} 
		\right]^2 \,. 
\end{aligned}
\end{equation}

At this same time, since Brownian motion 
is radially symmetric about its starting point, 
we can without loss of generality 
let $v^{(i)} = (|v^{(i)}|, 0, \cdots, 0) \in \mathbb{R}^d$. 
This allows us to rewrite the previous expression as  
\begin{equation}
\begin{aligned}
	&\sum_{i=1}^n 
		\mathbb{E} \left[ 4 \sum_{j=1}^d (X_t)_j^{(i)} v^{(i), j} 
			- 2 \sum_{j=1}^d (X_t)_j^{(i)} \sum_{j=1}^d v^{(i), j} 
		\right]^2 
		\\ 
	=& 
		\sum_{i=1}^n \, 
		| v^{(i)} |^2 \, 
		\mathbb{E} \left[ 4 (X_t)^{(i)}_1
			- 2 \sum_{j=1}^d (X_t)_j^{(i)} 
		\right]^2 \\ 
	=& 
		\left| v \right|^2 \, 
		\mathbb{E} \left[ 4 (X_t)^{(i)}_1
			- 2 \sum_{j=1}^d (X_t)_j^{(i)} 
		\right]^2 
		\\ 
	=& \left| \grad F(x_0) \right|_g^2 \, 
		\frac{1}{4} \, 
		\mathbb{E} \left[ 4 (X_t)^{(i)}_1
			- 2 \sum_{j=1}^d (X_t)_j^{(i)} 
		\right]^2 
		\,, 
\end{aligned}
\end{equation}
where we used the fact that the Brownian motion components 
on each $S^d$ are independent and identically distributed, 
and that on stereographic coordinates 
$g_{ij}(0) = 4 \delta_{ij}$. 

This implies it is sufficient to analyze the result 
on a single sphere $S^d$.  
In fact, from this point onward, 
we drop the superscript $(i)$ 
for coordinates as it is no longer required. 
Here we let $X_t$ have density $p_t(x)$ 
in stereographic coordinates, 
and observe that $p_t$ is only a function of $|x|$. 
Therefore the radial process $|X_t|$ 
is independent of the normalized coordinates  
$Y_j := (X_t)_j / |X_t|$, 
where $( Y_j )_{j=1}^d$ 
is uniformly distributed on $S^{d-1}$. 
This allows us to write 
\begin{equation}
\begin{aligned}
	\mathbb{E} \left[ 4 (X_t)_1
			- 2 \sum_{j=1}^d (X_t)_j 
		\right]^2 
	=& 
		\, \mathbb{E} \left[ |X_t|^2 \right] \, 
		\mathbb{E} \left[ 4 Y_1 
				- 2 \sum_{j=1}^d Y_j 
			\right]^2 
			\\ 
	=& 
		\, \mathbb{E} \tan\left( \frac{1}{2} 
			d_{g'}(\gamma_t(X_0)^{(i)}, X_t^{(i)}) 
			\right)^2 \, 
		\mathbb{E} \left[ 4 Y_1 
				- 2 \sum_{j=1}^d Y_j 
			\right]^2 \,, 
\end{aligned}
\end{equation}
where we used the form of unit speed geodesic from 
\cref{lm:stereo_geodesic} to get 
$|x| = \tan\left( \frac{1}{2} d_{g'}(0, x) \right)$. 
Therefore it is sufficient to only analyze 
the expectation with the normalize coordinates $(Y_j)_{j=1}^d$. 

First we claim that $\mathbb{E} Y_j Y_k = 0$ whenever $j\neq k$. 
Since the coordinate ordering is arbitrary, 
it is sufficient to prove for $j=1, k=2$. 
We start by writing out the expectation over spherical coordinates, 
where $Y_1 = \cos \varphi_1, Y_2 = \sin \varphi_1 \cos \varphi_2$, 
the integrand is over $[0,\pi] \times [0,\pi]$ and 
over the volume form 
$\sin^{d-2} \varphi_1 \sin^{d-3} \varphi_2 d \varphi_1 d\varphi_2$, 
which gives us 
\begin{equation}
	\mathbb{E} Y_1 Y_2 
	= 
		\frac{ \int_0^\pi \int_0^\pi \cos \varphi_1 \sin \varphi_1 \cos \varphi_2 
			\sin^{d-2} \varphi_1 \sin^{d-3} \varphi_2 d \varphi_1 d\varphi_2 
			}{ 
			\int_0^\pi \int_0^\pi 
				\sin^{d-2} \varphi_1 \sin^{d-3} \varphi_2 d \varphi_1 d\varphi_2 
			} \,. 
\end{equation}

Here we observe that $\sin$ is symmetric about $\pi/2$ 
and $\cos$ is antisymmetric about $\pi/2$, 
therefore 
\begin{equation}
	\int_0^\pi \cos \varphi_1 \sin^{d-2} \varphi_1 d\varphi_1 = 0 \,, 
\end{equation}
which proves the desired claim. 

This allows us to simplify the expectation 
by expanding the bracket and removing terms of the type 
$Y_j Y_k$ where $j\neq k$ 
\begin{equation}
	\mathbb{E} \left[ 4 Y_1 
				- 2 \sum_{j=1}^d Y_j 
			\right]^2 
	= 
		\mathbb{E} \, 8 Y_1^2 
				+ 4 \sum_{j=1}^d Y_j^2 
	= 
		\mathbb{E} \, 8 Y_1^2 + 4 \,, 
\end{equation}
where we used the fact that $\sum_{j=1}^d Y_j^2 = 1$. 

Finally, since the coordinate ordering is arbitrary 
therefore each $Y_j$ is identically distributed, 
and that $\mathbb{E} \sum_{j=1}^d Y_j^2 = 1$, 
we must also have 
\begin{equation}
	\mathbb{E} Y_j^2 = \frac{1}{d} \,, 
		\quad \text{ for all } j \in [d] \,. 
\end{equation}

Putting everything together, we have that 
\begin{equation}
\begin{aligned}
	& \mathbb{E} \left[ \left. \left( \div_{X_t} 
			b(t, X_0, X_t) 
		\right)^2 \right| X_0 = x_0 \right] 
		\\ 
	=& 
		\left| \grad F(x_0) \right|_g^2 \, 
		\frac{1}{4} \, 
		\mathbb{E} \left[ 4 (X_t)^{(i)}_1
			- 2 \sum_{j=1}^d (X_t)_j^{(i)} 
		\right]^2 
		\\ 
	=& 
		\left| \grad F(x_0) \right|_g^2 \, 
		\frac{1}{4} \, 
		\mathbb{E} \tan\left( \frac{1}{2} 
			d_{g'}(\gamma_t(X_0)^{(i)}, X_t^{(i)}) 
			\right)^2 \, 
		\mathbb{E} \left[ 4 Y_1 
				- 2 \sum_{j=1}^d Y_j 
			\right]^2 
			\\ 
	=& 
		\left| \grad F(x_0) \right|_g^2 \, 
		\left( \frac{2}{d} + 1 \right) \, 
		\mathbb{E} \tan\left( \frac{1}{2} 
			d_{g'}(\gamma_t(X_0)^{(i)}, X_t^{(i)}) 
			\right)^2 \,, 
\end{aligned}
\end{equation}
which is the desired result. 

\end{proof}

\subsection{Riemannian Normal Coordinates on $S^d$}
\label{subsec:app_normal_sd}

In this subsection, we will compute the metric 
and several results related to the normal coordinates on a sphere. 
More specifically, we first take the stereographic coordinates 
$(x_1, \cdots, x_d)$ on $S^d \ \{x^*\}$, 
and transform it to the new coordinates 
\begin{equation}
	y := \frac{ 2x }{ |x| } \arctan( |x| ) \,, 
\end{equation}
where $y$ now lives on a ball $B_\pi(0) \subset \mathbb{R}^d$. 
This also gives us the following inverse map 
\begin{equation}
	x := \frac{y}{|y|} \tan( |y|/2 ). 
\end{equation}

$y = (y_1, \cdots, y_d)$ is called the normal coordinates 
because all geodesics $\gamma(t)$ starting at 
$\gamma(0) = 0$ has the form 
\begin{equation}
	\gamma(t) = t(v^1, \cdots, v^d) \,, 
	\quad v \in \mathbb{R}^d \,. 
\end{equation}

\begin{lemma}
\label{lm:normal_coord_metric}
The Riemannian metric in the normal coordinates is 
\begin{equation}
	g_{ij}(y) 
	= \frac{ y_i y_j }{ |y|^2 } 
		\left[ 1 
			-  \frac{\sin^2( |y| ) }{|y|^2}
		\right]
	+ \delta_{ij} \frac{\sin^2( |y| ) }{|y|^2} \,, 
\end{equation}
where $\delta_{ij}$ denotes the Kronecker delta. 

\end{lemma}

\begin{proof}

it is sufficient to compute the quantity 
\begin{equation}
	\frac{ 4 \sum_{i=1}^d (dx^i)^2 }{ ( |x|^2 + 1 )^2 } 
\end{equation}
in terms of the normal coordinates $y$. 
We start by using $x = \frac{y}{|y|} \tan( |y|/2 )$ to write 
\begin{equation}
	dx^i 
	= \frac{ dy^i }{ |y| } \tan( |y|/2 ) 
		- \frac{ y_i }{ 2 |y|^3 } \sum_{j=1}^d 2y_j dy^j 
			\tan( |y|/2 ) 
		+ \frac{ y_i }{ |y| } \sec^2( |y|/2 ) \frac{1}{4 |y|}
			\sum_{j=1}^d 2y_j dy^j \,, 
\end{equation}
which then gives us 
\begin{equation}
\begin{aligned}
	( dx^i )^2 
	&= \frac{ (dy^i)^2 }{ |y|^2 } \tan^2( |y|/2 ) 
		+ \frac{ y_i^2 }{ |y|^2 } 
			\left( \sum_{j=1}^d y_j dy^j \right)^2 
			\left[ \frac{|y|}{2} \sec^2( |y|/2 ) 
				- \tan( |y|/2 )
			\right]^2 \\
	&\quad 
		+ \frac{ 2y_i dy^i }{ |y|^4 } 
			\tan( |y|/2 ) 
			\left( \sum_{j=1}^d y_j dy^j \right)
			\left[ \frac{|y|}{2} \sec^2( |y|/2 ) 
				- \tan( |y|/2 )
			\right] \,. 
\end{aligned}
\end{equation}

We can also compute that 
\begin{equation}
	(|x|^2 + 1)^2 = ( \tan^2( |y|/2 ) + 1 )^2 = \sec^4( |y|/2 ) \,. 
\end{equation}

Then we put everything together to write 
\begin{equation}
\begin{aligned}
	\frac{ 4 \sum_{i=1}^d (dx^i)^2 }{ ( |x|^2 + 1 )^2 } 
	&= \frac{ 4 \sum_i (dy^i)^2 \tan^2( |y|/2) }{ \sec^4( |y|/2 ) |y|^2 }
		+ \frac{4 |y|^2 \left( \sum_{j=1}^d y_j dy^j \right) }{
			\sec^4( |y|/2 ) |y|^6 }
			\left[ \frac{|y|}{2} \sec^2( |y|/2 ) 
				- \tan( |y|/2 )
			\right]^2 \\
	&\quad 
		+ \frac{ 8 \sum_{i} y_i dy^i }{ \sec^4( |y|/2 ) |y|^4 }
			\tan( |y|/2 ) 
			\left( \sum_{j=1}^d y_j dy^j \right)
			\left[ \frac{|y|}{2} \sec^2( |y|/2 ) 
				- \tan( |y|/2 )
			\right] \\
	&= \frac{ 4 \sum_i (dy^i)^2 }{ |y|^2 } 
		\frac{ \tan^2( |y|/2 ) }{ \sec^4( |y|/2 ) } 
	+ \frac{ 4 \left( \sum_{j=1}^d y_j dy^j \right)^2 }{ |y|^4 }
		\left[ \frac{|y|^2}{4} 
			- \frac{ \tan^2( |y|/2 ) }{ \sec^4( |y|/2 ) }
		\right] \,. 
\end{aligned}
\end{equation}

At this point, we can write 
$\tan^2(\theta) / \sec^4(\theta) = \sin^2(\theta) \cos^2(\theta)$, 
and expand the term $\left( \sum_{j=1}^d y_j dy^j \right)^2$, 
then group together the equal indices to write 
\begin{equation}
	\left( \sum_{j=1}^d y_j dy^j \right)^2
	= \sum_{i=j} y_i^2 (dy^i)^2 
		+ \sum_{i\neq j} y_i y_j dy^i dy^j \,, 
\end{equation}
this gives us 
\begin{equation}
\begin{aligned}
	g_{ii}(y) 
	&= \frac{y_i^2}{|y|^2} 
			\left( 1 - \frac{4}{|y|^2} 
				\sin^2( |y|/2 ) \cos^2( |y|/2 ) 
			\right)
		+ \frac{4}{|y|^2}
		\sin^2( |y|/2 ) \cos^2( |y|/2 ) 
		\,, \\
	g_{ij}(y) 
	&= \frac{ y_i y_j }{ |y|^2 } 
		\left[ 1 
			-  \frac{4}{|y|^2}
			\sin^2( |y|/2 ) \cos^2( |y|/2 ) 
		\right] \,, 
	\quad i \neq j \,. 
\end{aligned}
\end{equation}

Finally the desired result follows from 
the double angle formula 
\begin{equation}
	2 \sin( |y| / 2 ) \cos( |y| / 2 ) 
	= \sin( |y| ) \,. 
\end{equation}

\end{proof}

\begin{lemma}
\label{lm:normal_coord_vol_est}
The Riemannian metric $g$ in the normal coordinates 
of $S^d$ has one eigenvalue of $1$ corresponding to 
the direction of $y$, 
and all other eigenvalues are 
\begin{equation}
	\frac{\sin^2( |y| ) }{|y|^2} \,, 
\end{equation}
with multiplicity $(d-1)$. 
Hence we obtain that 
\begin{equation}
	\det g = \left( \frac{ \sin( |y| ) }{|y|} 
			\right)^{2(d-1)} \,, 
\end{equation}
and thus we also have that whenever $|y| \leq \pi / 2$, 
we have the following estimate 
\begin{equation}
	\left( \frac{2}{\pi} \right)^{ 2(d-1) } 
	\leq \det g 
	\leq 1 \,. 
\end{equation}
\end{lemma}

\begin{proof}

Let $\widehat{y} = y / |y|$ be the unit vector, 
then we can rewrite the matrix $g$ in the following form 
\begin{equation}
	g(y) = \widehat{y} \widehat{y}^\top
		\left( 1 - \frac{\sin^2( |y| ) }{|y|^2} 
		\right)
		+ I_d \frac{\sin^2( |y| ) }{|y|^2} 
		\,,
\end{equation}
where $I_d$ is the identity matrix in $\mathbb{R}^{d\times d}$. 

Then clearly, $\widehat{y}$ is an eigenvector 
and we also have that 
\begin{equation}
	g(y) \widehat{y} = \widehat{y}
		\left( 
			1 - \frac{\sin^2( |y| ) }{|y|^2} 
			+ \frac{\sin^2( |y| ) }{|y|^2} 
		\right)
		= \widehat{y} \,, 
\end{equation}
hence $\widehat{y}$ corresponds to an eigenvalue of $1$. 

For all other directions $v$ orthogonal to $y$, 
we then have that 
\begin{equation}
	g(y) v = \frac{\sin^2( |y| ) }{|y|^2} 
		\, I_d \, v
		= \frac{\sin^2( |y| ) }{|y|^2} v \,, 
\end{equation}
therefore all other eigenvalues are 
$\frac{\sin^2( |y| ) }{|y|^2}$, 
and hence it has multiplicity $(d-1)$. 

Since determinant is just the product of all eigenvalues, 
we must have 
\begin{equation}
	\det g = \left( \frac{\sin( |y| ) }{|y|}
		\right)^{2(d-1)} \,. 
\end{equation}

Finally, we observe that for $\theta \in [0, \pi/2]$, 
$\sin( \theta ) / \theta $ is a decreasing function, 
therefore we have the following trivial lower bound 
for all $|y| \leq \pi/2$
\begin{equation}
	\det g 
	\geq \left( \frac{ 1 }{ \pi/2 }
		\right)^{2(d-1)}
	= \left( \frac{2}{\pi}
		\right)^{2(d-1)} \,. 
\end{equation}

Finally, to complete the proof, 
we observe that when $|y|=0$, 
we simply have $\det g = 1$. 

\end{proof}

Using the results of 
\cref{lm:normal_coord_metric,lm:normal_coord_vol_est}, 
we can write that 
\begin{equation}
	g_{ij}(x) 
	= \frac{ x_i x_j }{ |x|^2 } 
		\left[ 1 
			- \frac{ \sin^2( |x| ) }{|x|^2} 
		\right]
	+ \delta_{ij} \frac{ \sin^2( |x| ) }{|x|^2}  \,, 
\end{equation}
where $\delta_{ij}$ denotes the Kronecker delta, 
and that we know the matrix $g_{ij}$ 
has an eigenvalue of $1$ in the direction of $x / |x|$, 
and the rest of the eigenvalues are 
\begin{equation}
	\frac{ \sin^2( |x| ) }{|x|^2} \,. 
\end{equation}

Therefore we can also construct the inverse as follows 
\begin{equation}
	g^{ij}(x) 
	= \frac{ x_i x_j }{ |x|^2 } 
		\left[ 1 
			- \frac{|x|^2}{ \sin^2( |x| ) } 
		\right]
	+ \delta_{ij} \frac{|x|^2}{ \sin^2( |x| ) } 
		\,, 
\end{equation}
where observe that $g_{ij}, g^{ij}$ have matching eigenvectors, 
and the eigenvalues are exactly reciprocals of each other. 

Furthermore, we also have an explicit form for the determinant 
\begin{equation}
	\det g(x) = \left[ \frac{ \sin( |x| ) }{|x|} 
				\right]^{2(d-1)} \,. 
\end{equation}

\begin{lemma}
[Christoffel Symbol Formulas]
\label{lm:normal_coord_christoffel_symbol}
In normal coordinates on $S^d$, 
the Christoffel symbols has the following formula 
\begin{equation}
\begin{aligned}
	\Gamma_{ij}^k(y) 
	=& 
		y_k \left( \delta_{ij} - \frac{ y_i y_j }{ |y|^2 } \right) 
		\frac{1}{|y|^2} 
		\left( 1 - \frac{ \sin(|y|)^2 }{ |y|^2 } 
			- \frac{\sin(|y|)}{|y|} 
			\left( \cos(|y|) - \frac{\sin(|y|)}{|y|} \right) 
		\right) \\ 
	&+ 
		\left( \delta_{jk} y_i + \delta_{ik} y_j 
			- \frac{ 2 y_i y_j y_k }{ |y|^2 }
		\right) 
		\frac{1}{ \sin(|y|) \, |y| } 
		\left( \cos(|y|) - \frac{\sin(|y|)}{|y|} \right) \,. 
\end{aligned}
\end{equation}

\end{lemma}

\begin{proof}

We start by recalling the formula for the Christoffel symbols 
\begin{equation}
	\Gamma_{ij}^k 
	= 
		\frac{1}{2} g^{k\ell} 
		( \de_i g_{j\ell} + \de_j g_{i\ell} - \de_\ell g_{ij} ) \,. 
\end{equation}

Before we compute the derivative terms, 
we will compute some of the components 
\begin{equation}
\begin{aligned}
	\de_\ell \left( \frac{ y_i y_j }{ |y|^2 } \right) 
	&= 
	\left( \delta_{i\ell} y_j + \delta_{j\ell} y_i 
		- \frac{2 y_i y_j y_\ell}{|y|^2} \right) 
		\frac{1}{|y|^2} 
		\,, \\
	\de_\ell \left( \frac{\sin(|y|)^2 }{|y|^2} \right) 
	&= 
		\frac{2 y_\ell \sin(|y|) }{|y|^3} 
		\left( \cos(|y|) - \frac{\sin(|y|)}{|y|} \right) \,, 
\end{aligned}
\end{equation}
which implies 
\begin{equation}
\begin{aligned}
	\de_\ell g_{ij}
	=& \left( \delta_{i\ell} y_j + \delta_{j\ell} y_i 
		- \frac{2 y_i y_j y_\ell}{|y|^2} \right) 
		\frac{1}{|y|^2} 
		\left( 1 - \frac{\sin(|y|)^2 }{ |y|^2 } \right) 
		\\ 
	&+ \left( \delta_{ij} - \frac{y_i y_j}{ |y|^2 } \right) 
		\frac{ 2y_\ell \sin(|y|) }{ |y|^3 } 
		\left( \cos(|y|) - \frac{\sin(|y|)}{|y|} \right) \,. 
\end{aligned}
\end{equation}

Now we can compute the sum of the three terms inside the bracket as 
\begin{equation}
\begin{aligned}
	&\de_i g_{j\ell} + \de_j g_{i\ell} - \de_\ell g_{ij} \\ 
	&=  
		\frac{2 y_\ell}{|y|^2} 
		\left( 1 - \frac{\sin(|y|)^2 }{ |y|^2 } \right) 
	+ \left( 
		\delta_{j\ell} y_i + \delta_{i\ell} y_j 
		- \delta_{ij} y_\ell - \frac{y_i y_j y_\ell}{ |y|^2 } 
		\right) 
		\frac{ 2 \sin(|y|) }{ |y|^3 } 
		\left( \cos(|y|) - \frac{\sin(|y|)}{|y|} \right) \,. 
\end{aligned}
\end{equation}

Now we recall 
\begin{equation}
	g^{ij}(x) 
	= \frac{ x_i x_j }{ |x|^2 } 
		\left[ 1 
			- \frac{|x|^2}{ \sin^2( |x| ) } 
		\right]
	+ \delta_{ij} \frac{|x|^2}{ \sin^2( |x| ) } 
		\,, 
\end{equation}
and observe this implies $\Gamma_{ij}^k$ is 
a product of sums of the term $(a+b)(c+d)$, 
which we compute by opening up the brackets and write 
\begin{equation}
	\Gamma_{ij}^k = T_{11} + T_{12} + T_{21} + T_{22} \,, 
\end{equation}
where $T_{ij}$ is the product of 
the $i^\text{th}$ component of $\frac{1}{2} g^{k\ell}$ with 
the $j^\text{th}$ component of 
$\de_i g_{j\ell} + \de_j g_{i\ell} - \de_\ell g_{ij}$. 
We then compute separately 
\begin{equation}
\begin{aligned}
	T_{11} 
	&= 
		y_k \left( \delta_{ij} - \frac{y_i y_j}{|y|^2} \right) 
		\frac{1}{|y|^2} 
		\left( 1 - \frac{|y|^2}{ \sin(|y|)^2 } \right) 
		\left( 1 - \frac{ \sin(|y|)^2 }{|y|^2} \right) 
		\,, \\ 
	T_{12}
	&= 
		y_k \left( \delta_{ij} - \frac{y_i y_j}{|y|^2} \right) 
		\left( 1 - \frac{ \sin(|y|)^2 }{|y|^2} \right) 
		\frac{1}{|y| \sin(|y|)} 
		\left( \cos(|y|) - \frac{\sin(|y|)}{|y|} \right) 
		\,, \\ 
	T_{21}
	&= 
		y_k \left( \delta_{ij} - \frac{y_i y_j}{|y|^2} \right) 
		\frac{1}{ \sin(|y|)^2 } 
		\left( 1 - \frac{ \sin(|y|)^2 }{|y|^2} \right) 
		\,, \\ 
	T_{22}
	&= 
		y_k \left( \delta_{ij} - \frac{y_i y_j}{|y|^2} \right) 
		\frac{1}{|y| \sin(|y|)} 
		\left( \cos(|y|) - \frac{\sin(|y|)}{|y|} \right) 
		\\ 
	&\quad + 
		\left( 
			\delta_{jk} y_i + \delta_{ik} y_j 
			- \frac{ 2 y_i y_j y_k }{ |y|^2 }
		\right) 
		\frac{1}{|y| \sin(|y|)} 
		\left( \cos(|y|) - \frac{\sin(|y|)}{|y|} \right) 
		\,. 
\end{aligned}
\end{equation}

Finally, we get the desired result by adding these four terms and simplifying. 

\end{proof}

\section{Technical Results on Special Stochastic Processes}
\label{sec:app_stoc_proc}

\subsection{The Wright--Fisher Diffusion}
\label{subsec:app_wright_fisher}

In this section, we will review several existing results 
on the Wright--Fisher diffusion, 
in particular the connection with 
the radial process of spherical Brownian motion. 

We start by letting $\{W_t\}_{t\geq 0}$ 
be a standard Brownian motion on $S^d$ 
equipped with the Riemannian metric $g'$. 
Using standard results on the radial process 
\citep[Section 3]{hsu2002stochastic}, 
we can write down the SDE for $r_t := d_{g'}(W_0, W_t)$ 
\begin{equation}
\label{eq:app_radial_process}
	dr_t = \frac{d-1}{2} \cot( r_t ) \, dt 
		+ dB_t \,, 
\end{equation}
where $\{B_t\}_{t\geq 0}$ is a standard Brownian motion on $\mathbb{R}$. 

We will first establish a standard transformation of this radial process 
into the Wright--Fisher diffusion. 
This result is well known and commonly used, 
for example in \cite{mijatovic2018note}. 
We will provide a short proof for completeness. 

\begin{lemma}
\label{lm:radial_wright_fisher}
Let $r_t = d_{g'}(W_0, W_t)$ be the radial process of 
a standard Brownian motion $\{W_t\}_{t\geq 0}$ on $S^d$. 
Then $Y_t := \frac{1}{2} ( 1 - \cos(r_t) )$ 
is the unique solution of 
\begin{equation}
	dY_t = \frac{d}{4} ( 1 - 2Y_t) \, dt 
		+ \sqrt{ Y_t (1 - Y_t) } \, dB_t \,, 
\end{equation}
where $\{B_t\}_{t\geq 0}$ is a standard Brownian motion in $\mathbb{R}$. 
In other words, $\{Y_t\}_{t \geq 0}$ is 
the Wright--Fisher diffusion with parameters $(d/2, d/2)$. 
\end{lemma}

\begin{proof}

We will first compute It\^{o}'s Lemma for 
$\phi(x) = \cos(x)$ to get 
\begin{equation}
	\de_x \phi(x) = - \sin (x) = - \sqrt{ 1 - \phi(x)^2 }\,, 
	\quad 
	\de_{xx} \phi(x) = - \cos(x) = -\phi(x) \,. 
\end{equation}

This implies that $Z_t := \cos(r_t)$ solves the SDE 
\begin{equation}
\begin{aligned}
	dZ_t 
	&= 
		\left[ - \frac{d-1}{2} \cot( \arccos(Z_t) ) \sqrt{ 1 - Z_t^2 } 
			- \frac{1}{2} Z_t \right] \, dt 
		- \sqrt{ 1 - Z_t^2 } \, dB_t 
		\\ 
	&= 
		-\frac{d}{2} Z_t \, dt - \sqrt{ 1 - Z_t^2 } \, dB_t \,, 
\end{aligned}
\end{equation}
where we used the fact that $\sin(\arccos(x)) = \sqrt{1 - x^2}$. 
We also note the fact that the sign on Brownian motion is 
invariant in distribution. 

We will complete the proof by using It\^{o}'s Lemma again on 
$\psi(x) = \frac{1}{2} (1 - x)$, 
which gives us the SDE for $Y_t = \psi(Z_t)$ as 
\begin{equation}
	dY_t = \frac{d}{4} (1 - 2 Z_t) \, dt + \sqrt{ Y_t (1 - Y_t) } \, dB_t \,, 
\end{equation}
which is the desired result. 

\end{proof}

Next we will state another well known result for 
the transition density of the Wright--Fisher diffusion process. 
This result can be found in 
\cite{griffiths1979transition,tavare1984line,ethier1993transition,griffiths2010diffusion,jenkins2017exact}. 

\begin{theorem}
\label{thm:wf_transition_density}
Let $\{Y_t\}_{t\geq 0}$ be the Wright--Fisher diffusion, 
i.e. the solution of 
\begin{equation}
	dY_t = \frac{d}{4} ( 1 - 2Y_t) \, dt 
		+ \sqrt{ Y_t (1 - Y_t) } \, dB_t \,, 
\end{equation}
where $\{B_t\}_{t\geq 0}$ is a standard Brownian motion in $\mathbb{R}$. 
Then the density of $Y_t | Y_0 = x$ is given by 
\begin{equation}
	f(x,y;t) 
	= \sum_{m \geq 0} q_m(t) 
		\sum_{\ell = 0}^m \binom{m}{\ell} x^\ell (1-x)^{m-\ell} \, 
		\frac{ y^{d/2 + \ell - 1} (1-y)^{d/2 + m - \ell - 1} 
		}{B(d/2 + \ell, d/2 + m - \ell)} \,, 
\end{equation}
where $\{q_m(t)\}_{m \geq 0}$ is a probability distribution over $\mathbb{N}$, 
and $B(\theta_1, \theta_2)$ is the Beta function. 

In particular, if $Y_0 = 0$, then $Y_t$ has density 
\begin{equation}
	f(y;t) 
	= \sum_{m \geq 0} q_m(t) \, 
		\frac{ y^{d/2 - 1} (1-y)^{d/2 + m - 1} 
		}{B(d/2, d/2 + m)} \,. 
\end{equation}
\end{theorem}

\subsection{The Cox--Ingersoll--Ross Process}

In this section, we consider a class of 
Cox--Ingersoll--Ross (CIR) processes defined by 
\begin{equation}
\label{eq:cir_process}
	dY_t 
	= 
		\left[ 2 \lambda_* Y_t + \frac{1}{2\beta} 
		\right] \, dt 
	+ 
		\frac{2}{\sqrt{\beta}} \sqrt{ Y_t } \, dB_t \,, 
	\quad Y_0 = y_0 \geq 0 \,, 
\end{equation}
where $\lambda_*,\beta > 0$, 
and $\{B_t\}_{t\geq 0}$ is a standard one-dimensional Brownian motion. 
In other words, $\{Y_t\}_{t\geq 0}$ is a ``mean-avoiding'' process. 
Here we emphasize that the parameters 
are intentionally chosen so that 
$\{Y_t\}_{t\geq 0}$ is not mean-reverting. 

We start by adapting a Feymann-Kac type uniqueness theorem 
by \cite{ekstrom2011boundary}. 

\begin{theorem}
\citep[Theorem 2.3, Slightly Modified]{ekstrom2011boundary}
\label{thm:uniqueness_ekstrom}
Suppose the process 
\begin{equation}
	dY_t = \mu(Y_t) \, dt + \sigma(Y_t) \, dB_t \,, 
\end{equation}
is such that 
\begin{equation}
\begin{aligned}
	&\mu \in C^1([0,\infty), 
		\| \de_y \mu \|_\infty < \infty, 
		\text{ and } \mu(t,0) 
		\geq 0 \text{ for all } t \geq 0 \,, \\ 
	& \alpha(\cdot) := \frac{1}{2} \sigma(\cdot)^2 \in C^2([0,\infty)) 
		\text{ and } 
		\sigma(x) = 0 \text{ if and only if } x = 0 \,, \\ 
	& |\mu(x)| + | \sigma(x) | + | \alpha(x) | \leq C(1+x) 
		\text{ for all } t,x \geq 0 \,, \\
	& \phi \in C^1([0,\infty)) \text{ and } 
		\| \phi \|_\infty + \| \de_x \phi \|_\infty < \infty \,, 
\end{aligned}
\end{equation}
for some constant $C>0$. 
Then for the following partial differential equation 
with initial boundary conditions 
\begin{equation}
	\begin{cases}
		\de_t u = \mu \, \de_x u + \frac{1}{2} \sigma^2 \, \de_{xx} u \,, 
		& 
		t,x \in (0, \infty) \times (0, \infty) \,, \\
		u(0,x) = \phi(x) \,, & t = 0 \,, \\ 
		\de_t u(t,0) = \mu(0) \de_x u(t, 0) \,, 
			& x = 0 \,, 
	\end{cases}
\end{equation}
the unique classical solution 
$u \in C^{1,2}( (0,\infty)^2 ) \cap C( [0,\infty)^2 )$ 
admits the following stochastic representation 
\begin{equation}
	u(t,x) = \mathbb{E} [ \, \phi(Y_t) \, | \, Y_0 = x \, ] \,. 
\end{equation}
\end{theorem}

\begin{proof}

We will only show uniqueness as the regularity proof follows 
exactly from \cite{ekstrom2011boundary} 
by dropping the exponential term. 
Let $v^1,v^2$ be two classical solutions, then 
$v:= v^1 - v^2$ must satisfy the follow problem 
\begin{equation}
	\begin{cases}
		\de_t v = \mu \, \de_x v + \frac{1}{2} \sigma^2 \, \de_{xx} v \,, 
		& 
		t,x \in (0, \infty) \times (0, \infty) \,, \\
		v(0,x) = 0 \,, & t = 0 \,, \\ 
		\de_t v(t,0) = \mu(0) \de_x v(t, 0) \,, 
			& x = 0 \,. 
	\end{cases}
\end{equation}

Here we observe that $h(t,x) = (1+x) e^{Mt}$ 
is a supersolution for some $M>0$ large, 
more precisely 
\begin{equation}
	\de_t h \geq \mu \, \de_x h + \frac{1}{2} \sigma^2 \, \de_{xx} h \,, 
	\quad t,x \in (0, \infty) \times (0, \infty) \,. 
\end{equation}

Similarly observe that $\epsilon h$ is a supersolution 
for all $\epsilon > 0$, 
and $-\epsilon h$ is a subsolution for all $\epsilon > 0$. 
Then by the maximum principle for parabolic super/subsolutions
\citep[Section 7.1.4, Theorem 8]{evans2010partial}, 
the maximum of $h$ is attained at the boundary 
$\de (0,\infty)^2 = \{x=0\} \cup \{t=0\}$. 
Furthermore, we also have that 
$ -\epsilon h \leq v \leq \epsilon h$ 
for all $\epsilon > 0$. 
Therefore we must also have $v \equiv 0$, 
hence the solution is unique.

\end{proof}

While it is well known that the $Y_t$ 
is related a transformed non-central Chi-squared random variable 
for standard parameters, 
but we need to extend this result to the general case. 
In particular, we cannot guarantee that $Y_t$ does not hit $0$. 
We start by showing the characteristic function is indeed the desired one. 

\begin{lemma}
[Characteristic Function of $Y_t$]
\label{lm:cir_characteristic}
Let $\{Y_t\}_{t\geq 0}$ be the (unique strong) solution 
of \cref{eq:cir_process}. 
Then we have the following formula for the characteristic function 
\begin{equation}
	\mathbb{E} e^{i s Y_t} 
	= 
		\frac{ 
		\exp\left( 
			\frac{
				is \, y_0 \, e^{2 \lambda_* t}
			}{ 
				1 - \frac{is}{ \lambda_* \beta} 
					( e^{ 2\lambda_* t} - 1 )
			} 
			\right) }{ 
			\left( 1 - \frac{is}{ \lambda_* \beta} 
					( e^{ 2\lambda_* t} - 1 )
			\right)^{1 / 2}
			} \,. 
\end{equation}
\end{lemma}

\begin{proof}

We start by letting $\phi(x) := e^{i s x}$ 
and $u(t,x) := \mathbb{E} [ \phi(Y_t) | Y_0 = x ]$ 
satisfies 
the backward Kolmogorov equation 
\begin{equation}
\label{eq:kolmogorov_cf}
\begin{cases}
	\de_t u 
	= \left( 2 \lambda_* x + \frac{1}{\beta} \right) \de_x u 
		+ \frac{2x}{\beta} \de_{xx} u 
		\,, 
		& \quad (t,x) \in (0,\infty) \times (0,\infty) \,, \\ 
	u(0, x) = \phi(x) \,, & t = 0 \,, \\ 
	\de_t u(t,0) = \frac{1}{\beta} \de_x u(t,0) \,, & x=0 \,. 
\end{cases}
\end{equation}

We will guess and check the solution 
\begin{equation}
	u(t,x) := 
	\frac{ 
		\exp\left( 
			\frac{
				is \, x \, e^{2 \lambda_* t}
			}{ 
				1 - \frac{is}{ \lambda_* \beta} 
					( e^{ 2\lambda_* t} - 1 )
			} 
			\right) }{ 
			\left( 1 - \frac{is}{ \lambda_* \beta} 
					( e^{ 2\lambda_* t} - 1 )
			\right)^{1 / 2}
			} \,. 
\end{equation}

Using Wolfram Cloud \citep[Wolfram Cloud]{WolframCloud}, 
we can define the guessed solution 
\begin{lstlisting}[language=Mathematica]
u[t_,x_,s_] := Exp[I*s*x*Exp[2*L*t] / ( 1 - I*s/(L*b)*(Exp[2*L*t] - 1) )] / ( 1 - I*s/(L*b)*(Exp[2*L*t] - 1) )^(1 / 2)
\end{lstlisting}

Next we check all three conditions of the PDE 
in \cref{eq:kolmogorov_cf} 
by confirming that all three of the following commands 
output zero 
\begin{lstlisting}[language=Mathematica]
FullSimplify[ - D[u[t,x,s],{t}] +(2*L*x + 1/b )* D[u[t,x,s],{x}] + 2*x/b * D[u[t,x,s],{x,2}] ]
Limit[u[t,x,s] - Exp[I*s*x] ,t->0]
Limit[D[u[t,x,s],t] - 1/b*D[u[t,x,s],{x}] ,x->0]
\end{lstlisting}

Therefore, we confirm that $u(t,x)$ is indeed the unique solution 
based on \cref{thm:uniqueness_ekstrom}. 

\end{proof}

Next we will use a classical result of \cite{mcnolty1973some} 
to compute the density of $Y_t$ exactly. 

\begin{lemma}
\citep[Case 1]{mcnolty1973some}
\label{lm:cir_cf_density_mcnolty}
Let the density $f(x)$ supported on $[0,\infty)$ be defined by 
\begin{equation}
	f(x; \lambda,\gamma,Q) := 
		2^{Q/2 - 3/2} 
		\frac{ x^{(Q-1)/2} }{ \gamma^{Q-1} } 
		\lambda^Q 
		\exp\left[ 
			-\frac{\gamma^2}{4 \lambda} 
			-\frac{\lambda}{2} x
		\right] 
		I_{Q-1} \left[ 
			\gamma \left( \frac{x}{2} \right)^{1/2}
		\right] \,, 
	\quad x \geq 0 \,, 
\end{equation}
where $\gamma,\lambda \geq 0, Q > 0$, 
and $I_{Q-1}$ is the modified Bessel function of the first kind 
of degree $Q-1$. 
Then $f(x)$ has the following characteristic function 
\begin{equation}
	\varphi(t) := 
		\left( 1 - \frac{2it}{\lambda} \right)^{-Q} 
		\exp\left[ 
			\frac{it\gamma^2}{2\lambda^2 
				\left( 1 - \frac{2it}{\lambda} \right) }
		\right] \,. 
\end{equation}
\end{lemma}

We will substitute and calculate the density explicitly using this result. 

\begin{corollary}
\label{cor:cir_density}
Let $\{Y_t\}_{t\geq 0}$ be the (unique strong) solution 
of \cref{eq:cir_process}. 
Then we have the following formula for the density
\begin{equation}
	f(y;t) = 
		2^{ (1/2 - 3) / 2 } 
		\left( \frac{y}{y_0} \right)^{ (1/2 - 1) / 2 } 
		\frac{ \lambda_* \beta }{ 
			e^{ 1 \lambda_* t /2 }
			\sinh ( \lambda_*t ) 
			} 
		\exp\left[ 
			\frac{ \lambda_* \beta \left( 
				y e^{ -2 \lambda_* t } - \frac{y_0}{2}
				\right) }{
				1 - e^{ -2 \lambda_* t }
				} 
		\right] 
		I_{(-1/2)}\left( 
			\frac{ \lambda_* \beta }{ \sinh( \lambda_* t ) } 
			\sqrt{ \frac{y y_0}{2} }
		\right) \,. 
\end{equation}

\end{corollary}

\begin{proof}

\cref{lm:cir_cf_density_mcnolty} allows us to calculate 
the density directly from the characteristic function 
via the following substitution 
\begin{equation}
\begin{aligned}
	s := t \,, \quad 
	Q := 1 / 2 \,, \quad 
	\lambda := \frac{2 \lambda_* \beta }{ e^{2 \lambda_* t} - 1 } 
		\,, \quad 
	\gamma := \sqrt{y_0} \lambda e^{ \lambda_*t } \,, 
\end{aligned}
\end{equation}
which gives us the intermediate result of 
\begin{equation}
	f(y;t) = 2^{ (1/2 - 3)/2 } 
		\frac{ y^{(1/2-1)/2} }{ 
			\left(\sqrt{y_0} \lambda e^{\lambda_*t} \right)^{1/2 - 1} } 
		\lambda^{1/2} 
		\exp\left[ - \frac{ y_0 \lambda e^{ 2 \lambda_* t }}{ 4 } 
			+ \frac{\lambda}{2} y 
		\right] 
		I_{(1/2 - 1)} \left( \sqrt{y_0} \lambda e^{\lambda_* t} 
			\sqrt{\frac{y}{2}} \right) \,. 
\end{equation}

We get the desired result by simplifying.

\end{proof}

We will next prove an upper bound of the density 
in terms of $t$ when $x \in [0, R]$. 

\begin{lemma}
\label{lm:cir_density_est}
For the density $f(y;t)$ defined in \cref{cor:cir_density}, 
on the interval $y \leq R$, 
we have the following bound for $t \geq 0$
\begin{equation}
	f(y;t) \leq C e^{ - 2 \lambda_* t } \,, 
\end{equation}
where $C:=C(R, \lambda_*, \beta) > 0$ 
is a constant independent of $t$. 
\end{lemma}

\begin{proof}

We start separating the expression into three separate parts, 
\begin{equation}
\begin{aligned}
	T_1 
	&:= 
		\frac{ \lambda_* \beta }{ 
			e^{ \lambda_* t /2 }
			\sinh ( \lambda_*t ) 
			} 
		\,, \\
	T_2 
	&:= 
		\exp\left[ 
			\frac{ \lambda_* \beta \left( 
				y e^{ -2 \lambda_* t } - \frac{y_0}{2}
				\right) }{
				1 - e^{ -2 \lambda_* t }
				} 
		\right] 
		\,, \\ 
	T_3 
	&:= 
		\left( \frac{y}{y_0} \right)^{ (1/2 - 1) / 2 } 
		I_{(1/2 - 1)}\left( 
			\frac{ \lambda_* \beta }{ \sinh( \lambda_* t ) } 
			\sqrt{ \frac{y y_0}{2} } 
		\right) \,, 
\end{aligned}
\end{equation}
where we observe that $f(y;t):= 2^{(1/2-3)/2} T_1 T_2 T_3$. 

We start with the first term, and observe that 
\begin{equation}
	\frac{ \lambda_* \beta }{ 
			e^{ \lambda_* t /2 }
			\sinh ( \lambda_*t ) 
			} 
	\leq 
		C e^{ -(1/2+1) \lambda_* t } \,, 
\end{equation}
for some constant $C$. 

For the second term, since 
\begin{equation}
	\lim_{t \to \infty} \exp\left[ 
			\frac{ \lambda_* \beta \left( 
				y e^{ -2 \lambda_* t } - \frac{1}{2}
				\right) }{
				1 - e^{ -2 \lambda_* t }
				} 
		\right] = e^{ -\frac{1}{2} \lambda_* \beta y_0 } \,, 
\end{equation}
we actually have that $T_2 \leq C$ 
for some constant $C$. 

Before we observe the third term, 
we recall an inequality from 
\cite[Equation 6.25]{luke1972inequalities} 
for $z>0$ and $\nu > -1/2$ 
\begin{equation}
	1 < \Gamma(\nu+1) \left( \frac{2}{z} \right)^{\nu} I_{\nu} 
	< \cosh(z) \,, 
\end{equation}
where in this case, since $I_{\nu} = I_{-\nu}$, 
we replace $I_{(1/2-1)}$ with $I_{|1/2-1|}$ and write 
\begin{equation}
\begin{aligned}
	I_{(1/2 - 1)}\left( 
			\frac{ \lambda_* \beta }{ \sinh( \lambda_* t ) } 
			\sqrt{ \frac{y y_0}{2} } 
		\right)
	&\leq 
		C 
		\left( \frac{x^{ 1/2 }}{ \sinh( \lambda_* t } 
		\right)^{|1/2 - 1|} 
		\cosh\left( 
			\frac{ \lambda_* \beta }{ \sinh( \lambda_* t ) } 
			\sqrt{ \frac{ y y_0 }{2} } 
		\right) \\
	&\leq 
		C y^{ |1/2-1|/2 } e^{ - |(1/2-1) \lambda_*| t } \,. 
\end{aligned}
\end{equation}

Putting this together in $T_3$, we get that 
\begin{equation}
	T_3 \leq C e^{ - |(1/2-1) \lambda_*| t } \,, 
\end{equation}
as we no longer have to bound $y^{(1/2-1)/2}$ near $y=0$. 

Finally, $y \leq R$ implies that 
\begin{equation}
	f(y;t) \leq C e^{ - ( 1/2 +1 + |1/2-1| ) \lambda_* t } \,, 
\end{equation}
where we simplify $ 1/2 +1 + |1/2-1| = 2$, 
so we can equivalently write 
\begin{equation}
	f(y;t) \leq C e^{ - 2 \lambda_* t } \,, 
\end{equation}
which is the desired result. 

\end{proof}

\section{Miscellaneous Technical Lemmas}
\label{sec:app_misc_tech}

\begin{lemma}
[Gr\"{o}nwall's Inequality, Constant Rate]
\label{lm:gronwall}
Suppose $u:[0,T] \to \mathbb{R}$ is a differentiable function, 
and $a>0$ is a constant and $b:[0,T] \to \mathbb{R}$ 
such that 
\begin{equation}
	\de_t u(t) \leq - a u(t) + b(t) \,, \quad t \in [0,T] \,,
\end{equation}
then we have that
\begin{equation}
	u(t) \leq e^{-at} u(0) 
		+ \int_0^t e^{a(s-t)} b(s) \, ds 
		\,,
		\quad t \in [0,T] \,. 
\end{equation}
\end{lemma}

\begin{proof}

We start by computing 
\begin{equation}
	\de_t \left( e^{at} u(t) \right)
	= e^{at} \de_t u(t) + a e^{at} u(t) 
	\leq e^{at} b(t) \,,
\end{equation}
then integrating in $[0,t]$, we can get 
\begin{equation}
	e^{at} u(t) - u(0) \leq \int_0^t e^{as} b(s) \, ds \,,
\end{equation}
manipulating the above inequality gives us 
\begin{equation}
	u(t) \leq e^{-at} u(0) + \int_0^t e^{a(s-t)} b(s) \, ds \,,
\end{equation}
which is the desired result. 

\end{proof}

We will next adapt a result of \cite{robbins1955remark} 
to non-integer values. 
\begin{lemma}[Gamma Function Bounds]
\label{lm:gamma_bounds}
For all $z > 0$, we have the following bounds 
\begin{equation}
	\sqrt{2 \pi} z^{z+1/2} e^{-z} e^{ \frac{1}{12z + 1} }
	< \Gamma(z + 1) 
	< \sqrt{2 \pi} z^{z+1/2} e^{-z} e^{ \frac{1}{12z} } \,. 
\end{equation}
\end{lemma}

\begin{proof}

We will start by defining $z_0 := z - \lfloor z \rfloor$, 
and $n := z - z_0$. 
Then we will follow the steps of \cite{robbins1955remark} 
closely and only modify the term 
\begin{equation}
	S_n := \log \Gamma( z + 1 ) 
	= \Gamma(z_0 + 1) + \sum_{p=1}^{n-1} \log(z_0 + p + 1) \,. 
\end{equation}

Next we define the terms 
\begin{equation}
\begin{aligned}
	A_p &= \int_{z_0 + p}^{z_0 + p + 1} \log x \, dx \,, \\ 
	b_p &= \frac{1}{2} [ \log (z_0 + p + 1) - \log (z_0 + p) ] \,, \\ 
	\epsilon_p &= \int_{z_0 + p}^{z_0 + p + 1} \log x \, dx
		- \frac{1}{2} [ \log (z_0 + p + 1) + \log ( z_0 + p ) ] \,, 
\end{aligned}
\end{equation}
which leads to the following identity 
\begin{equation}
	\log(z_0 + p + 1) = A_p + b_p - \epsilon_p \,. 
\end{equation}
Here we remark we can view $\log (z_0 + p + 1)$ 
as a rectangle with width $1$, 
$A_p$ is the integral of a curve, 
$\epsilon_p$ as the trapezoid approximation error of the integral, 
and $b_p$ as trapezoid approximation error of the rectangle. 

Then we can write 
\begin{equation}
\begin{aligned}
	S_n &= \log \Gamma(z_0 + 1) 
		+ \sum_{p=1}^{n-1} ( A_p + b_p - \epsilon_p ) \\
	&= \log \Gamma(z_0 + 1) + \int_{z_0 + 1}^{z_0 + n} \log x \, dx 
		+ \frac{1}{2} [ \log (z_0 + n) - \log (z_0 + 1) ] 
		- \sum_{p=1}^{n-1} \epsilon_p \,. 
\end{aligned}
\end{equation}

Next we will use the fact that 
$\int \log x \, dx = x \log x - x$ 
to write 
\begin{equation}
	S_n 
	= \log \Gamma(z_0 + 1) 
		+ (z_0 + n + 1/2) \log (z_0 + n)
		- (z_0 + n + 1/2) \log (z_0 + 1)
		- (n-1) 
		- \sum_{p=1}^{n-1} \epsilon_p \,. 
\end{equation}

We can also compute $\epsilon_p$ as 
\begin{equation}
	\epsilon_p = (z_0 + p + 1/2) 
		\log \left( \frac{z_0 + p + 1}{z_0 + 1}
		\right) - 1 \,. 
\end{equation}

Choosing $x = ( 2(z_0 + p) + 1 )^{-1}$, we can also write 
\begin{equation}
	\frac{z_0 + p + 1}{z_0 + p} = \frac{1+x}{1-x} \,, 
\end{equation}
and we can write the Taylor expansion of the $\log$ term as 
\begin{equation}
	\log\left( \frac{1+x}{1-x} \right)
	= 2 \left( x + \frac{x^3}{3} + \frac{x^5}{5} + \cdots 
		\right) \,. 
\end{equation}

This allows us to write 
\begin{equation}
	\epsilon_p = \frac{1}{3 (2(z_0 + p) + 1)^2}
		+ \frac{1}{5 (2(z_0 + p) + 1)^4} 
		+ \frac{1}{7 (2(z_0 + p) + 1)^6} 
		+ \cdots \,. 
\end{equation}

Therefore we can get the following upper and lower bound as follows 
\begin{equation}
\begin{aligned}
	\epsilon_p 
	&< \frac{1}{3 (2(z_0 + p) + 1)^2} 
		\left( 1 + \frac{1}{(2(z_0 + p) + 1)^2} 
			+ \frac{1}{(2(z_0 + p) + 1)^4} 
			+ \cdots 
		\right) \\ 
	&= \frac{1}{3 (2(z_0 + p) + 1)^2 }
		\frac{1}{ 1 - (2(z_0 + p) + 1)^{-2} } \\ 
	&= \frac{1}{12} \left( \frac{1}{z_0 + p} - \frac{1}{z_0 + p + 1}
		\right) \,, \\ 
	\epsilon_p 
	&> \frac{1}{3 (2(z_0 + p) + 1)^2} 
		\left( 1 + \frac{1}{3(2(z_0 + p) + 1)^2} 
			+ \frac{1}{ 3^2 (2(z_0 + p) + 1)^4 } 
			+ \cdots 
		\right) \\ 
	&= \frac{1}{3 (2(z_0 + p) + 1)^2 }
		\frac{1}{1 - \frac{1}{3} (2(z_0 + p) + 1)^{-2} } \\ 
	&> \frac{1}{12} \left( 
		\frac{1}{z_0 + p + \frac{1}{12}} 
		- \frac{1}{z_0 + p + 1 + \frac{1}{12}}
		\right) \,. 
\end{aligned}
\end{equation}

Now defining 
\begin{equation}
	B := \sum_{p = 1}^\infty \epsilon_p \,, 
	\quad 
	r_n := \sum_{p=n}^\infty \epsilon_p \,, 
\end{equation}
we can get the following bounds 
\begin{equation}
\label{eq:gamma_rn_bounds}
	\frac{1}{12(z_0 + 1) + 1} < B < \frac{1}{12(z_0 + 1)} \,, 
	\quad
	\frac{1}{12(z_0 + n) + 1} < r_n < \frac{1}{12(z_0 + n) + 1} \,. 
\end{equation}

Putting everything together with $z = z_0 + n$, 
we have that 
\begin{equation}
	S_n = (z+1/2) \log z - z + r_n + C_0 \,, 
\end{equation}
where $C_0$ is a constant independent of $n$ defined by 
\begin{equation}
	C_0 := 1 - B + z_0 + 1 + \log \Gamma(z_0 + 1) 
			- (z_0 + 1) \log (z_0 + 1) 
			- \frac{1}{2} \log (z_0 + 1) \,. 
\end{equation}

Taking the limit as $n \to \infty$, 
we can recover Stirling's formula as 
\begin{equation}
	S_n = 
	\left[ (z+1/2) \log z - z + \frac{1}{2} \log (2 \pi) \right]
	(1 + o_n(1)) \,, 
\end{equation}
therefore we must have that $C_0 = \frac{1}{2} \log (2 \pi)$. 

Putting everything together, we have that 
\begin{equation}
	\Gamma(z) = \sqrt{2 \pi} (z + 1/2)^z e^{-z} e^{r_n} \,, 
\end{equation}
plugging the bounds from \cref{eq:gamma_rn_bounds}, 
we recover the desired result. 

\end{proof}

\begin{lemma}
\label{lm:xlogx}
We have the two following inequalities:
\begin{enumerate}
	\item Suppose $y \geq e$. Choosing $D := \max(c_1,c_2)$, 
		and $x = \frac{D}{y} \log y$, 
		we have that 
		\begin{equation}
			y \geq \frac{c_1}{x} \log \frac{c_2}{x} \,. 
		\end{equation}
	\item Suppose $y \geq e$. Choosing $x = 4 y \log y$ implies 
		\begin{equation}
			\frac{x}{\log x} \geq y \,. 
		\end{equation}
\end{enumerate}
\end{lemma}

\begin{proof}

For the first result, we simply write out the right hand side 
\begin{equation}
	\frac{c_1}{x} \log \frac{c_2}{x} 
	= 
		\frac{c_1}{\max(c_1, c_2)} y 
		\frac{\log \left[ \frac{c_2}{\max(c_1,c_2)} y 
			\frac{1}{\log y} \right] }{ 
			\log y 
			} 
	\leq 
		y \frac{ \log \left[ \frac{y}{ \log y} \right] }{ \log y } 
		\,. 
\end{equation}

Using the fact that $y \geq e$, we have that $\log y \geq 1$, 
and therefore 
\begin{equation}
	y \frac{ \log \left[ \frac{y}{ \log y} \right] }{ \log y }  
	\leq 
		y \frac{ \log \left[ y \right] }{ \log y } 
	= 
		y \,, 
\end{equation}
which is the desired result. 

For the second inequality, we similarly also write out 
\begin{equation}
	\frac{x}{\log x} 
	= 
		\frac{4 y \log y}{ \log ( 4 y \log y ) } 
	= 
		y \frac{ 4 \log y }{ \log 4 + \log y + \log \log y } \,. 
\end{equation}

Next we will use the fact that $y \geq e$ to get 
$\log \log y \leq \log y \leq y$ and $\log y \geq 1$, 
which further implies 
\begin{equation}
	y \frac{ 4 \log y }{ \log 4 + \log y + \log \log y } 
	\geq 
		y \frac{ 4 }{ \log 4 + 2 \log y } 
	\geq 
		y \frac{4}{ \log 4 + 2 } \,, 
\end{equation}
which gives us the desired result since $4 \geq \log 4 + 2$. 

\end{proof}

We will also need comparison theorems for 
one-dimensional SDEs. 
Here, we will modify a standard comparison theorem 
from \cite[Proposition 3.12]{pardoux:hal-01108223}, 
so that we won't require strong existence and uniqueness 
of solutions for the SDEs, 
but instead only local solutions up to a stopping time. 
In particular, we say that $(Z_t, \tau)$ is a local solution of an SDE 
\begin{equation}
	dZ_t = \Phi(Z_t) \, dt + \sigma(Z_t) \, dB_t \,, 
	\quad 
	Z_0 = z_0 \,, 
\end{equation}
if $\tau$ is an adapted stopping time, and $Z_t$ satisfies 
\begin{equation}
	Z_t = z_0 + \int_0^t \Phi(Z_s) \, ds + \int_0^t \sigma(Z_s) \, dB_s \,, 
	\quad 
	\text{ whenever } t < \tau \,. 
\end{equation}

\begin{proposition}
\label{thm:SDE_weak_comparison}
Let $(Z_t,\tau),(\wt{Z}_t, \wt{\tau})$ be local solutions 
(on some filtered probability space)
to the following one dimensional SDEs 
\begin{equation}
\begin{aligned}
	dZ_t &= \Phi(Z_t) \, dt + \sigma(Z_t) \, dB_t \,, \\ 
	d\wt{Z}_t &= \wt{\Phi}(\wt{Z}_t) \, dt + \sigma(Z_t) \, dB_t \,, 
\end{aligned}
\end{equation}
where $Z_0 = \wt{Z}_0$ almost surely, and $\Phi, \wt{\Phi}, \sigma$ are locally Lipschitz. 

If $\Phi(x) \geq \wt{\Phi}(x)$ for all $x \in \mathbb{R}$, 
then $Z_t \geq \wt{Z}_t$ a.s. 

\end{proposition}

\begin{proof}

The result follows directly from \cite[Proposition 3.12]{pardoux:hal-01108223} as the only difference is the existence and uniqueness of a local solution instead of a strong solution.

\end{proof}

\end{document}